\documentclass[11 pt]{elsarticle}
\usepackage{framed,multirow}
\usepackage[utf8]{inputenc}
\usepackage[T1]{fontenc}
\usepackage{siunitx}
\usepackage{graphicx}

\usepackage{hyperref}

\usepackage[ruled,vlined]{algorithm2e}
\usepackage{xcolor}

\SetKwInOut{KwRequire}{Require}
\SetKwInOut{KwEnsure}{Ensure}
\SetKwComment{Comment}{$\triangleright$\ }{}
\newcommand{\graycomment}[1]{\Comment{\textcolor{gray}{#1}}}

\newcounter{algline}

\DontPrintSemicolon

\usepackage{makecell}
\usepackage{subcaption}
\usepackage{hyperref,float}
\usepackage{booktabs}

\usepackage{tabu}
\usepackage{float}
\usepackage{placeins}
\usepackage{multirow} 
\usepackage{amssymb}
\usepackage{latexsym}
\usepackage{amsmath}
\usepackage[
left=3.8cm,
right=3.8cm,
top=4.8cm,
bottom=4.8cm
]{geometry}

\usepackage{caption}
\usepackage{amsthm}

\theoremstyle{plain}
\newtheorem{theorem}{Theorem}[section]

\theoremstyle{definition}
\newtheorem{definition}[theorem]{Definition}

\theoremstyle{remark}

\usepackage{amsfonts}
\usepackage{float}
\usepackage{url}
\usepackage{xcolor}

\definecolor{newcolor}{rgb}{.8,.349,.1}

\journal{arXiv}

\usepackage[nolist]{acronym}
\begin{acronym}
\acro{PDE}[PDE]{partial differential equations}
\acro{DG}{discontinuous Galerkin}
\acro{BC}{boundary condition}
\acro{KANs}{Kolmogorov-Arnold networks}
\acro{KAN}{Kolmogorov-Arnold network}
\acro{SciML}{Scientific machine learning}
\acro{BCs}{\ac{BC}s}
\acro{MLPs}{multilayer perceptrons}
\acro{ML}{machine learning}
\acro{RBFs}{radial basis functions}
\acro{PINNs}{physics-informed neural networks}
\acro{DeepONet}{deep operator networks}
\acro{EDNNs}{evolutionary deep neural networks}
\acro{EvoKAN}{evolutionary \ac{KAN}}
\acro{GP}{Gaussian process}
\acro{SAV}{scalar auxiliary variable}
\acro{EDNN}{evolutionary deep neural networks}
\acro{EvoKAN}{evolutionary \ac{KAN}}
\acro{PINN}{physics-informed neural networks}
\acro{RTE}{radiative transfer equation}
\acro{DG}{discontinuous Galerkin}
\end{acronym}

\begin{document}
% \linenumbers
% ========================================= %
% ================ authors ================ %
% ========================================= %

\begin{frontmatter}

\title{A Hyperbolic Neural Closure for M1 Radiation Transfer}

\author[1]{Bongseok Kim}
% \ead{kim4853@purdue.edu}

\author[2]{Jiahao Zhang}
% \ead{zhang4174@purdue.edu}

\author[3]{Johannes Krotz}
% \ead{jkrotz@nd.edu}

\author[4,5]{Dinshaw Balsara}
% \ead{dbalsara@nd.edu}

\author[3]{Ryan McClarren}
% \ead{rmcclarr@nd.edu}

\author[1,2]{Guang Lin\corref{cor1}}

\cortext[cor1]{Corresponding author.}
\ead{guanglin@purdue.edu}

\address[1]{School of Mechanical Engineering, Purdue University, West Lafayette, IN, USA}

\address[2]{Department of Mathematics, Purdue University, West Lafayette, IN, USA}

\address[3]{Department of Aerospace and Mechanical Engineering, University of Notre Dame, Notre Dame, IN, USA}

\address[4]{Department of Physics and Astronomy, University of Notre Dame, Notre Dame, IN, USA}
\address[5]{Department of Applied and Computational Mathematics and Statistics, University of Notre Dame, Notre Dame, IN, USA}

% \end{frontmatter}

% ========================================= %
% =============== abstract ================ %
% ========================================= %
\begin{NoHyper}
\begin{abstract}
% intro
In radiation transfer simulations, an M1 method achieves substantial computational savings by replacing the full angular transport equation with a low-order moment system. Because this reduced system is not closed, a closure model is required to represent the unknown higher-order moments using lower-order moments.
While machine learning (ML)-based closures can improve accuracy beyond classical analytic closures, unconstrained learned closures may produce non-real characteristic speeds and consequently cause numerical solver breakdown.
% proposed method
To guarantee real eigenvalues of the Jacobian associated with ML closures, we propose a hyperbolic neural closure for the M1 radiative transfer system.
Rather than directly predicting closure terms, we parameterize the Jacobian through two neural networks: (i) a symmetric matrix network and (ii) a strictly convex entropy network whose Hessian defines a positive definite symmetrizer.
These components are combined to yield a Jacobian that is similar to a symmetric matrix, thereby ensuring real eigenvalues. The closure is then reconstructed by numerical integration of the learned Jacobian field along a prescribed integration path.
% results
Numerical experiments show that the proposed closure not only achieves higher closure accuracy than classical analytic closures, but also improves solution accuracy and remains stable in discontinuous Galerkin simulations for radiative transfer problems.
\end{abstract}

\end{NoHyper}

%%Graphical abstract

\end{frontmatter}

\section{Introduction}
% intro
%% background
Radiative transfer arises in diverse physical phenomena, including astrophysics \cite{gonzalez2007heracles}, plasma physics \cite{fleck1971implicit}, atmospheric science \cite{cahalan2005i3rc}, biomedical optics \cite{klose2002optical_forward}, and heat transfer \cite{chai1994finite}.
The \ac{RTE} is a kinetic transport equation whose direct numerical simulation is computationally expensive because the solution depends on spatial, angular, frequency, and temporal variables.
%% Moment method
Moment methods~\cite{levermore1996moment} reduce computational complexity by replacing the full angular transport equation with a finite hierarchy of moment equations. Among these formulations, the M1 method~\cite{murchikova2017analytic,levermore1984relating} truncates the hierarchy at second order, leading to a hyperbolic system for radiation energy density and radiation flux together with a closure relation for the radiation pressure tensor.
%% problem statement
The remaining difficulty is the \textit{closure problem}: the moment system contains more unknowns than governing equations, leaving the radiation pressure tensor unspecified.
%% classical approach
A widely used approach to address the closure problem is to use analytic closure models for the radiation pressure tensor, such as the Levermore closure~\cite{levermore1984relating}, the Minerbo closure~\cite{minerbo1978maximum}, and several related formulations surveyed in~\cite{murchikova2017analytic}.
%% problem of classical analytic closure
However, classical closure models can become restrictive in strongly anisotropic transport regimes~\cite{frank2007radiative,hauck2011entropy} and struggle to represent complex angular radiation distributions~\cite{hauck2011entropy}.

% difficulty of using a data-driven approach
As an alternative to classical closure models, \ac{ML}-based closures may offer a promising approach for constructing more accurate closure relations for M1 radiative transfer.
However, the \ac{ML}-based closure can violate \textit{hyperbolicity} of the resulting M1 system, since the learned closure directly modifies the flux Jacobian governing wave propagation.
For systems of conservation laws, preserving hyperbolicity requires the directional flux Jacobian to have real eigenvalues and a complete set of eigenvectors, ensuring well-posedness of the Cauchy problem and finite speed propagation~\cite{dafermos2005hyperbolic,godlewski2013numerical}.
This requirement is particularly critical in radiative transfer simulations, where transport equations are typically solved using numerical methods supporting high-order discretization, such as \ac{DG} methods~\cite{reed1973triangular,cockburn1989tvb,cockburn1990tvb,cockburn1998runge}, weighted essentially non-oscillatory schemes \cite{balsara2020weno}, and approximate Riemann solvers~\cite{balsara2018hlli,balsara2025multidimensional}, all of which rely on physically consistent wave propagation.

% ML closure literature
Recent studies have explored various \ac{ML}-based approaches for moment closure modeling in kinetic equations and radiative transfer. 
Existing approaches include low-dimensional generalized moment representations for kinetic equations~\cite{han2019uniformly}, direct prediction of closure terms from resolved moment variables~\cite{scoggins2021sprays,bois2022eulerpoisson}, and invariant moment closure models for kinetic equations~\cite{li2023invariance}. 
Other studies approximated entropy-based closures through entropy ansatz coefficients~\cite{porteous2023structure}, convex entropy potentials for entropy variable reconstruction~\cite{schotthofer2021convex}, and convex approximations of regularized entropy minimization for radiative transport~\cite{schotthofer2025regularized}. 
% Recent studies on radiative transfer and kinetic transport introduced gradient-based closures for the spatial derivative of the unclosed higher order moment~\cite{huang2022gradient}, followed by formulations that enforce global hyperbolicity through sufficient coefficient conditions for symmetrizable hyperbolicity and impose additional coefficient conditions to control physical characteristic speeds~\cite{huang2023global,huang2023characteristic}.
% The above studies demonstrate the importance of preserving hyperbolicity and physically admissible characteristic speeds in closure modeling.
% However, a general framework for constructing hyperbolicity-preserving ML closures remains absent.
% Existing approaches typically enforce hyperbolicity through explicit algebraic conditions derived from low-dimensional spectral structures, such as characteristic polynomial or determinant conditions. Because these conditions depend explicitly on the spectral structure of a particular system, extending such constructions to general multidimensional hyperbolic systems becomes increasingly difficult.
Recent studies on radiative transfer and kinetic transport introduced gradient based closures for the spatial derivative of the unclosed higher order moment~\cite{huang2022gradient}, followed by formulations that enforce global hyperbolicity through sufficient coefficient conditions for symmetrizable hyperbolicity and additional constraints on the characteristic speeds~\cite{huang2023global,huang2023characteristic}. More recently, these ideas were extended to multidimensional moment systems by constructing symmetrizable hyperbolic neural closures through matrix parameterizations designed for the underlying system~\cite{huang2026series4}. These studies demonstrate the importance of preserving hyperbolicity and physically admissible characteristic speeds in \ac{ML} closure models.
However, these methods derive hyperbolicity from explicit algebraic conditions that depend on the spectral or matrix structure of a particular system. Extending such constructions to more complex settings, such as higher order moment models or higher dimensional systems, generally requires deriving new algebraic conditions tailored to each case. This motivates the development of a more systematic construction that does not rely on problem specific algebraic conditions.

To develop a hyperbolic closure model, we draw attention to classical symmetrization theory.
From this perspective, symmetrization provides a direct route to hyperbolicity for systems of conservation laws~\cite{dafermos2005hyperbolic,tadmor2016entropy}.
A recent parametric formulation, SymCLaw~\cite{liu2026parametric}, uses entropy symmetrization to parameterize learned fluxes through a strictly convex entropy function and an associated entropy flux potential.
In this work, we propose a hyperbolic neural closure framework for closure modeling in hyperbolic PDE systems, demonstrated on the M1 radiative transfer system.
The closure is constructed by parameterizing the flux Jacobian induced by the unknown closure variable in a symmetrizable form. A structural network produces symmetric matrix factors, while an input convex entropy network with a positive quadratic term produces a strictly positive definite Hessian that acts as a symmetrizer. The resulting closure induced directional Jacobians are similar to symmetric matrices and therefore have real eigenvalues. The radiation pressure tensor is obtained by numerical path integration of this Jacobian field, and local differential features are used only as auxiliary inputs for improving pressure accuracy.
Unlike SymCLaw~\cite{liu2026parametric}, which directly parameterizes the entire flux function, the present work addresses closure modeling where the known physical flux structure is preserved and only the unclosed higher-order moments are learned.
The main contributions of this work are as follows:
\begin{enumerate}
\item \textbf{Hyperbolicity.}
The proposed framework establishes a new approach for constructing ML closures that preserve hyperbolic structure through symmetrizable Jacobian representation, where the closure variable is recovered by numerical path integration.

\item \textbf{Closure accuracy.}
By incorporating gradient-based spatial information unavailable in classical analytic closures, the proposed model improves pressure tensor prediction accuracy, particularly for the off-diagonal component.

\item \textbf{DG solution accuracy and stability.}
The proposed closure model outperforms classical analytic closures in \ac{DG} solution accuracy and remains stable when coupled with a \ac{DG} solver for M1 radiative transfer.
\end{enumerate}

The remainder of this paper is organized as follows. In Sec.~\ref{Sec:preliminaries}, we review the preliminaries on hyperbolic conservation laws, entropy structure, and symmetrization relevant to the M1 radiative transfer system. In Sec.~\ref{Sec:method}, we present the proposed hyperbolic neural closure framework, including the symmetrizable Jacobian construction and the recovery of the radiation pressure tensor. In Sec.~\ref{Sec:numerical_experiments}, we report numerical experiments to assess both closure accuracy and the preservation of hyperbolicity in the resulting moment system. Finally, Sec.~\ref{sec:conclusion} concludes the paper.

\section{Preliminaries}
\label{Sec:preliminaries}
% Preliminareis
% 1.1 General hyperbolic systems of conservation laws
% 1.2 Parametric hyperbolic system
% 2. M1 radiation transfer review
% We begin by reviewing the theoretical foundations for parametric hyperbolic systems of conservation laws~\cite{liu2026parametric}, which later form the basis of the proposed closure model represented in Sec.~\ref{Sec:method}.
We begin by reviewing the theoretical foundations for parametric hyperbolic systems of conservation laws~\cite{liu2026parametric}, which form the basis of the proposed closure model in Sec.~\ref{Sec:method}.

\subsection{Hyperbolic conservation laws}

Our primary goal is to discover a closure model that resides in the hyperbolicity guaranteed admissible set. In this work, the underlying theory used to ensure hyperbolicity at the PDE level plays an important role in the development of the closure model.
To establish this connection, we first consider the general hyperbolic PDEs of the $d$-dimensional system:
\begin{equation}
\frac{\partial \mathbf{u}}{\partial t} + \sum_{i=1}^d \frac{\partial f_i(\mathbf{u})}{\partial x_i} = 0,
\qquad
x \in \Omega \subset \mathbb R^d,
\qquad
t \in (0,T),
\label{eq:prelim_general_claw}
\end{equation}
where $\mathbf{u}=(u_1,\dots,u_p)^\top \in \mathcal D \subset \mathbb R^p$ is the vector of conserved variables and $\mathcal D$ is assumed to be convex. For each $1 \le i \le d$, the flux Jacobian is $A_i(\mathbf{u}):=\nabla_\mathbf{u} f_i(\mathbf{u})= \big\{ \frac{\partial f_i(\mathbf{\mathbf{u}})}{\partial u_j}  \big\}_{1 \leq i,j \leq p} \in \mathbb R^{p\times p}$, and for any unit vector $\mathbf{n}=(n_1,\dots,n_d)^\top \in \mathbb R^{d}$, we define the directional Jacobian $A(\mathbf{u};\mathbf{n}):=\sum_{i=1}^d n_i A_i(\mathbf{u})$.

\begin{definition}[Hyperbolicity]
The system \eqref{eq:prelim_general_claw} is hyperbolic on $\mathcal D$ if, for every $\mathbf{u} \in \mathcal D$ and every $\mathbf{n} \in \mathbb R^{d}$, the matrix $A(\mathbf{u};\mathbf{n})$ has $p$ real eigenvalues and a complete set of linearly independent eigenvectors.
\end{definition}

\subsection{Parametric hyperbolic conservation laws}

For systems of conservation laws, convex entropy functions provide a classical approach for symmetrization.
When the entropy function $\eta(\cdot)$ is strongly convex, we can define the entropy variables $\mathbf{v} := (\nabla_\mathbf{u}\eta(\mathbf{u}))^\top$, where the variables $\mathbf{u}$ and $\mathbf{v}$ are in a one-to-one correspondence.
Writing $g_i(\mathbf{v}):=f_i(\mathbf{u}(\mathbf{v}))$ for $1 \le i \le d$, the conservation law can be recast as
\begin{equation}
\nabla_\mathbf{v} \mathbf{u}(\mathbf{v})\,\frac{\partial \mathbf{v}}{\partial t}
+
\sum_{i=1}^d \nabla_\mathbf{v} g_i(\mathbf{v})\,\frac{\partial \mathbf{v}}{\partial x_i}
=
0.
\label{eq:prelim_symmetric_form}
\end{equation}
Since $\eta(\cdot)$ is strongly convex, its Hessian $\mathcal{H}_\mathbf{u}\eta(\mathbf{u}):=\nabla_\mathbf{u}^2\eta(\mathbf{u})$ is symmetric positive definite, and therefore $\nabla_\mathbf{v} \mathbf{u}(\mathbf{v})=\bigl(\mathcal{H}_\mathbf{u}\eta(\mathbf{u})\bigr)^{-1}$ is also symmetric positive definite.
Here, the entropy function follows the classical entropy symmetrization framework and is introduced solely to construct a positive definite symmetrizer for the learned Jacobian representation.

\begin{theorem}[Symmetrization]
\label{thm:prelim_symmetrization}
Let $\eta(u)$ be strongly convex. Then $\eta(\cdot)$ is an entropy for \eqref{eq:prelim_general_claw} if and only if $\nabla_\mathbf{v}\mathbf{u}(\mathbf{v})$ is symmetric positive definite and, for each $1 \le i \le d$, the matrix $\nabla_\mathbf{v} g_i(\mathbf{v})$ is symmetric. 
In that case, \eqref{eq:prelim_symmetric_form} is a symmetrized form of \eqref{eq:prelim_general_claw}. For every $\mathbf{n} \in \mathbb R^{d}$, the matrix $A(\mathbf{u};\mathbf{n})=\sum_{i=1}^d n_i \nabla_\mathbf{u} f_i(\mathbf{u})$ is similar to the symmetric matrix
\begin{equation}
\bigl(\nabla_\mathbf{u} \mathbf{v}(\mathbf{u})\bigr)^{1/2}
\left(
\sum_{i=1}^d n_i \nabla_\mathbf{v} g_i(\mathbf{v})
\right)
\bigl(\nabla_\mathbf{u} \mathbf{v}(\mathbf{u})\bigr)^{1/2}.
\label{eq:prelim_similarity_matrix}
\end{equation}
Hence, the existence of a strictly convex entropy implies hyperbolicity.
\end{theorem}

\begin{proof}
Since $\mathbf{v}=(\nabla_\mathbf{u} \eta(\mathbf{u}))^\top$, we have $\nabla_\mathbf{u} \mathbf{v}(\mathbf{u})=\mathcal{H}_\mathbf{u}\eta(\mathbf{u})$ and $\nabla_\mathbf{v} \mathbf{u}(\mathbf{v})=\bigl(\nabla_\mathbf{u} \mathbf{v}(\mathbf{u})\bigr)^{-1}$.
For each $i$, by the chain rule,
\begin{equation}
\nabla_\mathbf{v} g_i(\mathbf{v})
=
\nabla_\mathbf{u} f_i(\mathbf{u})\,\nabla_\mathbf{v} \mathbf{u}(\mathbf{v})
=
A_i(\mathbf{u})\,\nabla_\mathbf{v} \mathbf{u}(\mathbf{v}).
\label{eq:prelim_nabla_v_gi}
\end{equation}
Thus
\begin{equation}
A_i(\mathbf{u})
=
\nabla_\mathbf{v} g_i(\mathbf{v})\,\nabla_\mathbf{u} \mathbf{v}(\mathbf{u}).
\label{eq:prelim_Ai_factorization}
\end{equation}
Summing over $i$ with weights $n_i$ gives
\begin{equation}
A(\mathbf{u};\mathbf{n})
=
\sum_{i=1}^d n_i A_i(\mathbf{u})
=
\left(\sum_{i=1}^d n_i \nabla_\mathbf{v} g_i(\mathbf{v})\right)\nabla_\mathbf{u} \mathbf{v}(\mathbf{u}).
\label{eq:prelim_directional_factorization}
\end{equation}
Multiply \eqref{eq:prelim_directional_factorization} on the left by $\bigl(\nabla_\mathbf{u} \mathbf{v}(\mathbf{u})\bigr)^{1/2}$ and on the right by $\bigl(\nabla_\mathbf{u} \mathbf{v}(\mathbf{u})\bigr)^{-1/2}$:
\begin{align}
&\bigl(\nabla_{\mathbf{u}} \mathbf{v}(\mathbf{u})\bigr)^{1/2}
A(\mathbf{u};\mathbf{n})
\bigl(\nabla_{\mathbf{u}} \mathbf{v}(\mathbf{u})\bigr)^{-1/2} \\
&=
\bigl(\nabla_{\mathbf{u}} \mathbf{v}(\mathbf{u})\bigr)^{1/2}
\left(\sum_{i=1}^d n_i \nabla_{\mathbf{v}} \mathbf{g}_i(\mathbf{v})\right)
\nabla_{\mathbf{u}} \mathbf{v}(\mathbf{u})
\bigl(\nabla_{\mathbf{u}} \mathbf{v}(\mathbf{u})\bigr)^{-1/2} \\
&=
\bigl(\nabla_{\mathbf{u}} \mathbf{v}(\mathbf{u})\bigr)^{1/2}
\left(\sum_{i=1}^d n_i \nabla_{\mathbf{v}} \mathbf{g}_i(\mathbf{v})\right)
\bigl(\nabla_{\mathbf{u}} \mathbf{v}(\mathbf{u})\bigr)^{1/2}.
\end{align}
Hence, $A(\mathbf{u};\mathbf{n})$ is similar to the matrix $\bigl(\nabla_{\mathbf{u}} \mathbf{v}(\mathbf{u})\bigr)^{1/2}
\left(\sum_{i=1}^d n_i \nabla_{\mathbf{v}} \mathbf{g}_i(\mathbf{v})\right)
\bigl(\nabla_{\mathbf{u}} \mathbf{v}(\mathbf{u})\bigr)^{1/2}$.

If each $\nabla_{\mathbf{v}} \mathbf{g}_i(\mathbf{v})$ is symmetric, then the matrix $\bigl(\nabla_{\mathbf{u}} \mathbf{v}(\mathbf{u})\bigr)^{1/2}
\left(\sum_{i=1}^d n_i \nabla_{\mathbf{v}} \mathbf{g}_i(\mathbf{v})\right)
\bigl(\nabla_{\mathbf{u}} \mathbf{v}(\mathbf{u})\bigr)^{1/2}$ is symmetric because it is equal to its transpose:
\begin{align}
&\Biggl[
\bigl(\nabla_{\mathbf{u}} \mathbf{v}(\mathbf{u})\bigr)^{1/2}
\left(\sum_{i=1}^d n_i \nabla_{\mathbf{v}} \mathbf{g}_i(\mathbf{v})\right)
\bigl(\nabla_{\mathbf{u}} \mathbf{v}(\mathbf{u})\bigr)^{1/2}
\Biggr]^\top \\
&=
\left(\bigl(\nabla_{\mathbf{u}} \mathbf{v}(\mathbf{u})\bigr)^{1/2}\right)^\top
\left(\sum_{i=1}^d n_i \nabla_{\mathbf{v}} \mathbf{g}_i(\mathbf{v})\right)^\top
\left(  \bigl(\nabla_{\mathbf{u}} \mathbf{v}(\mathbf{u})\bigr)^{1/2} \right)^\top \\
&=
\bigl(\nabla_{\mathbf{u}} \mathbf{v}(\mathbf{u})\bigr)^{1/2}
\left(\sum_{i=1}^d n_i \nabla_{\mathbf{v}} \mathbf{g}_i(\mathbf{v})\right)
\bigl(\nabla_{\mathbf{u}} \mathbf{v}(\mathbf{u})\bigr)^{1/2}.
\end{align}
Hence $A(\mathbf{u};\mathbf{n})$ is similar to a symmetric matrix, and therefore has only real eigenvalues and a complete eigenbasis.
\end{proof}
% \label{sec:M1}

\section{Proposed method: Hyperbolic neural closure for M1 radiation transfer}
\label{Sec:method}

In this section, we briefly review the M1 moment system derived from the radiative transfer equation. We then introduce the proposed closure model, which is derived from the parametric symmetric Jacobian and serves as a closure for the M1 radiation transfer system.

\subsection{Derivation of the M1 system}

We briefly derive the M1 moment system from the Boltzmann transport equation (cf.~\cite{murchikova2017analytic}).
Let $\mathcal{F}(p^\mu,x^\mu)$ denote the radiation distribution function.
In relativistic form, the Boltzmann equation can be written as
\begin{equation}
    \frac{dx^{\alpha}}{d\tau}
    \frac{\partial \mathcal{F}}{\partial x^{\alpha}}
    +
    \frac{dp^{i}}{d\tau}
    \frac{\partial \mathcal{F}}{\partial p^{i}}
    =
    (-p^{\alpha}u_{\alpha})
    \mathcal{S}(p^\mu,x^\mu,\mathcal{F}),
\end{equation}
where $p^\mu$ is the radiation four-momentum, $u^\mu$ is the fluid four-velocity, and $\mathcal{S}$ is the collision term describing interactions of radiation with matter.

In this work, we consider flat Cartesian geometry, neglect external forces,
and take the radiation particles to be massless. Hence
$dp^i/d\tau=0$ and $dx^\alpha/d\tau=p^\alpha$. Substituting into the Boltzmann equation gives
\[
p^\alpha \frac{\partial \mathcal{F}}{\partial x^\alpha}
=
(-p^\alpha u_\alpha)\mathcal{S}.
\]
Writing
$p^\alpha=(p^0,\mathbf{p})$, this becomes
\[
p^0\frac{\partial\mathcal{F}}{\partial t}
+
\mathbf{p}\cdot\nabla_{\mathbf{x}}\mathcal{F}
=
(-p^\alpha u_\alpha)\mathcal{S}.
\]
Using
$\mathbf{p}=p^0\boldsymbol{\omega}$ with
$\boldsymbol{\omega}\in\mathbb{S}^{d-1}$, we obtain
\begin{equation}
    p^0
    \left(
    \frac{\partial \mathcal{F}}{\partial t}
    +
    \boldsymbol{\omega}\cdot\nabla_{\mathbf{x}}\mathcal{F}
    \right)
    =
    (-p^\alpha u_\alpha)\mathcal{S}.
\end{equation}
After dividing by $p^0$, the transport equation becomes
\begin{equation}
    \frac{\partial \mathcal{F}}{\partial t}
    +
    \boldsymbol{\omega}\cdot\nabla_{\mathbf{x}}\mathcal{F}
    =
    \widetilde{\mathcal{S}},
\end{equation}
where
\[
\widetilde{\mathcal{S}}
:=
\frac{-p^\alpha u_\alpha}{p^0}\mathcal{S}.
\]
Taking angular moments over the unit sphere $\mathbb S^{d-1}$ yields the moment system.

The zeroth, first, and second angular moments are defined by
\begin{equation}
    E(\mathbf{x},t)
    := \int_{\mathbb{S}^{d-1}} \mathcal{F}\,d\boldsymbol{\omega},
    \qquad
    \mathbf{F}(\mathbf{x},t)
    := \int_{\mathbb{S}^{d-1}}
    \boldsymbol{\omega}\mathcal{F}\,d\boldsymbol{\omega},
    \qquad
    \mathbf{P}(\mathbf{x},t)
    := \int_{\mathbb{S}^{d-1}}
    \boldsymbol{\omega}\otimes\boldsymbol{\omega}\mathcal{F}\,
    d\boldsymbol{\omega}.
\end{equation}
Taking the zeroth angular moment gives
\begin{equation}
    \frac{\partial E}{\partial t}
    +
    \nabla_{\mathbf{x}}\cdot\mathbf{F}
    =
    S^{(0)},
    \qquad
    S^{(0)}
    :=
    \int_{\mathbb{S}^{d-1}}
    \widetilde{\mathcal{S}}\,d\boldsymbol{\omega}.
\end{equation}
Multiplying the transport equation by $\boldsymbol{\omega}$ and integrating over $\mathbb{S}^{d-1}$ gives
\begin{equation}
    \frac{\partial \mathbf{F}}{\partial t}
    +
    \nabla_{\mathbf{x}}\cdot\mathbf{P}
    =
    \mathbf{S}^{(1)},
    \qquad
    \mathbf{S}^{(1)}
    :=
    \int_{\mathbb{S}^{d-1}}
    \boldsymbol{\omega}\widetilde{\mathcal{S}}\,
    d\boldsymbol{\omega}.
\end{equation}
The two-moment system is not closed because the pressure tensor
$\mathbf{P}$ is a second angular moment and must be expressed in terms of
the lower moments $E$ and $\mathbf{F}$.

The M1 closure assumes that $\mathbf{P}$ depends locally on $E$ and
$\mathbf{F}$. Define the flux factor and flux direction by
\begin{equation}
    f := \frac{\|\mathbf{F}\|}{E},
    \qquad
    \mathbf{n}:=
    \begin{cases}
        \mathbf{F}/\|\mathbf{F}\|, & \|\mathbf{F}\|>0,\\
        \mathbf{0}, & \|\mathbf{F}\|=0.
    \end{cases}
\end{equation}
The physically admissible range is $0\le f\le 1$. The optically thick and
free-streaming limits are
\begin{equation}
    \mathbf{P}_{\mathrm{thick}}
    =
    \frac{E}{d}\mathbf{I},
    \qquad
    \mathbf{P}_{\mathrm{thin}}
    =
    E\,\mathbf{n}\otimes\mathbf{n}.
\end{equation}
The M1 pressure tensor is written as
\begin{equation}
    \mathbf{P}(E,\mathbf{F})
    =
    E\left[
    \frac{1-\chi(f)}{d-1}\mathbf{I}
    +
    \frac{d\chi(f)-1}{d-1}
    \mathbf{n}\otimes\mathbf{n}
    \right],
    \qquad d\ge 2,
\end{equation}
where $\chi(f)$ is the Eddington factor. Equivalently,
\begin{equation}
    \mathbf{P}(E,\mathbf{F})
    =
    \frac{d\chi(f)-1}{d-1}
    \mathbf{P}_{\mathrm{thin}}
    +
    \frac{d[1-\chi(f)]}{d-1}
    \mathbf{P}_{\mathrm{thick}}.
\end{equation}
This form gives $\mathbf{P}=E\mathbf{I}/d$ in the optically thick limit
and $\mathbf{P}=E\mathbf{n}\otimes\mathbf{n}$ in the free-streaming limit.

In three spatial dimensions, analytic M1 closures are commonly written
using $p=\chi(f)$ as the scalar Eddington factor. For example, the
Levermore closure~\cite{levermore1996moment} is
\begin{equation}
    \chi_{\mathrm{Lev}}(f)
    =
    \frac{3+4f^2}{5+2\sqrt{4-3f^2}}.
    \label{eq:Levermore}
\end{equation}
Other analytic closures, such as the Kershaw, Wilson, Minerbo, and Janka closures (cf.~\cite{murchikova2017analytic}), differ by the choice of the scalar function $\chi(f)$.

Neglecting source terms, the homogeneous M1 system can be written as a
system of conservation laws for
$
    \mathbf{u}:=
    (E,\mathbf{F}^{\top})^{\top}
    \in\mathbb{R}^{d+1}.
$
For each coordinate direction $x_i$, the flux is
\begin{equation}
    \mathbf{f}_i(\mathbf{u})
    =
    \begin{bmatrix}
        F_i\\
        P_{1i}(E,\mathbf{F})\\
        \vdots\\
        P_{di}(E,\mathbf{F})
    \end{bmatrix},
    \qquad i=1,\ldots,d.
\end{equation}
Therefore,
\begin{equation}
    \frac{\partial \mathbf{u}}{\partial t}
    +
    \sum_{i=1}^{d}
    \frac{\partial \mathbf{f}_i(\mathbf{u})}{\partial x_i}
    =
    \mathbf{0}.
\end{equation}
Equivalently,
\begin{equation}
    \frac{\partial}{\partial t}
    \begin{bmatrix}
        E\\
        \mathbf{F}
    \end{bmatrix}
    +
    \nabla_{\mathbf{x}}\cdot
    \begin{bmatrix}
        \mathbf{F}\\
        \mathbf{P}(E,\mathbf{F})
    \end{bmatrix}
    =
    \mathbf{0}.
\end{equation}
Thus, the M1 closure problem is the construction of the constitutive map
$(E,\mathbf{F})\mapsto\mathbf{P}(E,\mathbf{F})$.

% \subsection{Hyperbolic conservation laws with gradient-based input augmentation}
\subsection{Hyperbolic neural closure}

Figure~\ref{fig:sensitivity_3dbar_journal_augmented} shows that the off-diagonal component $P_{xy}$ is significantly more sensitive to gradient-based features such as $\nabla\!\cdot\!\mathbf{F}$, $(\nabla\times\mathbf{F})_z$, and $\|\nabla E\|^2$ than to the local state variables alone. 
The gradient-based features provide information beyond the local moment state $(E,\mathbf{F})$, which is not accounted for in classical analytic closures such as Eq.~\ref{eq:Levermore}.
Accordingly, we introduce a collection of local differential features $\mathcal{G}(\mathbf{u})$ and incorporate them only in the closure term of the flux.

Based on the Eq.~\eqref{eq:prelim_general_claw}, we write the augmented system as
\begin{equation}
\frac{\partial \mathbf{u}}{\partial t}
+
\sum_{i=1}^d \frac{\partial \widehat f_i\bigl(\mathbf{u},\mathcal{G}(\mathbf{u})\bigr)}{\partial x_i}
=0,
\qquad
x \in \Omega \subset \mathbb R^d,
\qquad
t \in (0,T),
\label{eq:prelim_general_claw_augmented}
\end{equation}
where $\widehat f_i\bigl(\mathbf{u},\mathcal{G}(\mathbf{u})\bigr)$ is chosen so that the components already determined by $\mathbf{u}$ remain unchanged.
For the $M_1$ moment system with $\mathbf{u}=(E,\mathbf{F})^\top$, the energy flux remains $F_i$, while the pressure tensor closure is augmented as $\widetilde{\mathbf{P}}=\widetilde{\mathbf{P}}\bigl(\mathbf{u},\mathcal{G}(\mathbf{u})\bigr)$. Hence the $i$th flux becomes
\begin{equation}
\widehat f_i\bigl(\mathbf{u},\mathcal{G}(\mathbf{u})\bigr)
=
\left[
F_i,\,
\widetilde P_{1i}\bigl(\mathbf{u},\mathcal{G}(\mathbf{u})\bigr),\,
\ldots,\,
\widetilde P_{di}\bigl(\mathbf{u},\mathcal{G}(\mathbf{u})\bigr)
\right]^\top,
\qquad 1 \le i \le d.
\end{equation}
Equivalently,
\begin{equation}
\frac{\partial}{\partial t}
\begin{bmatrix}
E\\
\mathbf{F}
\end{bmatrix}
+
\nabla_{\mathbf{x}}\cdot
\begin{bmatrix}
\mathbf{F}\\
\widetilde{\mathbf{P}}\bigl(\mathbf{u},\mathcal{G}(\mathbf{u})\bigr)
\end{bmatrix}
=0.
\label{eq:m1_block_augmented}
\end{equation}

To analyze hyperbolicity, we freeze the differential feature vector and regard it as a parameter $\boldsymbol{\gamma}\in\Gamma$. For each fixed $\boldsymbol{\gamma}$, we consider the parametric conservation law
\begin{equation}
\frac{\partial \mathbf{u}}{\partial t}
+
\sum_{i=1}^d \frac{\partial \widehat f_i(\mathbf{u},\boldsymbol{\gamma})}{\partial x_i}
=0,
\qquad
x \in \Omega \subset \mathbb R^d,
\qquad
t \in (0,T),
\label{eq:prelim_general_claw_frozen}
\end{equation}
with flux Jacobians $\widehat A_i(\mathbf{u},\boldsymbol{\gamma}):=\nabla_{\mathbf{u}}\widehat f_i(\mathbf{u},\boldsymbol{\gamma})$ and directional Jacobian
\begin{equation}
\widehat A(\mathbf{u},\boldsymbol{\gamma};\mathbf{n})
:=
\sum_{i=1}^d n_i \widehat A_i(\mathbf{u},\boldsymbol{\gamma}),
\qquad
\mathbf{n}\in\mathbb{R}^d.
\end{equation}
In the present setting, $\boldsymbol{\gamma}$ is identified locally with the frozen value of $\mathcal{G}(\mathbf{u})$.
Then \eqref{eq:prelim_general_claw_frozen} can be recast as
\begin{equation}
\nabla_{\mathbf{v}} \mathbf{u}(\mathbf{v})\,\frac{\partial \mathbf{v}}{\partial t}
+
\sum_{i=1}^d
\nabla_{\mathbf{v}} \widehat g_i(\mathbf{v},\boldsymbol{\gamma})\,\frac{\partial \mathbf{v}}{\partial x_i}
=0.
\label{eq:prelim_parametric_symmetric_form}
\end{equation}

\begin{theorem}[Frozen-parameter symmetrization]
\label{thm:prelim_parametric_symmetrization}
Let $\boldsymbol{\gamma}\in\Gamma$ be fixed and let $\eta(\cdot)$ be strictly convex. Then $\eta(\cdot)$ is an entropy for \eqref{eq:prelim_general_claw_frozen} if and only if $\nabla_{\mathbf{v}}\mathbf{u}(\mathbf{v})$ is symmetric positive definite and, for each $1\le i\le d$, the matrix $\nabla_{\mathbf{v}} \widehat g_i(\mathbf{v},\boldsymbol{\gamma})$ is symmetric. In that case, \eqref{eq:prelim_parametric_symmetric_form} is a symmetrized form of \eqref{eq:prelim_general_claw_frozen}, and for every $\mathbf{n}\in\mathbb{R}^d$, the matrix $\widehat A(\mathbf{u},\boldsymbol{\gamma};\mathbf{n})$ is similar to the symmetric matrix
\begin{equation}
\bigl(\nabla_{\mathbf{u}} \mathbf{v}(\mathbf{u})\bigr)^{1/2}
\left(
\sum_{i=1}^d n_i \nabla_{\mathbf{v}} \widehat g_i(\mathbf{v},\boldsymbol{\gamma})
\right)
\bigl(\nabla_{\mathbf{u}} \mathbf{v}(\mathbf{u})\bigr)^{1/2}.
\label{eq:prelim_parametric_similarity_matrix}
\end{equation}
Hence, for each frozen $\boldsymbol{\gamma}$, the existence of a strictly convex entropy implies hyperbolicity.
\end{theorem}

\begin{proof}
Since $\mathbf{v}=(\nabla_{\mathbf{u}}\eta(\mathbf{u}))^\top$, we have
\begin{equation}
\nabla_{\mathbf{u}} \mathbf{v}(\mathbf{u})
=
\mathcal{H}_{\mathbf{u}}\eta(\mathbf{u}),
\qquad
\nabla_{\mathbf{v}} \mathbf{u}(\mathbf{v})
=
\bigl(\nabla_{\mathbf{u}} \mathbf{v}(\mathbf{u})\bigr)^{-1}.
\end{equation}
For each $i$, the chain rule gives
\begin{equation}
\nabla_{\mathbf{v}} \widehat g_i(\mathbf{v},\boldsymbol{\gamma})
=
\nabla_{\mathbf{u}} \widehat f_i(\mathbf{u},\boldsymbol{\gamma})\,\nabla_{\mathbf{v}} \mathbf{u}(\mathbf{v})
=
\widehat A_i(\mathbf{u},\boldsymbol{\gamma})\,\nabla_{\mathbf{v}} \mathbf{u}(\mathbf{v}),
\label{eq:prelim_parametric_nabla_v_gi}
\end{equation}
hence
\begin{equation}
\widehat A_i(\mathbf{u},\boldsymbol{\gamma})
=
\nabla_{\mathbf{v}} \widehat g_i(\mathbf{v},\boldsymbol{\gamma})\,\nabla_{\mathbf{u}} \mathbf{v}(\mathbf{u}).
\label{eq:prelim_parametric_Ai_factorization}
\end{equation}
Therefore,
\begin{equation}
\widehat A(\mathbf{u},\boldsymbol{\gamma};\mathbf{n})
=
\left(
\sum_{i=1}^d n_i \nabla_{\mathbf{v}} \widehat g_i(\mathbf{v},\boldsymbol{\gamma})
\right)
\nabla_{\mathbf{u}} \mathbf{v}(\mathbf{u}).
\label{eq:prelim_parametric_directional_factorization}
\end{equation}
Multiplying \eqref{eq:prelim_parametric_directional_factorization} on the left by $\bigl(\nabla_{\mathbf{u}} \mathbf{v}(\mathbf{u})\bigr)^{1/2}$ and on the right by $\bigl(\nabla_{\mathbf{u}} \mathbf{v}(\mathbf{u})\bigr)^{-1/2}$ yields \eqref{eq:prelim_parametric_similarity_matrix}. If each $\nabla_{\mathbf{v}} \widehat g_i(\mathbf{v},\boldsymbol{\gamma})$ is symmetric, then the matrix in \eqref{eq:prelim_parametric_similarity_matrix} is symmetric. Hence $\widehat A(\mathbf{u},\boldsymbol{\gamma};\mathbf{n})$ is similar to a symmetric matrix and therefore has only real eigenvalues and a complete eigenbasis. 
% This implies that the learned closure is constructed through a Jacobian representation that satisfies the hyperbolicity condition at each local feature state.
% This implies that the learned closure is constructed through a Jacobian representation constrained by the symmetrization structure.
% This proves hyperbolicity for each frozen $\boldsymbol{\gamma}$.
This proves hyperbolicity for the parametric system corresponding to each frozen feature vector $\gamma$.
\end{proof}

\begin{figure}[htp!]
    \centering
    \begin{subfigure}[b]{0.8\linewidth}
        \centering
        \includegraphics[width=\linewidth]{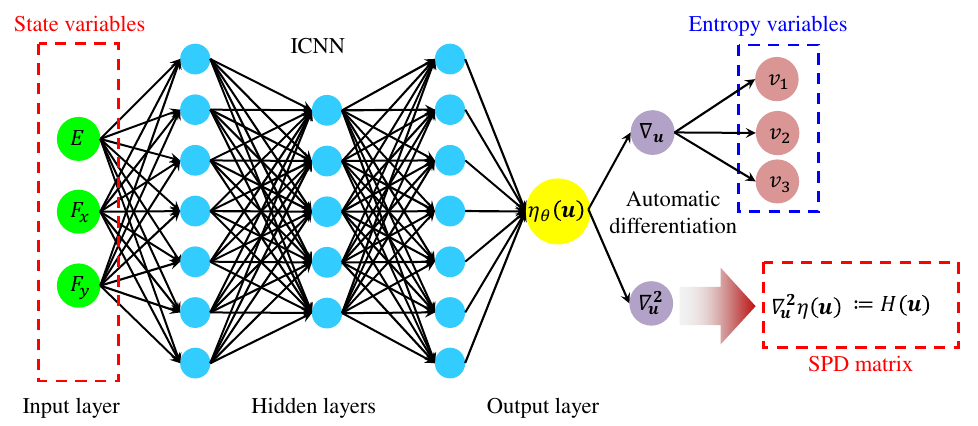}
        \subcaption{Entropy neural network}
    \end{subfigure}  
        \begin{subfigure}[b]{0.8\linewidth}
        \centering
        \includegraphics[width=\linewidth]{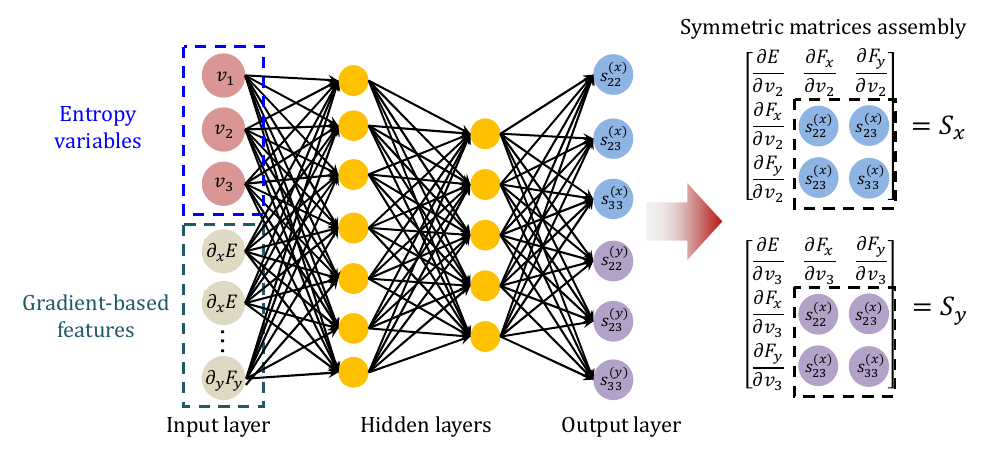}
        \subcaption{Symmetric neural network}
    \end{subfigure}  
\caption{Components of the proposed hyperbolic neural closure for the two-dimensional M1 system with $\mathbf{u}=(E,F_x,F_y)^\top\in\mathbb{R}^3$. 
(a) Input convex neural network (ICNN) that maps state variables to entropy, whose gradient and Hessian provide entropy variables and a symmetric positive definite (SPD) matrix. 
(b) Symmetric neural network that takes entropy variables and gradient-based features as inputs and constructs symmetric flux Jacobians.}
\label{Fig:neural_networks}
\end{figure}

Theorem~\ref{thm:prelim_parametric_symmetrization} establishes hyperbolicity for the parametric system
associated with a fixed feature vector $\gamma$.
When the feature vector is instantiated as local differential features $\gamma=G(\mathbf{u})$, the resulting closure becomes a gradient-dependent model.
A rigorous well-posedness analysis of the corresponding full closed \ac{PDE} system is beyond the scope of the present work.

% Based on Theorem~\ref{thm:prelim_parametric_symmetrization}, we introduce a neural closure framework that preserves hyperbolicity for the M1 radiative transfer system. 
Based on Theorem~\ref{thm:prelim_parametric_symmetrization}, we introduce a neural closure framework that constructs locally symmetrizable closure Jacobians for the M1 radiative transfer system.
Let $\mathbf{u}=[E,F_x,F_y]^\top$ and $\mathcal{G}(\mathbf{u})=[\partial_x E,\partial_y E,\partial_x F_x,\partial_y F_x,\partial_x F_y,\partial_y F_y]^\top \in \mathbb{R}^6$. 
We parameterize the entropy network as an input convex neural network (ICNN) represented in Fig.~\ref{Fig:neural_networks}a:
\begin{align}
\mathbf{z}^{(0)} &= \mathbf{u}, \\
\mathbf{z}^{(l+1)} 
&= \sigma\!\big(
W^{(l)} \mathbf{z}^{(l)} 
+ U^{(l)} \mathbf{u} 
+ \mathbf{b}^{(l)}
\big),
\qquad l = 0,\dots,L-1,
\\
\eta_\theta(\mathbf{u})
&= \mathbf{w}_\eta^\top \mathbf{z}^{(L)}
+ \frac{\alpha}{2}\|\mathbf{u}\|_2^2
+ b_{\eta,0},
\end{align}
subject to
\begin{equation}
W^{(l)} \ge 0, 
\quad
\mathbf{w}_\eta \ge 0,
\quad
\alpha > 0.
\end{equation}
where $\sigma(z)=\log(1+e^z)$ is the softplus activation. We define the entropy variables and Hessian as
$\mathbf{v}(\mathbf{u})=\nabla_{\mathbf{u}}\eta_\theta(\mathbf{u})$,
$\mathbf{H}(\mathbf{u})=\nabla_{\mathbf{u}}^2 \eta_\theta(\mathbf{u})$,
and
$\nabla_{\mathbf{v}}\mathbf{u}(\mathbf{v})=\mathbf{H}(\mathbf{u})^{-1}$.
Under the ICNN architecture, $\mathbf{H}(\mathbf{u})$ is SPD due to the convex network structure, and the additional quadratic term guarantees strict positive definiteness. Therefore, with $\alpha>0$, we obtain $\mathbf{H}(\mathbf{u})\succ0$, which implies that the mapping $\mathbf{u}\mapsto\mathbf{v}$ is one-to-one. In the numerical experiments of this work, we use a single hidden layer ICNN for simplicity.

To parameterize the symmetrizable Jacobian representation introduced in Theorem~\ref{thm:prelim_parametric_symmetrization}, we introduce a second neural network that predicts the free coefficients in the symmetric matrices in Fig.~\ref{Fig:neural_networks}b.
The structural network takes the augmented input $\mathbf{z}(\mathbf{u},\mathcal{G}(\mathbf{u}))=[\mathbf{v}(\mathbf{u}),\mathcal{G}(\mathbf{u})]^\top \in \mathbb{R}^9$ and outputs six coefficients $\mathcal{N}_S(\mathbf{z}(\mathbf{u},\mathcal{G}(\mathbf{u})))=[s_{22}^{(x)},s_{23}^{(x)},s_{33}^{(x)},s_{22}^{(y)},s_{23}^{(y)},s_{33}^{(y)}]^\top$. 
To account for the additive constant introduced by the line integration, we introduce an anchor network 
$\mathcal{N}_A(\mathbf{z}(\mathbf{0},\mathbf{0}))=[P_{xx}^0,P_{yy}^0,P_{xy}^0]^\top$, 
which predicts the reference pressure values used as integration constants in the closure reconstruction.

The symmetric matrices $S_x(\mathbf{u},\mathcal{G}(\mathbf{u}))$ and $S_y(\mathbf{u},\mathcal{G}(\mathbf{u}))$ are defined by
\begin{equation}
S_x
=
\begin{bmatrix}
\frac{\partial E}{\partial v_2} & \frac{\partial F_x}{\partial v_2} & \frac{\partial F_y}{\partial v_2}\\
\frac{\partial F_x}{\partial v_2} & s_{22}^{(x)} & s_{23}^{(x)}\\
\frac{\partial F_y}{\partial v_2} & s_{23}^{(x)} & s_{33}^{(x)}
\end{bmatrix},
\qquad
S_y
=
\begin{bmatrix}
\frac{\partial E}{\partial v_3} & \frac{\partial F_x}{\partial v_3} & \frac{\partial F_y}{\partial v_3}\\
\frac{\partial F_x}{\partial v_3} & s_{22}^{(y)} & s_{23}^{(y)}\\
\frac{\partial F_y}{\partial v_3} & s_{23}^{(y)} & s_{33}^{(y)}
\end{bmatrix},
\label{eq:m1_Sx_Sy_explicit}
\end{equation}
and the flux Jacobians are parameterized as $J_x(\mathbf{u},\mathcal{G}(\mathbf{u})) = S_x(\mathbf{u},\mathcal{G}(\mathbf{u}))\,H(\mathbf{u})$ and $J_y(\mathbf{u},\mathcal{G}(\mathbf{u})) = S_y(\mathbf{u},\mathcal{G}(\mathbf{u}))\,H(\mathbf{u})$.
% By construction, the first rows satisfy $(J_x)_{1,:}=[0,1,0]$ and $(J_y)_{1,:}=[0,0,1]$, so the energy flux remains $(F_x,F_y)$.
Since the M1 system has the fixed flux structure $f_E=(F_x,F_y)$, the first row of each Jacobian is constrained to preserve this exact relation. As a result, $(J_x)_{1,:}=[0,1,0]$ and $(J_y)_{1,:}=[0,0,1]$, so only the closure part of the flux is learned while the energy flux remains unchanged.
Utilizing the relation $\nabla_{\mathbf{v}}\mathbf{u}(\mathbf{v})=H(\mathbf{u})^{-1}$, the terms $\frac{\partial E}{\partial v_j}$, $\frac{\partial F_x}{\partial v_j}$, and $\frac{\partial F_y}{\partial v_j}$ are computed explicitly as the entries of the inverse Hessian $\bigl(\nabla_{\mathbf{u}}^2\eta_\theta(\mathbf{u})\bigr)^{-1}$.

Using the straight-line path $\mathbf{r}(t)=t\mathbf{u}$, $t\in[0,1]$,
% we reconstruct the pressure components by line integration:
we reconstruct the pressure components along a prescribed straight-line integration path:
\begin{equation}
\begin{aligned}
P_{xx}(\mathbf{u},\mathcal{G}(\mathbf{u}))
&=
P_{xx}^0
+
\int_0^1 \nabla_{\mathbf{u}} P_{xx}(t\mathbf{u},\mathcal{G}(\mathbf{u}))\cdot \mathbf{u}\,dt, \\
P_{xy}^{(x)}(\mathbf{u},\mathcal{G}(\mathbf{u}))
&=
P_{xy}^0
+
\int_0^1 \nabla_{\mathbf{u}} P_{xy}^{(x)}(t\mathbf{u},\mathcal{G}(\mathbf{u}))\cdot \mathbf{u}\,dt, \\
P_{xy}^{(y)}(\mathbf{u},\mathcal{G}(\mathbf{u}))
&=
P_{xy}^0
+
\int_0^1 \nabla_{\mathbf{u}} P_{xy}^{(y)}(t\mathbf{u},\mathcal{G}(\mathbf{u}))\cdot \mathbf{u}\,dt, \\
P_{yy}(\mathbf{u},\mathcal{G}(\mathbf{u}))
&=
P_{yy}^0
+
\int_0^1 \nabla_{\mathbf{u}} P_{yy}(t\mathbf{u},\mathcal{G}(\mathbf{u}))\cdot \mathbf{u}\,dt,
\end{aligned}
\label{eq:m1_pressure_line_integrals}
\end{equation}
and define the final off-diagonal closure by $P_{xy}(\mathbf{u},\mathcal{G}(\mathbf{u}))=\frac12\bigl(P_{xy}^{(x)}(\mathbf{u},\mathcal{G}(\mathbf{u}))+P_{xy}^{(y)}(\mathbf{u},\mathcal{G}(\mathbf{u}))\bigr)$.
Since the learned Jacobian field is parameterized directly through neural networks,
exact integrability conditions are not explicitly enforced.
In the present work, the path integration serves as a practical reconstruction
procedure for recovering closure quantities, and numerical experiments show that
the resulting closure remains stable and accurate in DG simulations.

In implementation, the line integrals are evaluated by midpoint quadrature with $t_k = \frac{k-\tfrac12}{N_q}$ for $k=1,\dots,N_q$, so that $P_{xx}(\mathbf{u},\mathcal{G}(\mathbf{u})) \approx P_{xx}^0 + \frac{1}{N_q}\sum_{k=1}^{N_q} \nabla_{\mathbf{u}}P_{xx}(t_k\mathbf{u},\mathcal{G}(\mathbf{u}))\cdot \mathbf{u}$, and analogously for $P_{xy}^{(x)}$, $P_{xy}^{(y)}$, and $P_{yy}$.

For each frozen value $\boldsymbol{\gamma}=\mathcal{G}(\mathbf{u})$, the matrices $S_x(\mathbf{u},\boldsymbol{\gamma})$ and $S_y(\mathbf{u},\boldsymbol{\gamma})$ are symmetric by construction, and therefore $\nabla_{\mathbf{v}}\widehat g_x(\mathbf{v},\boldsymbol{\gamma})=S_x(\mathbf{u},\boldsymbol{\gamma})$ and $\nabla_{\mathbf{v}}\widehat g_y(\mathbf{v},\boldsymbol{\gamma})=S_y(\mathbf{u},\boldsymbol{\gamma})$ are symmetric. Since $\eta_\theta(\mathbf{u})$ is chosen to be strictly convex, Theorem~\ref{thm:prelim_parametric_symmetrization} implies that the resulting frozen-parameter system is symmetrizable hyperbolic. 
Hence the pressure closures $P_{xx}(\mathbf{u},\mathcal{G}(\mathbf{u}))$, 
$P_{yy}(\mathbf{u},\mathcal{G}(\mathbf{u}))$, and 
$P_{xy}(\mathbf{u},\mathcal{G}(\mathbf{u}))$ are reconstructed from a Jacobian representation that satisfies the hyperbolicity condition in the frozen-feature setting.

To regularize the wave-speed bound during training, for each unit direction $\mathbf{n}\in\mathcal{D}$ we construct the directional symmetric matrix as $\mathbf{S}{\mathbf{n}} = n_x \mathbf{S}x + n_y \mathbf{S}y$. Since $\mathbf{H}\succ 0$, we use its Cholesky factorization $\mathbf{H}=\mathbf{L}\mathbf{L}^{\top}$ and define the symmetrized directional operator $\mathbf{B}{\mathbf{n}}=\mathbf{L}^{\top}\mathbf{S}{\mathbf{n}}\mathbf{L}$. 
The directional wave speed is computed from the symmetrized operator as
$a_{\mathbf{n}} = \max_i \left| \lambda_i \left( \mathbf{B}_{\mathbf{n}} \right) \right|$,
and violations of the target speed limit $c$ are measured using a smooth hinge function
$\phi_{\tau}(a_{\mathbf{n}}; c) = \tau \, \mathrm{softplus} \left( \frac{a_{\mathbf{n}} - c}{\tau} \right)$,
where $\tau > 0$ controls the smoothness.
The wave-speed penalty is averaged over the direction set and added to the training objective as
$
\mathcal{L}{\mathrm{ws}}
=\frac{1}{|\mathcal{D}|}\sum{\mathbf{n}\in\mathcal{D}}
\phi_{\tau}(a_{\mathbf{n}};c)^2.
$
To ensure scale consistency, the input variables are nondimensionalized as 
$F_i/(cE)$, $\partial_i E/(\sigma_t E)$, and $\partial_i F_j/(\sigma_t cE)$, for $i,j\in\{x,y\}$.

\begin{algorithm}[htp!]
\caption{Training procedure for the hyperbolic neural closure}
\label{alg:hpnn_training}

\KwRequire{Training data $\mathcal{S}=\{(\mathbf{u}^m,\mathbf{G}^m,\mathbf{P}^m)\}_{m=1}^{N_s}$.}
\KwEnsure{Hyperbolic neural closure $\widehat{\mathbf{P}}(\mathbf{u},\mathbf{G})$.}

\textcolor{gray}{$\triangleright$ Data preparation and initialization}\; 
Initialize $\Theta$ for $\eta_\theta$, $\mathcal{N}_S$, and $\mathcal{N}_A$\;
Set $N_q$, direction set $\mathcal{D}$, wave speed bound $c$\;

\While{stopping criterion is not satisfied}{
    Sample a mini batch $\mathcal{B}\subset\mathcal{S}$\;

    Compute $\mathbf{v}=\nabla_{\mathbf{u}}\eta_\theta(\mathbf{u})$ and
    $\mathbf{H}=\nabla_{\mathbf{u}}^2\eta_\theta(\mathbf{u})\succ0$
    for all $(\mathbf{u},\mathbf{G},\mathbf{P})\in\mathcal{B}$\;

    Form $\mathbf{z}=[\mathbf{v},\mathbf{G}]^\top$ and predict
    $\mathbf{S}_x(\mathbf{u},\mathbf{G})$ and $\mathbf{S}_y(\mathbf{u},\mathbf{G})$
    over $\mathcal{B}$ \graycomment{By Eq.~\eqref{eq:m1_Sx_Sy_explicit}}

    Form $\mathbf{J}_x=\mathbf{S}_x\mathbf{H}$ and
    $\mathbf{J}_y=\mathbf{S}_y\mathbf{H}$ over $\mathcal{B}$

    Extract $\nabla_{\mathbf{u}}P_{xx}$,
    $\nabla_{\mathbf{u}}P_{xy}^{(x)}$,
    $\nabla_{\mathbf{u}}P_{xy}^{(y)}$, and
    $\nabla_{\mathbf{u}}P_{yy}$\;

    Reconstruct $\widehat{P}_{ij}$ by midpoint integration with $N_q$ points
    \graycomment{By Eq.~\eqref{eq:m1_pressure_line_integrals}}

    Set $\widehat{P}_{xy}=\frac{1}{2}
    (\widehat{P}_{xy}^{(x)}+\widehat{P}_{xy}^{(y)})$\;

    Compute $\mathcal{L}_{\rm data}$ over $\mathcal{B}$\;

    Compute $\mathcal{L}_{\rm ws}$ over $\mathcal{B}$ and $\mathcal{D}$\;

    Update $\Theta$ by minimizing
    $\mathcal{L}=\mathcal{L}_{\rm data}+\rho_{\rm ws}\mathcal{L}_{\rm ws}$\;
}

\Return $\widehat{\mathbf{P}}(\mathbf{u},\mathbf{G})$\;
\end{algorithm}

\subsection{Modal discontinuous Galerkin discretization}

Consider the two-dimensional conservation law
\begin{equation}
\frac{\partial \mathbf{u}}{\partial t} 
+ \frac{\partial \mathbf{f}^{(x)}(\mathbf{u})}{\partial x}
+ \frac{\partial \mathbf{f}^{(y)}(\mathbf{u})}{\partial y}
= \mathbf{s}(\mathbf{u},\mathbf{x}),
\qquad
\mathbf{u}:\Omega\times(0,T)\to\mathbb{R}^m.
\label{eq:dg_general_system}
\end{equation}
Let $\Omega=\bigcup_{K\in\mathcal{T}_h}K$ be a rectangular mesh. For each element $K$, with reference map $x=x_K+\frac{h_x^K}{2}\xi$, $y=y_K+\frac{h_y^K}{2}\eta$, and $(\xi,\eta)\in[-1,1]^2$, we define $V_h^K=\mathrm{span}\{\psi_{ij}(\xi,\eta)=P_i(\xi)P_j(\eta):0\le i\le p_x,\ 0\le j\le p_y\}$ and expand $\mathbf{u}_h|_K(\mathbf{x},t)=\sum_{i=0}^{p_x}\sum_{j=0}^{p_y}\mathbf{u}_{ij}^K(t)\,\psi_{ij}(\xi,\eta)$.
For $\mathbf{v}_h\in(V_h^K)^m$, the local weak form is
\begin{align}
\int_K \frac{\partial \mathbf{u}_h}{\partial t}\cdot\mathbf{v}_h\,d\mathbf{x}
-\int_K \mathbf{f}^{(x)}(\mathbf{u}_h)\cdot\partial_x \mathbf{v}_h\,d\mathbf{x}
-\int_K \mathbf{f}^{(y)}(\mathbf{u}_h)\cdot\partial_y \mathbf{v}_h\,d\mathbf{x}
+\int_{\partial K}\widehat{\mathbf{f}}(\mathbf{u}_h^-,\mathbf{u}_h^+;\mathbf{n})\cdot\mathbf{v}_h^-\,ds
\nonumber\\
=\int_K \mathbf{s}(\mathbf{u}_h,\mathbf{x})\cdot\mathbf{v}_h\,d\mathbf{x}.
\label{eq:dg_weak_form}
\end{align}

The one-dimensional Legendre matrices are $M_{ij}^x=\int_{-1}^1 P_i(\xi)P_j(\xi)\,d\xi=\frac{2}{2i+1}\delta_{ij}$ and $S_{ij}^x=\int_{-1}^1 P_i'(\xi)P_j(\xi)\,d\xi$, with analogous definitions for $\mathbf{M}^y$ and $\mathbf{S}^y$. Hence
\begin{equation}
\mathbf{M}_K=\frac{h_x^Kh_y^K}{4}\,(\mathbf{M}^x\otimes\mathbf{M}^y),
\qquad
\mathbf{G}_{x,K}=\frac{h_y^K}{2}\,(\mathbf{S}^x\otimes\mathbf{M}^y),
\qquad
\mathbf{G}_{y,K}=\frac{h_x^K}{2}\,(\mathbf{M}^x\otimes\mathbf{S}^y).
\label{eq:dg_element_matrices}
\end{equation}

By substituting the modal expansion into the weak form and evaluating the resulting basis integrals using the element matrices in \eqref{eq:dg_element_matrices}, we obtain the semidiscrete system
\begin{equation}
\mathbf{M}_K \frac{d\mathbf{U}_K}{dt}
=
\mathbf{R}_K^{\mathrm{vol}}(\mathbf{U}_K)
+\mathbf{R}_K^{\mathrm{src}}(\mathbf{U}_K)
+\mathbf{R}_K^{\mathrm{face}}(\mathbf{U}_{K^-},\mathbf{U}_{K^+}),
\label{eq:dg_element_semidiscrete}
\end{equation}
where $\mathbf{U}_K$ contains all modal coefficients $\mathbf{u}_{ij}^K$, and
$\mathbf{R}_K^{\mathrm{vol}}$, $\mathbf{R}_K^{\mathrm{src}}$, and $\mathbf{R}_K^{\mathrm{face}}$
denote the volume, source, and interface residual contributions, respectively.

To assess the numerical flux on element interfaces, we first evaluate the interior modal representation on the corresponding element faces.
For a vertical face $\xi=\pm1$, the trace along the polynomial basis is
$\bigl(\mathbf{R}_{x,\pm}\mathbf{U}_K\bigr)_j=\sum_{i=0}^{p_x}\mathbf{u}_{ij}^K P_i(\pm1)$, and for a horizontal face $\eta=\pm1$, 
$\bigl(\mathbf{R}_{y,\pm}\mathbf{U}_K\bigr)_i=\sum_{j=0}^{p_y}\mathbf{u}_{ij}^K P_j(\pm1)$.

We now incorporate the resulting interface flux contribution back into the element residual. We define the associated lift operators as
$\mathbf{L}_{x,\pm}=\mathbf{R}_{x,\pm}^{\top}\,\frac{h_y^K}{2}\mathbf{M}^y$ and 
$\mathbf{L}_{y,\pm}=\mathbf{R}_{y,\pm}^{\top}\,\frac{h_x^K}{2}\mathbf{M}^x$.
Using the numerical flux evaluated from the traces of neighboring elements, the interface residual contribution for an interior face $e$ shared by $K^-$ and $K^+$ is given by
\begin{equation}
\mathbf{R}_{K^-}^{e,\mathrm{face}}
=
\mathbf{L}_{K^-}^{e}\,
\widehat{\mathbf{F}}^e(\mathbf{R}_{K^-}^{e}\mathbf{U}_{K^-},\mathbf{R}_{K^+}^{e}\mathbf{U}_{K^+};\mathbf{n}_e),
\label{eq:dg_face_trace_lift}
\end{equation}
with the analogous formula on $K^+$.
We use a local Lax--Friedrichs flux
\begin{equation}
\widehat{\mathbf{f}}(\mathbf{u}^-,\mathbf{u}^+;\mathbf{n})
=
\frac12\Bigl(\mathbf{f}_{\mathbf{n}}(\mathbf{u}^-)+\mathbf{f}_{\mathbf{n}}(\mathbf{u}^+)\Bigr)
-\frac{a_e}{2}\,(\mathbf{u}^+-\mathbf{u}^-),
\qquad
\mathbf{f}_{\mathbf{n}}(\mathbf{u})=n_x\mathbf{f}^{(x)}(\mathbf{u})+n_y\mathbf{f}^{(y)}(\mathbf{u}).
\label{eq:dg_rusanov_flux}
\end{equation}
Substitution of \eqref{eq:dg_face_trace_lift} into \eqref{eq:dg_element_semidiscrete} yields
\begin{align}
\mathbf{M}_K\frac{d\mathbf{U}_K}{dt}
=
\mathbf{R}_K^{\mathrm{vol}}(\mathbf{U}_K)
+\mathbf{R}_K^{\mathrm{src}}(\mathbf{U}_K)
-\sum_{e\subset\partial K}\mathbf{L}_K^e
\Bigl[
\tfrac12\bigl(\mathbf{f}_{\mathbf{n}_e}(\mathbf{U}_{K^-}^e)+\mathbf{f}_{\mathbf{n}_e}(\mathbf{U}_{K^+}^e)\bigr)
-\tfrac{a_e}{2}\bigl(\mathbf{U}_{K^+}^e-\mathbf{U}_{K^-}^e\bigr)
\Bigr].
\label{eq:dg_semidiscrete_residual}
\end{align}

By grouping the modal coefficients associated with each physical variable separately, we write
$\mathbf{E}_K$, $\mathbf{F}_{x,K}$, and $\mathbf{F}_{y,K}$ for the element coefficient vectors of $E$, $F_x$, and $F_y$, respectively. Here, $\mathbf{M}_K$ denotes the element mass matrix, $\mathbf{G}_{x,K}$ and $\mathbf{G}_{y,K}$ denote the derivative operators in the $x$- and $y$-directions, and $\mathbf{r}_{K}^{\mathrm{face}}$ represents the interface flux residual contribution.
For the M1 system, with $\mathbf{u}=(E,F_x,F_y)^\top$, the HN closure predicts 
\[
\mathbf{P}_{\theta,K}
=
\mathcal{P}_{\theta}(\bar{E}_K,\bar{\mathbf{F}}_K,\mathbf{G}_K)
=
\begin{bmatrix}
P_{xx,K} & P_{xy,K}\\
P_{xy,K} & P_{yy,K}
\end{bmatrix},
\]
where 
\[
\bar{\mathbf{F}}_K
=
\bigl((F_x)_{00}^K,(F_y)_{00}^K\bigr)^\top,
\]
with $(F_x)_{00}^K$ and $(F_y)_{00}^K$ corresponding to the element-averaged flux components, and defines
\[
\mathbf{D}_{\theta,K}
=
\bar{E}_K^{-1}\mathbf{P}_{\theta,K}.
\]
The elementwise specialization used in the solver is
\begin{equation}
\sigma_{a,K}\mathbf{M}_K\mathbf{E}_K
-\mathbf{G}_{x,K}\mathbf{F}_{x,K}
-\mathbf{G}_{y,K}\mathbf{F}_{y,K}
=
\mathbf{M}_K\mathbf{q}_K-\mathbf{r}_{E,K}^{\mathrm{face}},
\label{eq:dg_m1_phi}
\end{equation}
\begin{equation}
\sigma_{t,K}\mathbf{M}_K\mathbf{F}_{x,K}
-\bigl((D_{\theta,K})_{11}\mathbf{G}_{x,K}
+(D_{\theta,K})_{12}\mathbf{G}_{y,K}\bigr)\mathbf{E}_K
=
-\mathbf{r}_{F_x,K}^{\mathrm{face}},
\label{eq:dg_m1_jx}
\end{equation}
\begin{equation}
\sigma_{t,K}\mathbf{M}_K\mathbf{F}_{y,K}
-\bigl((D_{\theta,K})_{21}\mathbf{G}_{x,K}
+(D_{\theta,K})_{22}\mathbf{G}_{y,K}\bigr)\mathbf{E}_K
=
-\mathbf{r}_{F_y,K}^{\mathrm{face}},
\label{eq:dg_m1_jy}
\end{equation}
If the closure is also used in the interface dissipation, the face speed is chosen using a wave-speed based on the closure tensor,
\begin{equation}
a_e
=
\max\!\left\{
\sqrt{\mathbf{n}_e^\top \mathbf{D}_{\theta,K^-}\mathbf{n}_e},
\sqrt{\mathbf{n}_e^\top \mathbf{D}_{\theta,K^+}\mathbf{n}_e}
\right\}.
\end{equation}

\section{Numerical experiments}
\label{Sec:numerical_experiments}

In this section, we numerically validate the proposed HN closure and compare its performance with the classical Levermore closure through three tests:
(i) closure accuracy,
(ii) eigenvalue violation test, and
(iii) DG simulations for radiative transfer problems.
The HN closure was trained using two \texttt{NVIDIA GeForce RTX 4090} GPUs. 
For the DG simulations, the DG solver was executed on an \texttt{AMD Ryzen Threadripper PRO 5955WX 16-Core} CPU. 
During the DG simulation, the trained HN closure was called whenever closure evaluation was required by the moment system. 
For training and quantitative comparison, high-fidelity closure data and reference solutions were generated using a Monte Carlo solver~\cite{krotz2024hybrid}.

\subsection{Experiment 1: Lattice problem}
\label{ex1}

We first test the proposed closure on a two-dimensional steady-state lattice problem. This example serves as a benchmark for evaluating both closure accuracy and solution quality.
Let $\mathbf{F}=(F_x,F_y)^\top$ and
\[
\mathbf{P}=
\begin{pmatrix}
P_{xx} & P_{xy}\\
P_{yx} & P_{yy}
\end{pmatrix},
\qquad
\sigma_t=\sigma_a+\sigma_s.
\]
The two-dimensional steady-state radiation moment system is
\begin{equation}
\sigma_a E+\nabla_{\mathbf{x}}\cdot \mathbf{F}=Q,
\label{eq:exp1_energy}
\end{equation}
\begin{equation}
\sigma_t \mathbf{F}+\nabla_{\mathbf{x}}\cdot \mathbf{P}=\mathbf{0},
\label{eq:exp1_flux}
\end{equation}
that is,
\[
\sigma_a E+\frac{\partial F_x}{\partial x}+\frac{\partial F_y}{\partial y}=Q,
\]
\[
\sigma_t F_x+\frac{\partial P_{xx}}{\partial x}+\frac{\partial P_{xy}}{\partial y}=0,
\qquad
\sigma_t F_y+\frac{\partial P_{yx}}{\partial x}+\frac{\partial P_{yy}}{\partial y}=0.
\]
Equivalently, with $\mathbf{u}=(E,F_x,F_y)^\top$, the system may be written as
\begin{equation}
\begin{pmatrix}
\sigma_a E\\
\sigma_t F_x\\
\sigma_t F_y
\end{pmatrix}
+
\frac{\partial}{\partial x}
\begin{pmatrix}
F_x\\
P_{xx}\\
P_{yx}
\end{pmatrix}
+
\frac{\partial}{\partial y}
\begin{pmatrix}
F_y\\
P_{xy}\\
P_{yy}
\end{pmatrix}
=
\begin{pmatrix}
Q\\
0\\
0
\end{pmatrix}.
\label{eq:exp1_conservative}
\end{equation}

The computational setup is shown in Fig.~\ref{fig:configuration_ex1}. 
The domain is a \(7\times7\) square partitioned into unit cells. 
The blue cells represent obstacle regions with strong absorption, the white cells correspond to background regions where scattering is dominant, and the red central cell acts as a localized source. 
This heterogeneous checkerboard configuration produces complex transport behavior, including anisotropic propagation and shadowing effects, providing a test for both closure accuracy and overall solution quality.

\begin{figure}[h!]
    \centering
    \begin{subfigure}[b]{0.8\linewidth}
        \centering
        \includegraphics[width=\linewidth]{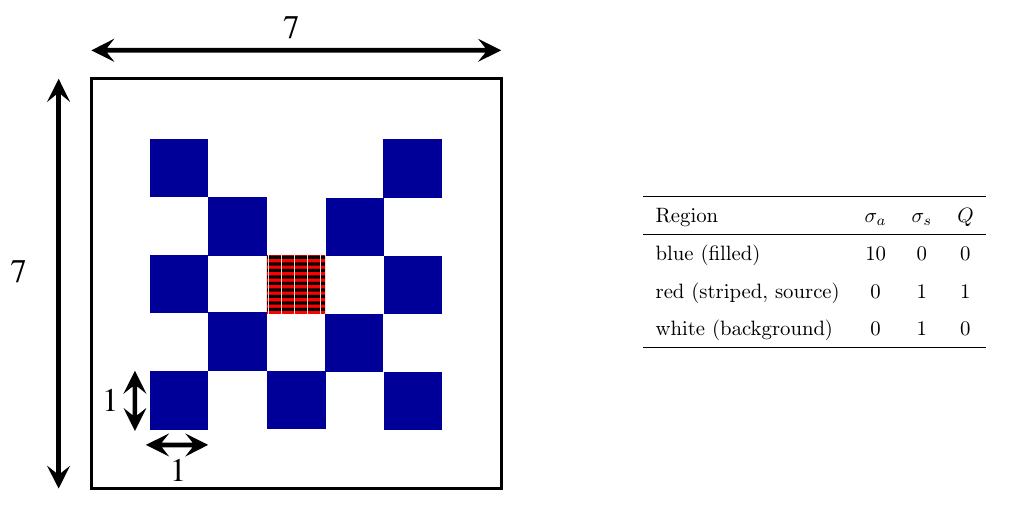}
    \end{subfigure} 
    \caption{Lattice benchmark configuration on a \(7\times7\) square domain with \((x,y)\in[0,7]\times[0,7]\). The blue filled cells are the obstacle (Obs) regions, with \((\sigma_a,\sigma_s,Q)=(10,0,0)\). The white cells are the background (Bg) regions, with \((\sigma_a,\sigma_s,Q)=(0,1,0)\). The red striped central cell is the source region, with \((\sigma_a,\sigma_s,Q)=(0,1,1)\). Each highlighted block has size \(1\times1\).}
    \label{fig:configuration_ex1}
\end{figure}

\begin{table}[htbp!]
\footnotesize
\renewcommand{\arraystretch}{1.08}
\setlength{\tabcolsep}{4pt}
\centering
\caption{Neural network specifications of the HN closure for the two-dimensional steady-state lattice problem (Fig.~\ref{fig:configuration_ex1}).}
\label{tab:hpnn_network_specs_ex1}
\begin{tabular}{lccc}
\hline
 & \makecell[c]{Entropy \\ network $\eta_\theta$}
 & \makecell[c]{Symmetric \\ network $\mathcal{N}_S$}
 & \makecell[c]{Integration constant \\ network $\mathcal{N}_A$} \\
\hline
Input dimension
& $3$
& $9$
& $9$ \\

Output dimension
& $1$
& $6$
& $3$ \\

Hidden layers
& $[128]$
& $[128,128]$
& $[64,64]$ \\

Activation
& \makecell[c]{Softplus}
& Tanh
& Tanh \\

Trainable parameters
& $641$
& $18{,}566$
& $4{,}995$ \\
\hline
Total trainable parameters
& \multicolumn{3}{c}{$24{,}202$} \\
\hline
\end{tabular}
\end{table}

\begin{figure}[h!]
    \centering
    \begin{subfigure}[b]{0.8\linewidth}
        \centering
        \includegraphics[width=\linewidth]{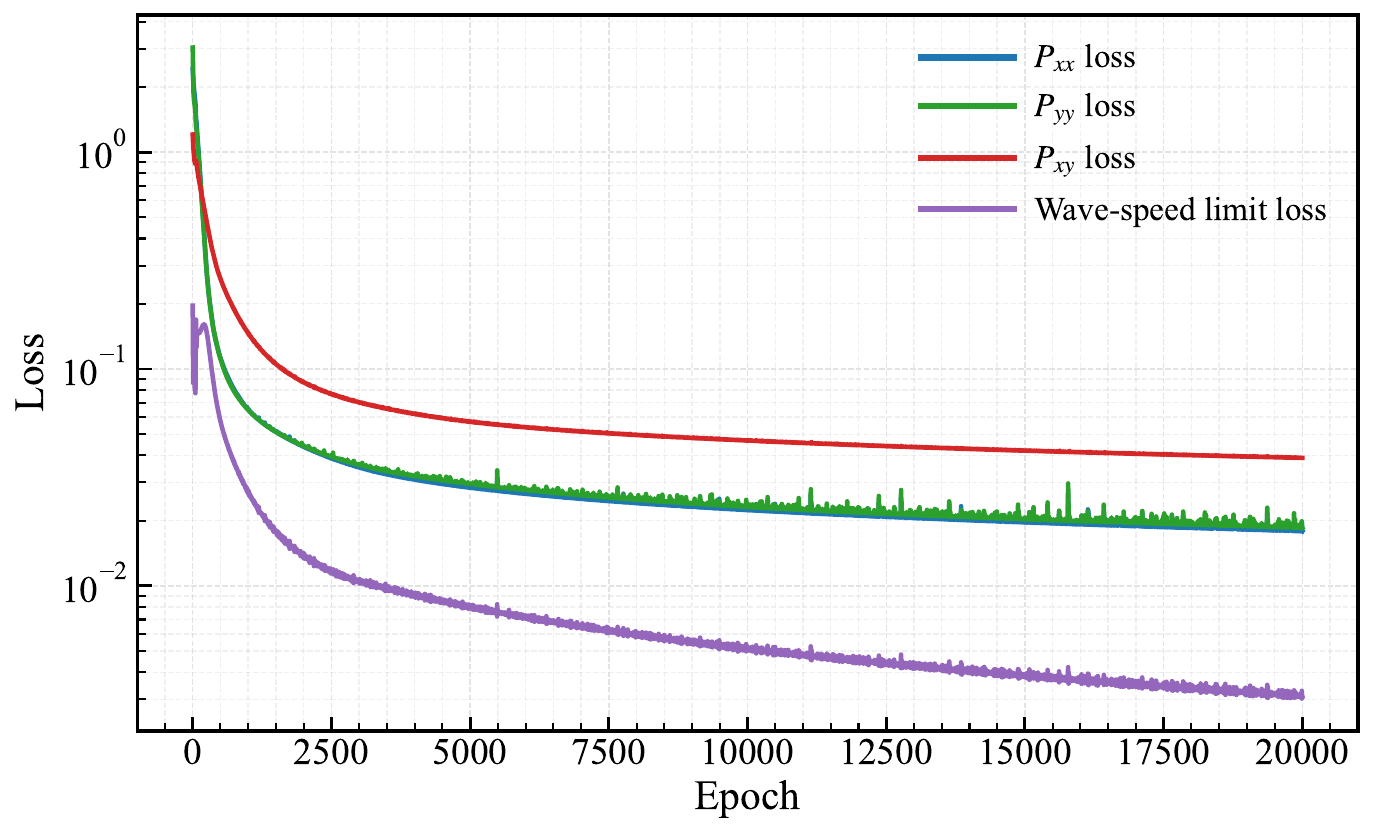}
    \end{subfigure} 
    \caption{Training loss histories of the HN closure, showing the component losses for $P_{xx}$, $P_{yy}$, $P_{xy}$, and the wave-speed limit penalty.}
    \label{fig:loss_ex1}
\end{figure}

\begin{figure}[h!]
    \centering
    \begin{subfigure}[b]{0.8\linewidth}
        \centering
        \includegraphics[width=\linewidth]{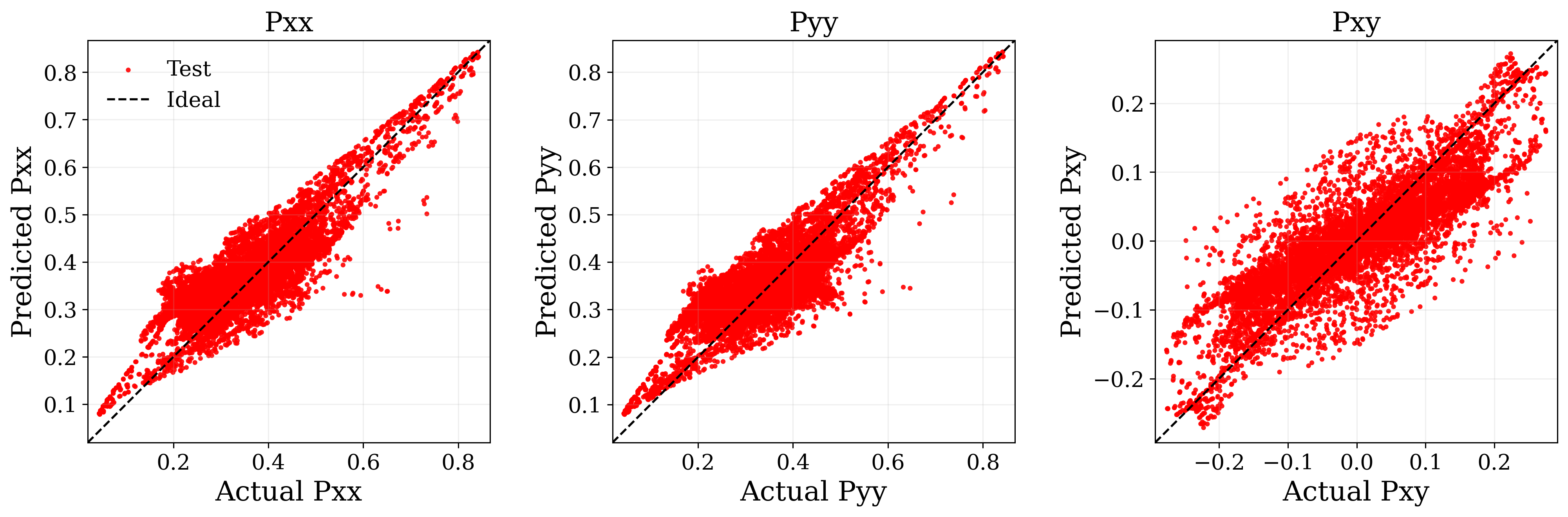}
        \subcaption{Levermore closure}
    \end{subfigure} \\
    \begin{subfigure}[b]{0.8\linewidth}
        \centering
        \includegraphics[width=\linewidth]{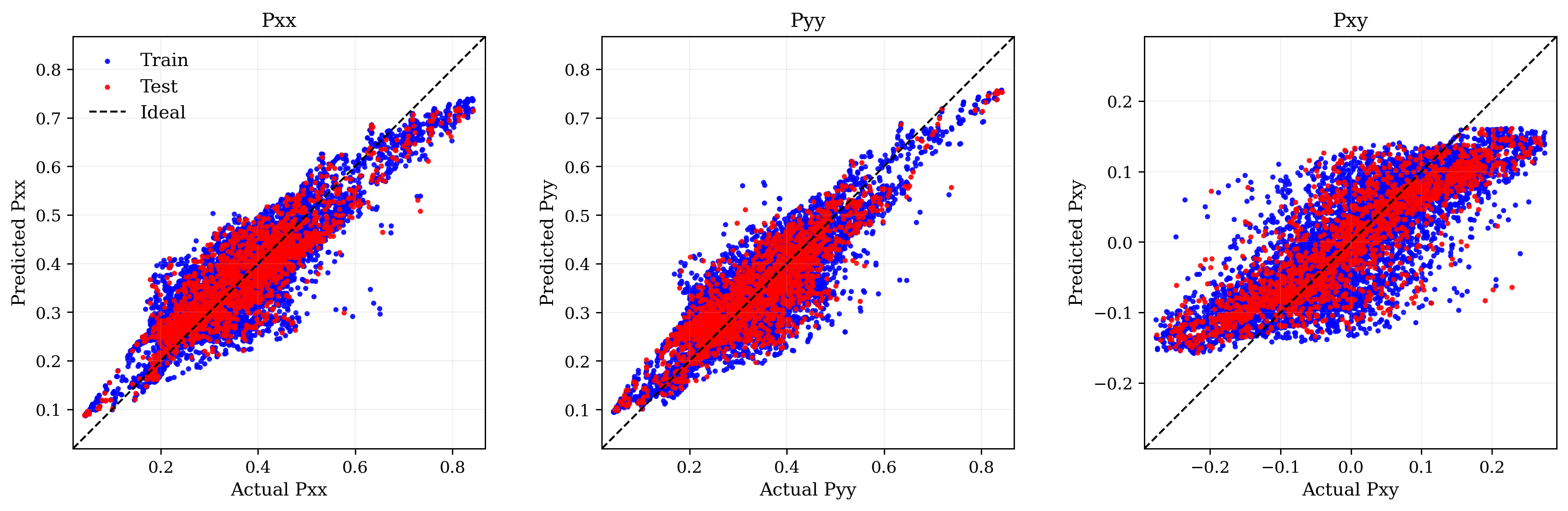}
        \subcaption{HN closure (local)}
    \end{subfigure} \\
    \begin{subfigure}[b]{0.8\linewidth}
        \centering
        \includegraphics[width=\linewidth]{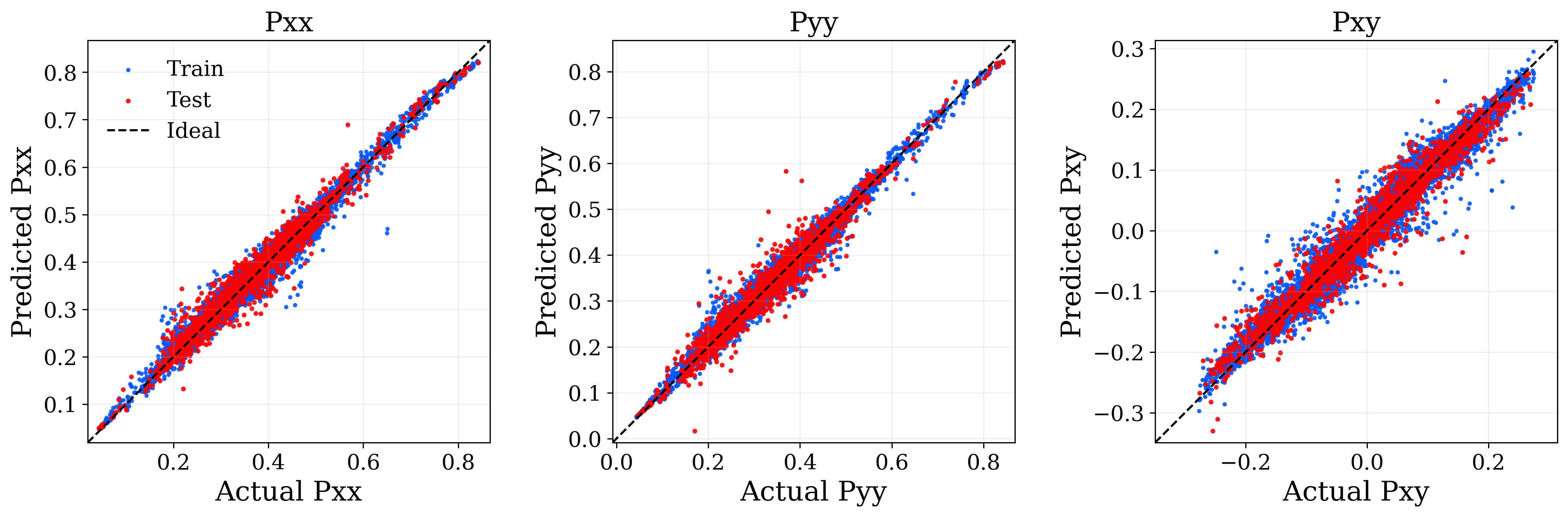}
        \subcaption{\textbf{HN closure (local + nonlocal)}}
    \end{subfigure}
    \caption{Comparison between reference and predicted closure outputs $(P_{xx},P_{yy},P_{xy})$ for the steady lattice radiative transfer problem. The top panel shows the analytical Levermore closure. The middle panel shows the HN closure using only local state variables. The bottom panel shows the hyperbolic neural closure using both local state variables and gradient information. For the learned closures, training samples are shown in blue, testing samples in red, and the dashed line represents the ideal relation $y=x$. The model using gradient information shows the best agreement with the reference targets.}
    \label{fig:actual_predicted_ex1}
\end{figure}

\begin{figure}[h!]
    \centering
    \begin{subfigure}[b]{0.8\linewidth}
        \centering
        \includegraphics[width=\linewidth]{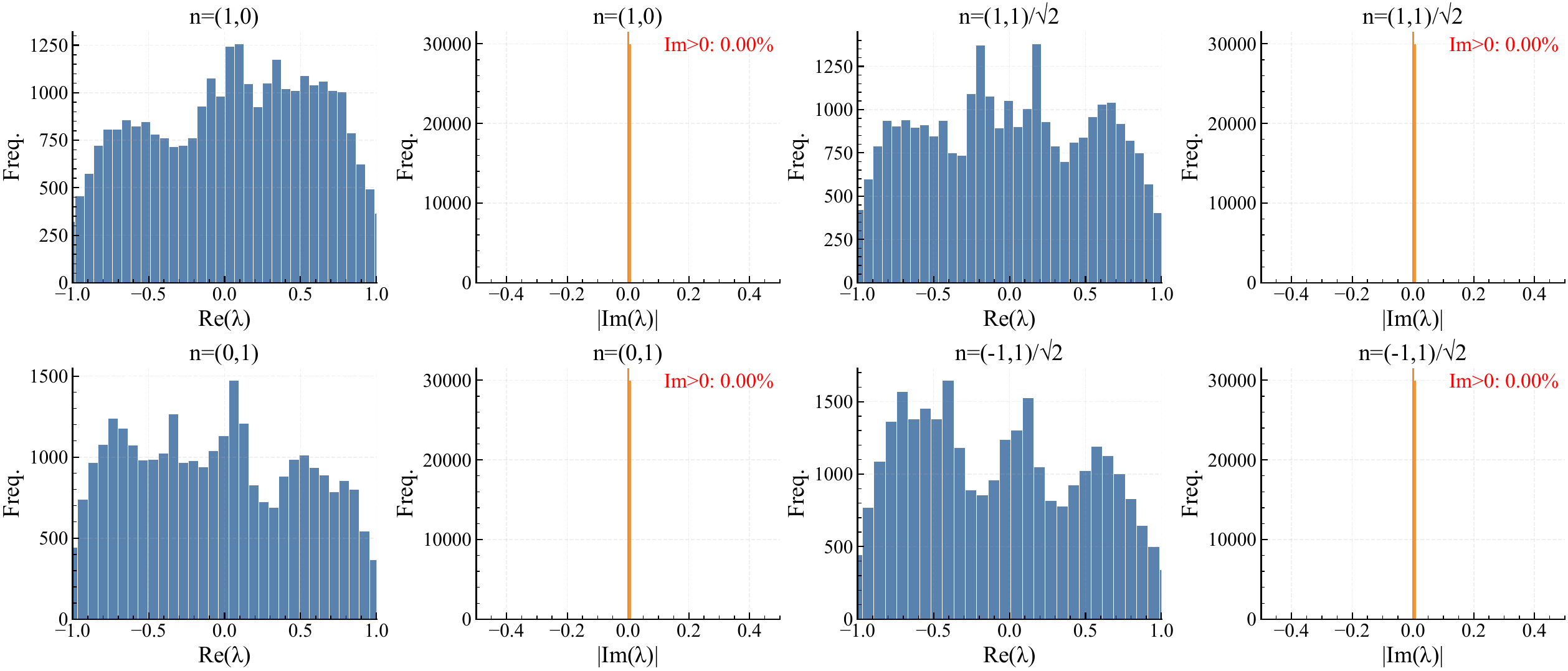}
    \end{subfigure}
    \caption{Eigenvalue histograms of the directional Jacobian $A(\mathbf{n})=\mathbf{n}_x J_x+\mathbf{n}_y J_y$ for the trained HN closure in the two-dimensional lattice radiative transfer problem. Histograms are computed from a total of \(100{,}000\) sampled states. For each direction $\mathbf{n}\in\{(1,0),(0,1),(1,1)/\sqrt{2},(-1,1)/\sqrt{2}\}$, the left panel shows the distribution of $\mathrm{Re}(\lambda)$ and the right panel shows $|\mathrm{Im}(\lambda)|$. The annotated fraction $\mathrm{Im}>0$ is $0.00\%$ in all cases, indicating real-valued spectra for these sampled states.}
    \label{fig:histogram_ex1}
\end{figure}

\begin{table}[h!]
\footnotesize
\centering
\caption{Comparison of closure accuracy for each pressure component in the lattice radiative transfer problem on a $7\times7$ domain. The HN closure is trained using Monte Carlo reference data generated on a $100\times100$ lattice with $N_p=10^7$ particles, with 80\% of the data used for training and 20\% for testing. The neural network uses hidden width 128 with Tanh activation, quadrature points $N_q=4$, and is trained for 20{,}000 epochs. Reported metrics are evaluated for $P_{xx}$, $P_{yy}$, and $P_{xy}$, where $\hat P$ denotes the predicted pressure tensor.}
\label{tab:closure_metrics_p}
\begin{tabular}{lllcccc}
\hline
Closure & Split & Component & MSE & $R^2$ & $\max\lvert P-\hat P\rvert$ & $\mathrm{mean}\,\lvert P-\hat P\rvert$ \\
\hline
\multirow{6}{*}{HN closure}
& \multirow{3}{*}{Train}
& $P_{xx}$ & $2.7980\times10^{-4}$ & $0.9823$ & $1.8971\times10^{-1}$ & $1.1502\times10^{-2}$ \\
& & $P_{yy}$ & $2.6626\times10^{-4}$ & $0.9816$ & $1.6438\times10^{-1}$ & $1.1159\times10^{-2}$ \\
& & $P_{xy}$ & $4.8071\times10^{-4}$ & $0.9609$ & $2.1331\times10^{-1}$ & $1.5250\times10^{-2}$ \\ \cline{2-7}
& \multirow{3}{*}{Test} 
& $P_{xx}$ & $4.6160\times10^{-4}$ & $0.9696$ & $1.2528\times10^{-1}$ & $1.4677\times10^{-2}$ \\
& & $P_{yy}$ & $5.4556\times10^{-4}$ & $0.9611$ & $2.1268\times10^{-1}$ & $1.5137\times10^{-2}$ \\
& & $P_{xy}$ & $6.2433\times10^{-4}$ & $0.9503$ & $1.9345\times10^{-1}$ & $1.7475\times10^{-2}$ \\
\hline
\multirow{3}{*}{Levermore closure}
& \multirow{3}{*}{All (Test)}
& $P_{xx}$ & $3.7428\times10^{-3}$ & $0.7612$ & $3.1305\times10^{-1}$ & $4.9941\times10^{-2}$ \\
& & $P_{yy}$ & $4.0208\times10^{-3}$ & $0.7205$ & $3.0191\times10^{-1}$ & $5.2808\times10^{-2}$ \\
& & $P_{xy}$ & $3.5868\times10^{-3}$ & $0.7095$ & $2.5367\times10^{-1}$ & $4.5525\times10^{-2}$ \\
\hline
\end{tabular}
\label{table:closure_accuracy_ex1}
\end{table}

\begin{figure}[h!]
    \centering
    \begin{subfigure}[b]{0.49\linewidth}
        \centering
        \includegraphics[width=\linewidth]{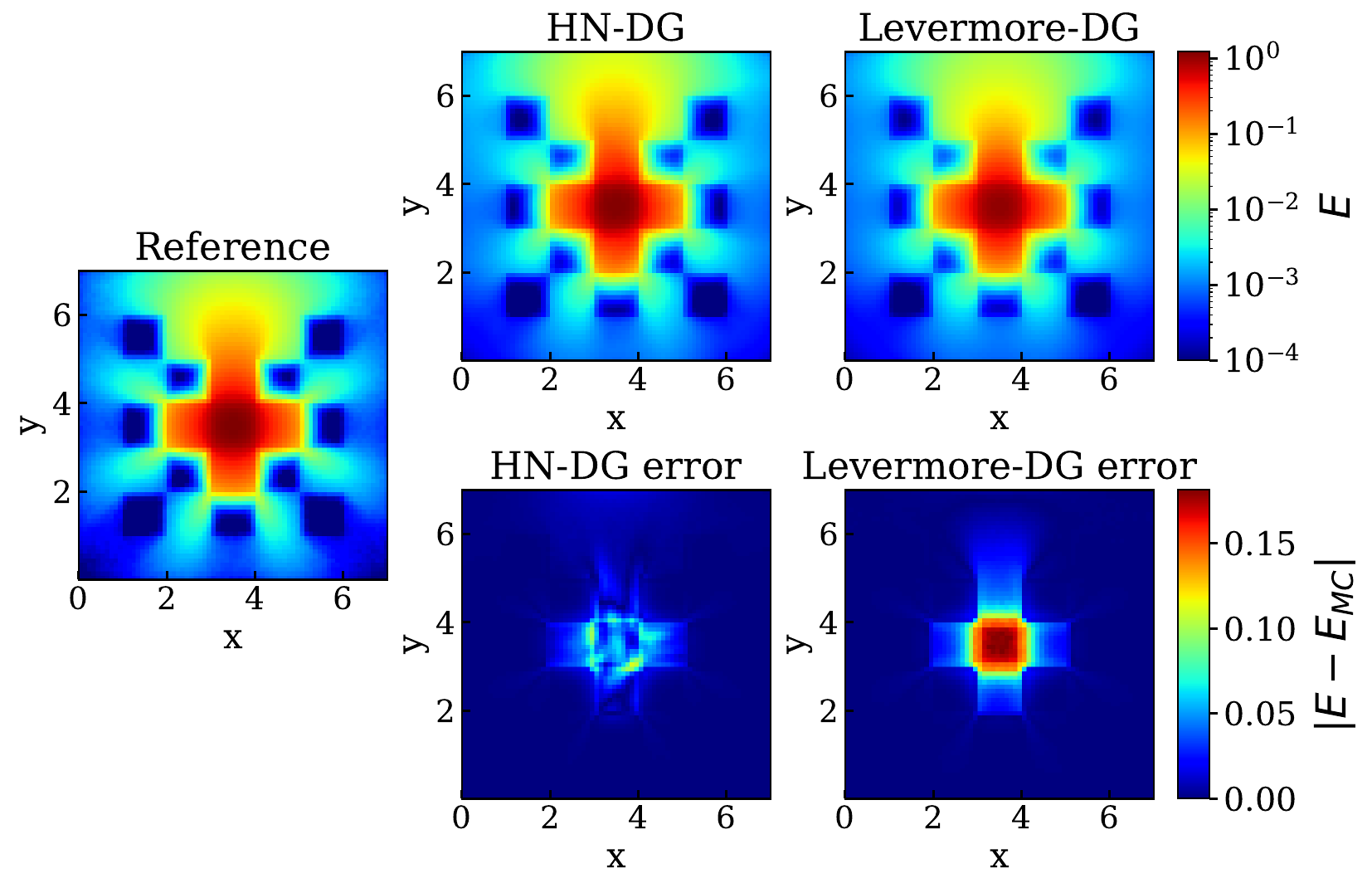}
        \subcaption{Trained setting: Obs$(10,0)$, Bg$(0,1)$}
    \end{subfigure}
    \begin{subfigure}[b]{0.49\linewidth}
        \centering
        \includegraphics[width=\linewidth]{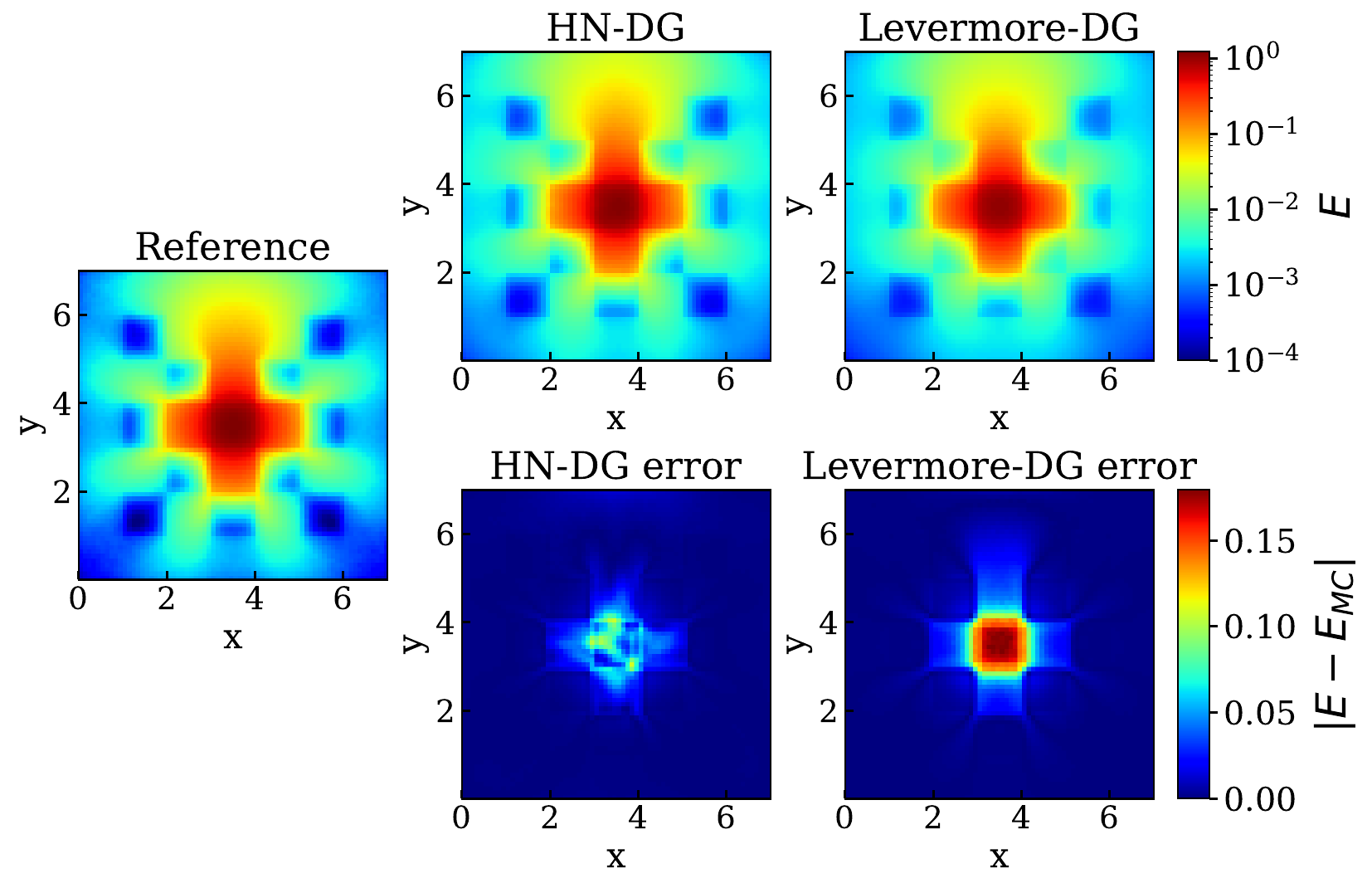}
        \subcaption{Tested setting: Obs$(5,0)$, Bg$(0,1)$}
    \end{subfigure} \\
    \begin{subfigure}[b]{0.49\linewidth}
        \centering
        \includegraphics[width=\linewidth]{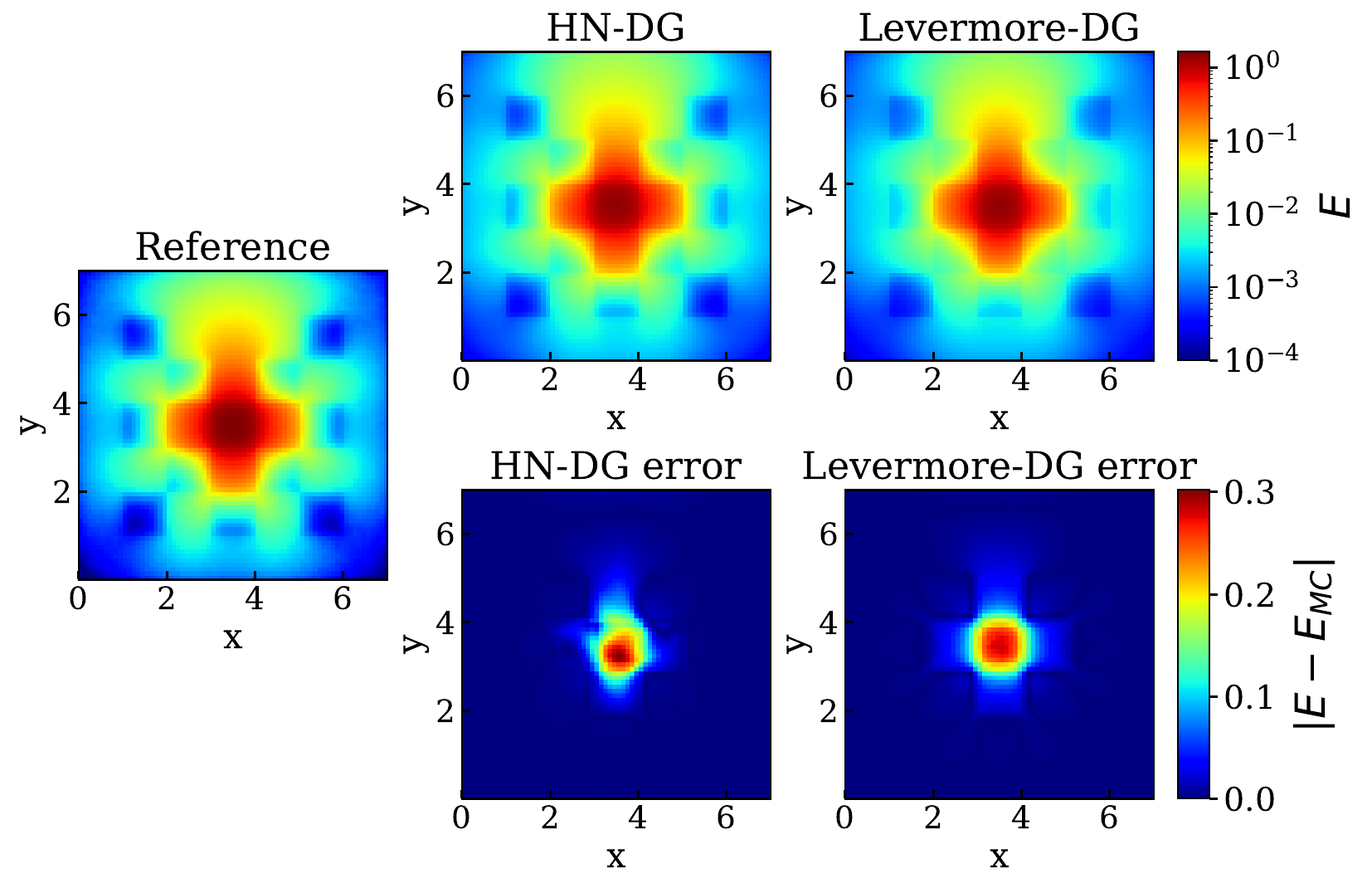}
        \subcaption{Tested setting: Obs$(4,0)$, Bg$(0,2)$}
    \end{subfigure}
    \begin{subfigure}[b]{0.49\linewidth}
        \centering
        \includegraphics[width=\linewidth]{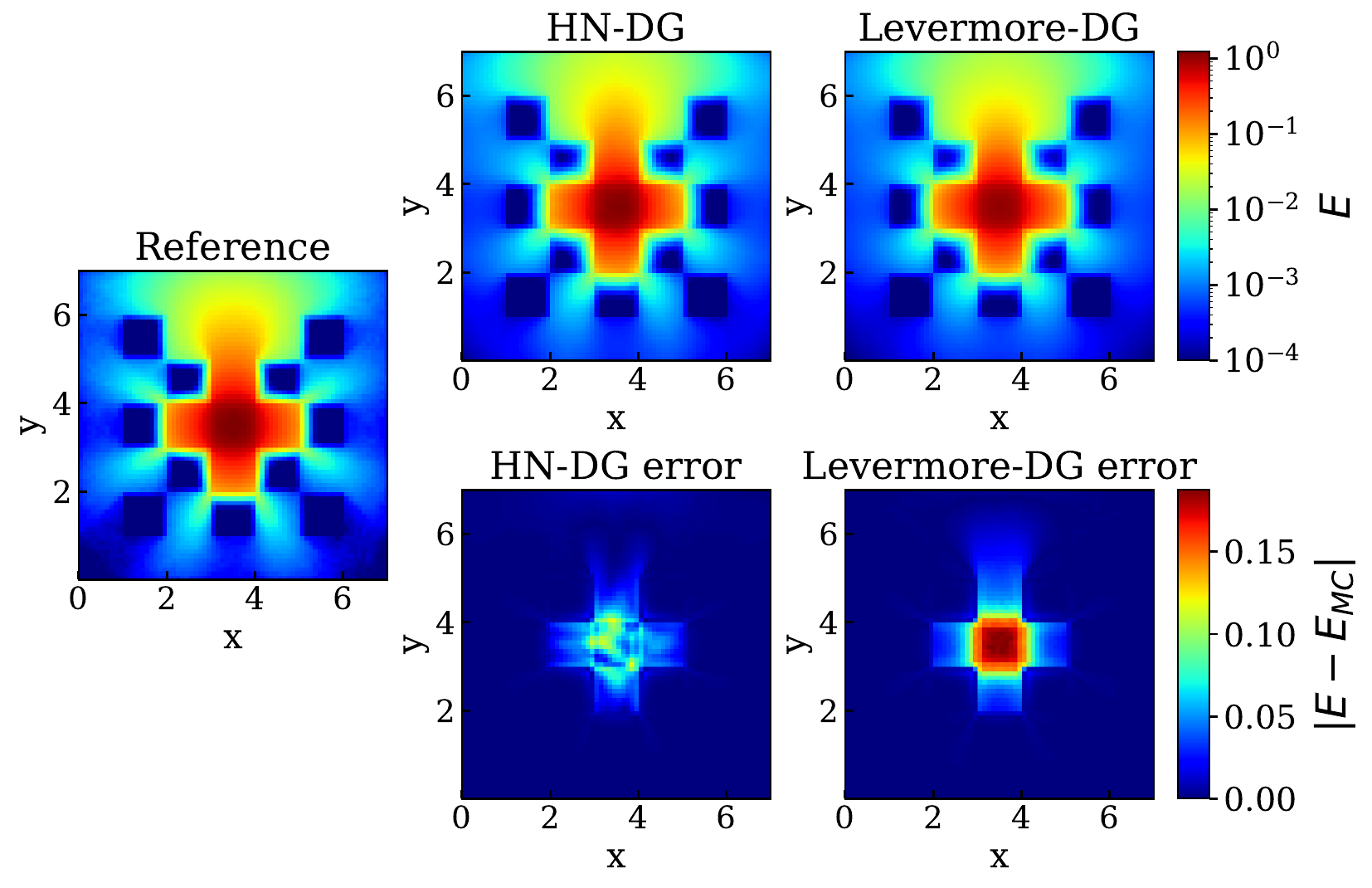}
        \subcaption{Tested setting: Obs$(15,1)$, Bg$(0,1)$}
    \end{subfigure}
\caption{Comparison across four material settings of the checkerboard problem, where Obs denotes the obstacle (blue filled cells) coefficients \((\sigma_a,\sigma_s)\) and Bg denotes the background (white cells) coefficients \((\sigma_a,\sigma_s)\). In each subfigure, the top row shows the Monte Carlo ground truth (left), the HN closure solution (middle), and the Levermore closure solution (right). The bottom row shows the corresponding absolute errors with respect to the Monte Carlo reference.}
    \label{fig:solution_ex1}
\end{figure}

\begin{table}[h!]
\centering
\footnotesize
\caption{Error comparison for four lattice benchmark settings. Here, Obs denotes the obstacle (blue filled cells) material coefficients \((\sigma_a,\sigma_s)\), and Bg denotes the background (white cells) material coefficients \((\sigma_a,\sigma_s)\). The reference is the Monte Carlo solution (\(N_p=5\times10^6\), \(70\times70\) grid). RMSE (root mean square error) and MAE (mean absolute error) are reported together with the relative \(L^2\) error. Smaller error is highlighted in bold for each metric within each setting.}
\label{tab:four_case_error}
\begin{tabular}{llccc}
\hline
Setting & Closure & Relative \(L^2\) error & RMSE & MAE \\
\hline
\multirow{2}{*}{Fig.~\ref{fig:solution_ex1}a: Obs\((10,0)\), Bg\((0,1)\)}
& HN  & \(\mathbf{7.5280\times10^{-2}}\) & \(\mathbf{1.3481\times10^{-2}}\) & \(\mathbf{5.2890\times10^{-3}}\) \\
& Levermore & \(1.7146\times10^{-1}\)          & \(3.0705\times10^{-2}\)          & \(9.9785\times10^{-3}\) \\
\hline
\multirow{2}{*}{Fig.~\ref{fig:solution_ex1}b: Obs\((5,0)\), Bg\((0,1)\)}
& HN     & \(\mathbf{7.8590\times10^{-2}}\) & \(\mathbf{1.4209\times10^{-2}}\) & \(\mathbf{5.6324\times10^{-3}}\) \\
& Levermore & \(1.6613\times10^{-1}\)          & \(3.0035\times10^{-2}\)          & \(1.0285\times10^{-2}\) \\
\hline
\multirow{2}{*}{Fig.~\ref{fig:solution_ex1}c: Obs\((4,0)\), Bg\((0,2)\)}
& HN    & \(\mathbf{1.5853\times10^{-1}}\) & \(\mathbf{3.7045\times10^{-2}}\) & \(\mathbf{1.0249\times10^{-2}}\) \\
& Levermore & \(1.7643\times10^{-1}\)          & \(4.1227\times10^{-2}\)          & \(1.2411\times10^{-2}\) \\
\hline
\multirow{2}{*}{Fig.~\ref{fig:solution_ex1}d: Obs\((15,1)\), Bg\((0,1)\)}
& HN    & \(\mathbf{9.3253\times10^{-2}}\) & \(\mathbf{1.6780\times10^{-2}}\) & \(\mathbf{5.9224\times10^{-3}}\) \\
& Levermore & \(1.7829\times10^{-1}\)          & \(3.2082\times10^{-2}\)          & \(1.0201\times10^{-2}\) \\
\hline
\end{tabular}
\label{table:solution_accuracy_ex1}
\end{table}

The HN closure is implemented using three neural networks with a total of 24,202 trainable parameters, as summarized in Table~\ref{tab:hpnn_network_specs_ex1}. The model is trained on Monte Carlo reference data using an 80/20 train-test split, and the training process (Fig.~\ref{fig:loss_ex1}) shows stable convergence across all tensor components together with the wave-speed penalty.
The learned HN closure achieves high accuracy in predicting the radiation pressure tensor, as shown in Fig.~\ref{fig:actual_predicted_ex1}. Quantitatively, Table~\ref{tab:closure_metrics_p} reports test-set errors of $4.62\times10^{-4}$, $5.46\times10^{-4}$, and $6.24\times10^{-4}$ for $P_{xx}$, $P_{yy}$, and $P_{xy}$, respectively, with corresponding $R^2$ values of $0.9696$, $0.9611$, and $0.9503$. In contrast, the Levermore closure yields errors on the order of $3.6\times10^{-3}$ to $4.0\times10^{-3}$ with significantly lower $R^2$ values ($\approx 0.71$--$0.76$), indicating that the HN closure improves accuracy by nearly an order of magnitude.

The spectral properties of the flux Jacobian induced by the HN closure are examined in Fig.~\ref{fig:histogram_ex1}, where all sampled directional eigenvalues remain real with zero fraction of nonzero imaginary parts, indicating no observed violation of real-valued spectral structure.
This improved closure accuracy translates directly into better solution quality. As shown in Fig.~\ref{fig:solution_ex1} and Table~\ref{table:solution_accuracy_ex1}, the HN closure consistently reduces the relative $L^2$ error across all tested configurations. For example, in the reference setting (Obs$(10,0)$, Bg$(0,1)$), the relative error is reduced from $1.71\times10^{-1}$ (Levermore) to $7.53\times10^{-2}$ (HN). Similar improvements are observed across all other configurations, demonstrating both improved predictive accuracy and stable numerical behavior of the proposed closure.
Although the wave-speed penalty reduced most violations during training,
approximately \(1\)--\(2\%\) of the evaluated states still exceeded the prescribed
wave-speed bound in the DG simulations. Therefore, the face wave speed was clipped
by an upper bound when computing the local Lax--Friedrichs flux.

\subsection{Experiment 2: Beam-crossing problem}
\label{ex2}

We next test the proposed closure on a two-dimensional \emph{time-dependent} radiation moment problem with two intersecting source regions. Unlike the steady-state lattice example, this problem evolves from vacuum initial data and is driven by a continuous volumetric source over a finite time interval.
Let $\mathbf{F}=(F_x,F_y)^\top$ and
\[
\mathbf{P}=
\begin{pmatrix}
P_{xx} & P_{xy}\\
P_{yx} & P_{yy}
\end{pmatrix},
\qquad
\sigma_t=\sigma_a+\sigma_s.
\]
The two-dimensional transient radiation moment system is
\begin{equation}
\frac{1}{c}\frac{\partial E}{\partial t}+\sigma_a E+\nabla_{\mathbf{x}}\cdot \mathbf{F}=Q,
\label{eq:exp2_energy}
\end{equation}
\begin{equation}
\frac{1}{c}\frac{\partial \mathbf{F}}{\partial t}+\sigma_t \mathbf{F}+\nabla_{\mathbf{x}}\cdot \mathbf{P}=\mathbf{0},
\label{eq:exp2_flux}
\end{equation}
that is,
\[
\frac{1}{c}\frac{\partial E}{\partial t}
+\sigma_a E
+\frac{\partial F_x}{\partial x}
+\frac{\partial F_y}{\partial y}
=Q,
\]
\[
\frac{1}{c}\frac{\partial F_x}{\partial t}
+\sigma_t F_x
+\frac{\partial P_{xx}}{\partial x}
+\frac{\partial P_{xy}}{\partial y}=0,
\qquad
\frac{1}{c}\frac{\partial F_y}{\partial t}
+\sigma_t F_y
+\frac{\partial P_{yx}}{\partial x}
+\frac{\partial P_{yy}}{\partial y}=0.
\]
Equivalently, with $\mathbf{u}=(E,F_x,F_y)^\top$, the system can be written as
\begin{equation}
\frac{1}{c}\frac{\partial}{\partial t}
\begin{pmatrix}
E\\
F_x\\
F_y
\end{pmatrix}
+
\begin{pmatrix}
\sigma_a E\\
\sigma_t F_x\\
\sigma_t F_y
\end{pmatrix}
+
\frac{\partial}{\partial x}
\begin{pmatrix}
F_x\\
P_{xx}\\
P_{yx}
\end{pmatrix}
+
\frac{\partial}{\partial y}
\begin{pmatrix}
F_y\\
P_{xy}\\
P_{yy}
\end{pmatrix}
=
\begin{pmatrix}
Q\\
0\\
0
\end{pmatrix}.
\label{eq:exp2_conservative}
\end{equation}

In our beam-crossing setup, as illustrated in Fig.~\ref{fig:configuration_ex2}, the computational domain is
$
\Omega=[0,7]\times[0,7],
$
with transport speed
$
c=1.
$
The initial condition is vacuum:
\[
E(\mathbf{x},0)=0,
\qquad
\mathbf{F}(\mathbf{x},0)=\mathbf{0}.
\]
The source term consists of two perpendicular volumetric source strips,
\[
Q(x,y)=
\begin{cases}
1, & 0.5<x<1.0,\;\; 2.5<y<4.5,\\
1, & 2.5<x<4.5,\;\; 0.5<y<1.0,\\
0, & \text{otherwise},
\end{cases}
\]
which continuously emit over the time interval $0\le t\le T_f$. In the Monte Carlo solver, particles are born uniformly in time over $[0,T_f]$ and are emitted isotropically from these two source regions.

The material coefficients are piecewise constant:
\[
\sigma_s(x,y)=0
\qquad \text{for all } (x,y)\in\Omega,
\]
and
\[
\sigma_a(x,y)=
\begin{cases}
10, & x<0.5 \text{ or } x>6.5 \text{ or } y<0.5 \text{ or } y>6.5,\\
0.02, & \text{otherwise}.
\end{cases}
\]
Thus, the interior is a weakly absorbing medium, while a boundary frame of width $0.5$ acts as a strong absorbing layer to suppress leakage and reflections.
For the time horizon we use
$T_f=3.5,
N_t=10,
\Delta t=\frac{T_f}{N_t}=0.35.$
This example is designed to test how well a closure handles the interaction of two crossing radiation fields. Since the underlying angular distribution becomes highly nontrivial near the overlap region of the two source-driven fronts, the problem provides a useful benchmark for comparing the standard Levermore closure and the learned HN closure against the Monte Carlo reference data.

\begin{figure}[h!]
    \centering
    \begin{subfigure}[b]{0.8\linewidth}
        \centering
        \includegraphics[width=\linewidth]{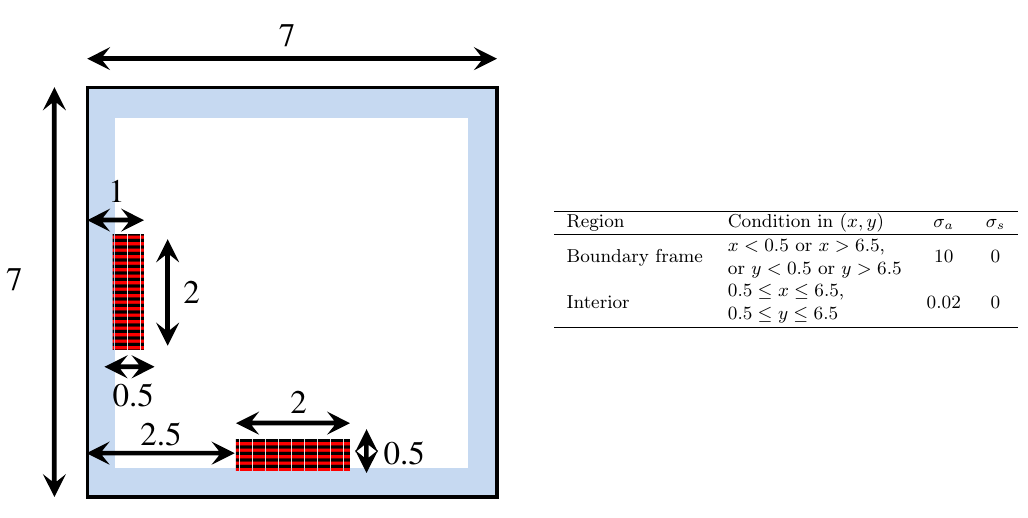}
    \end{subfigure}
    \caption{Beam-crossing configuration on a \(7\times7\) square domain \(\Omega=[0,7]\times[0,7]\), where \(x\in[0,7]\) and \(y\in[0,7]\). Material coefficients: \(\sigma_s=0\) everywhere, \(\sigma_a=10\) in the boundary frame (\(x<0.5\) or \(x>6.5\), or \(y<0.5\) or \(y>6.5\)), and \(\sigma_a=0.02\) in the interior. The problem is driven by two source patches, \((0.5<x<1.0,\;2.5<y<4.5)\) and \((2.5<x<4.5,\;0.5<y<1.0)\), emitting over \(0\le t\le T_f\).}
    \label{fig:configuration_ex2}
\end{figure}

\begin{table}[htbp!]
\footnotesize
\renewcommand{\arraystretch}{1.08}
\setlength{\tabcolsep}{4pt}
\centering
\caption{Neural network specifications of the HN closure used for the beam-crossing simulation.}
\label{tab:hpnn_network_specs}
\begin{tabular}{lccc}
\hline
 & \makecell[c]{Entropy \\ network $\eta_\theta$}
 & \makecell[c]{Symmetric \\ network $\mathcal{N}_S$}
 & \makecell[c]{Integration constant \\ network $\mathcal{N}_A$} \\
\hline
Input dimension
& $3$
& $9$
& $9$ \\

Output dimension
& $1$
& $6$
& $3$ \\

Hidden layers
& $[128]$
& $[128,128]$
& $[64,64]$ \\

Activation
& \makecell[c]{Softplus}
& Tanh
& Tanh \\

Trainable parameters
& $641$
& $18{,}566$
& $4{,}995$ \\
\hline
Total trainable parameters
& \multicolumn{3}{c}{$24{,}202$} \\
\hline
\end{tabular}
\end{table}

\begin{figure}[h!]
    \centering
    \begin{subfigure}[b]{0.8\linewidth}
        \centering
        \includegraphics[width=\linewidth]{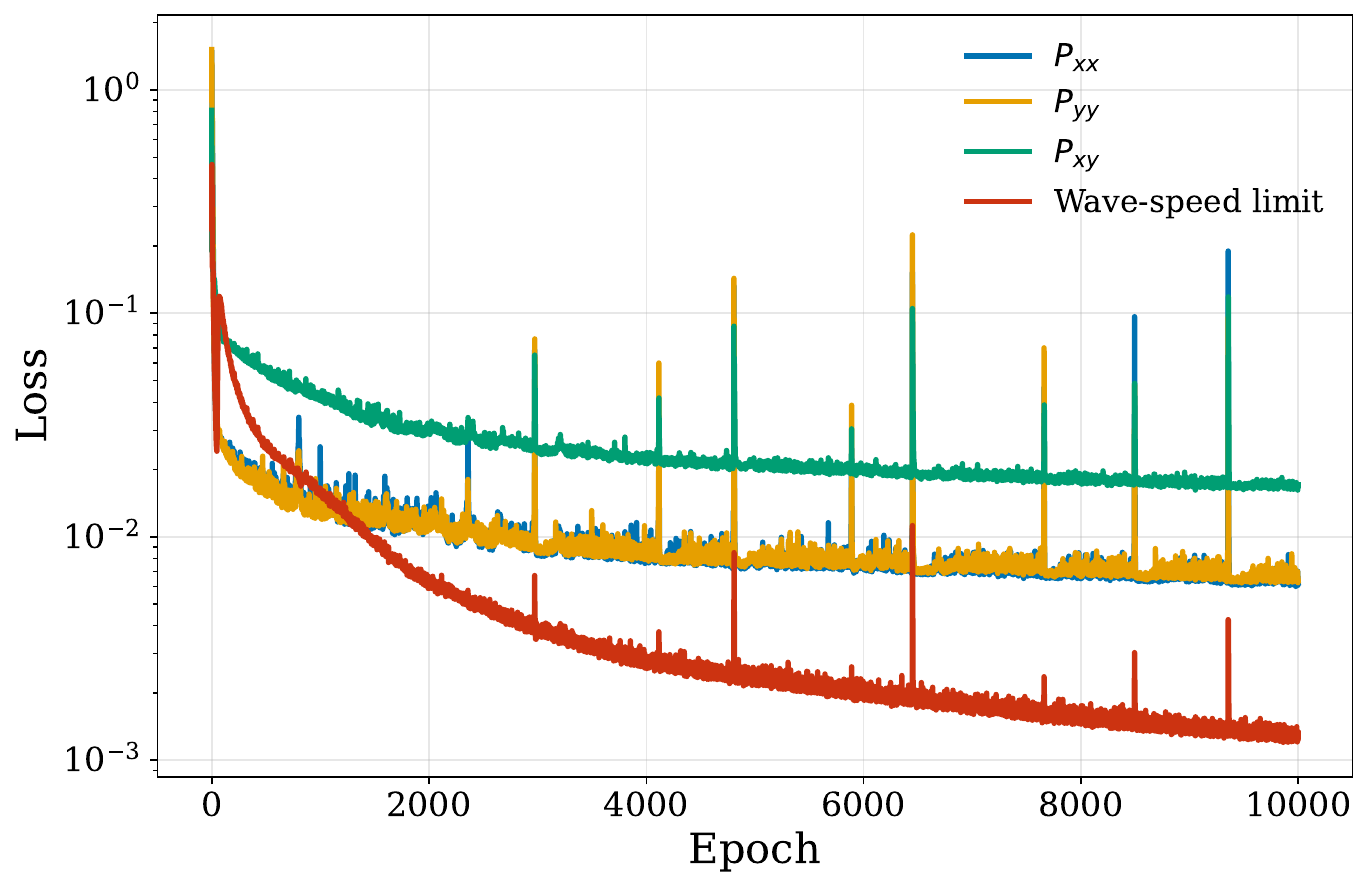}
    \end{subfigure} 
    \caption{Training loss histories of the HN closure for the beam crossing problem, showing the losses for $P_{xx}$, $P_{yy}$, $P_{xy}$, and the wave speed penalty.}
    \label{fig:loss_ex2}
\end{figure}

\begin{figure}[h!]
    \centering
    \begin{subfigure}[b]{0.8\linewidth}
        \centering
        \includegraphics[width=\linewidth]{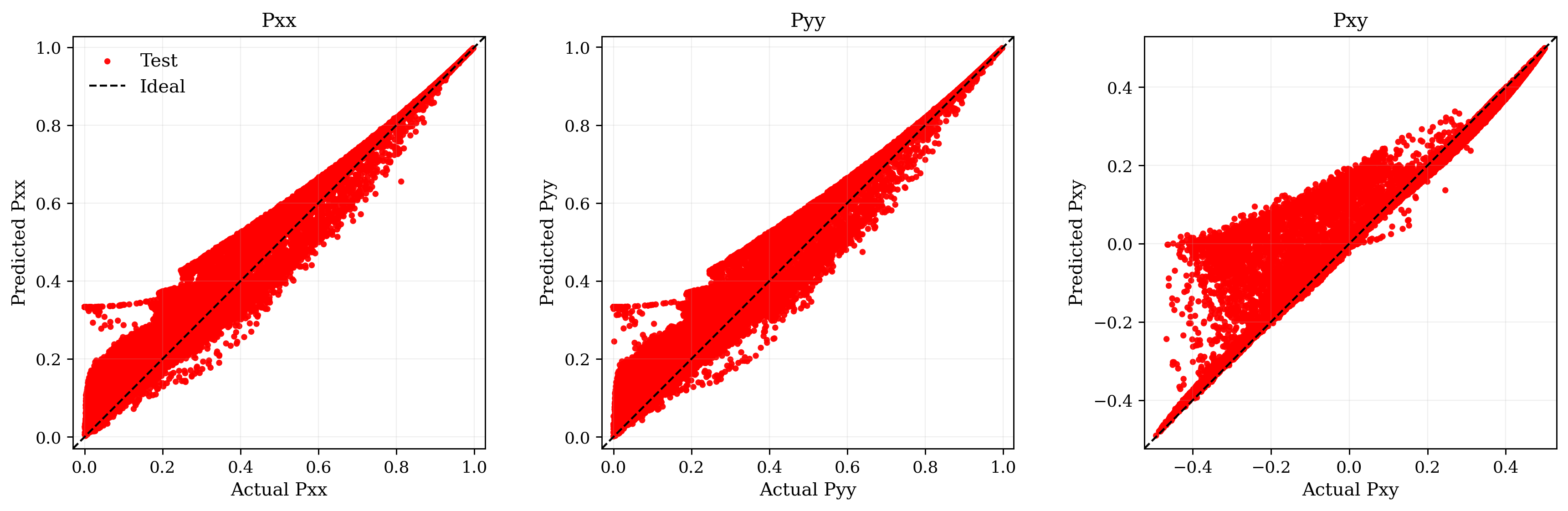}
        \subcaption{Levermore closure}
    \end{subfigure} \\
    \begin{subfigure}[b]{0.8\linewidth}
        \centering
        \includegraphics[width=\linewidth]{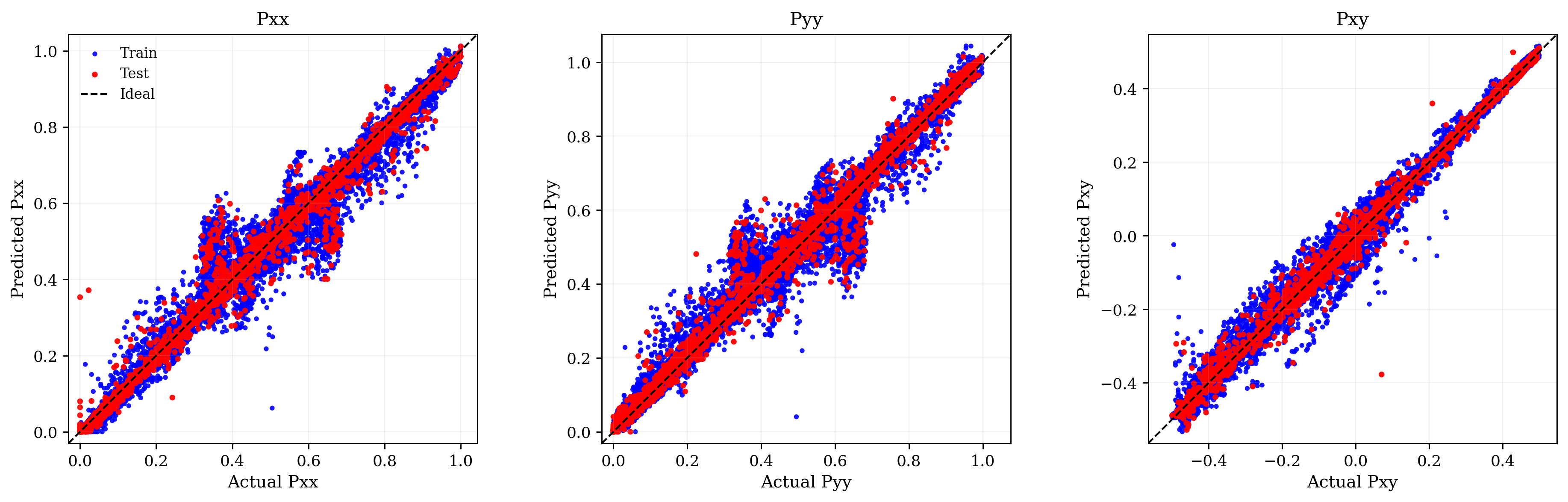}
        \subcaption{\textbf{HN closure}}
    \end{subfigure}
    \caption{Scatter plots comparing reference and predicted closure outputs $(P_{xx}, P_{yy}, P_{xy})$ for the steady lattice problem. In the HN results, training samples are shown in blue and testing samples in red. The dashed line represents the ideal relation $y=x$.}
    \label{fig:actual_predicted_ex2}
\end{figure}

\begin{figure}[h!]
    \centering
    \begin{subfigure}[b]{0.8\linewidth}
        \centering
        \includegraphics[width=\linewidth]{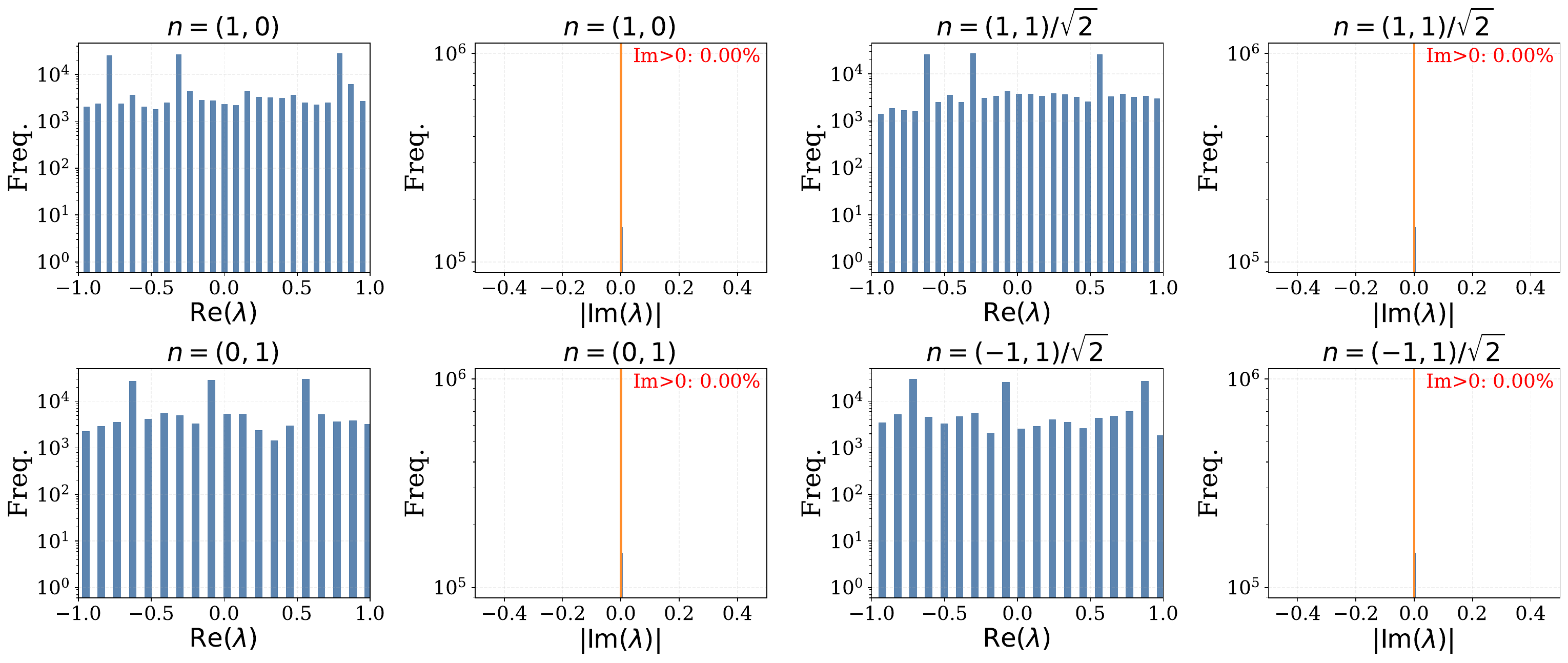}
    \end{subfigure}
    \caption{Eigenvalue histograms of the directional Jacobian $A(\mathbf{n})=\mathbf{n}_x J_x+\mathbf{n}_y J_y$ for the trained closure in the beam-crossing problem. Histograms are computed from a total of \(4900\) sampled states. For each direction $\mathbf{n}\in\{(1,0),(0,1),(1,1)/\sqrt{2},(-1,1)/\sqrt{2}\}$, the left panel shows the distribution of $\mathrm{Re}(\lambda)$ and the right panel shows $|\mathrm{Im}(\lambda)|$. The annotated fraction $\mathrm{Im}>0$ is $0.00\%$ in all cases, indicating real-valued spectra for these sampled states.}
    \label{fig:histrogram_ex2}
\end{figure}

\begin{table}[h!]
\footnotesize
\centering
\caption{Closure accuracy comparison for the time dependent beam crossing problem on a $7\times7$ domain. The neural closure is trained using Monte Carlo reference data generated on a $70\times70$ beam-crossing simulation with $N_p=10^6$ particles. The training data include all snapshots up to $T_f=3.5$ using $N_t=10$ uniform time steps, with 80\% of the data used for training and 20\% for testing. The neural network uses hidden width 128 with Tanh activation, quadrature points $N_q=4$, wave speed limiting in four directions, and is trained for 10{,}000 epochs with learning rate $10^{-3}$. Errors are reported for $P_{xx}$, $P_{yy}$, and $P_{xy}$.}
\label{tab:beam_crossing_closure_metrics_p}
\begin{tabular}{lllcccc}
\hline
Closure & Split & Component & MSE & $R^2$ & $\max\lvert P-\hat P\rvert$ & $\mathrm{mean}\,\lvert P-\hat P\rvert$ \\
\hline
\multirow{6}{*}{HN closure}
& \multirow{3}{*}{Train}
& $P_{xx}$ & $4.6542\times10^{-4}$ & $0.9938$ & $3.6696\times10^{-1}$ & $1.1035\times10^{-2}$ \\
& & $P_{yy}$ & $4.8359\times10^{-4}$ & $0.9936$ & $3.7132\times10^{-1}$ & $1.0643\times10^{-2}$ \\
& & $P_{xy}$ & $4.4348\times10^{-4}$ & $0.9832$ & $3.5996\times10^{-1}$ & $8.4952\times10^{-3}$ \\ \cline{2-7}
& \multirow{3}{*}{Test}
& $P_{xx}$ & $5.1325\times10^{-4}$ & $0.9928$ & $2.0566\times10^{-1}$ & $1.1338\times10^{-2}$ \\
& & $P_{yy}$ & $5.1303\times10^{-4}$ & $0.9934$ & $2.2795\times10^{-1}$ & $1.0875\times10^{-2}$ \\
& & $P_{xy}$ & $5.2971\times10^{-4}$ & $0.9812$ & $2.9528\times10^{-1}$ & $9.0392\times10^{-3}$ \\
\hline
\multirow{3}{*}{Levermore closure}
& \multirow{3}{*}{All (Test)}
& $P_{xx}$ & $5.7021\times10^{-2}$ & $0.6804$ & $3.3333\times10^{-1}$ & $1.8739\times10^{-1}$ \\
& & $P_{yy}$ & $5.6980\times10^{-2}$ & $0.7144$ & $3.3333\times10^{-1}$ & $1.8683\times10^{-1}$ \\
& & $P_{xy}$ & $3.2647\times10^{-3}$ & $0.8844$ & $4.1301\times10^{-1}$ & $1.8340\times10^{-2}$ \\
\hline
\end{tabular}
\end{table}

\begin{figure}[h!]
    \centering
    \begin{subfigure}[b]{1\linewidth}
        \centering
        \includegraphics[width=\linewidth]{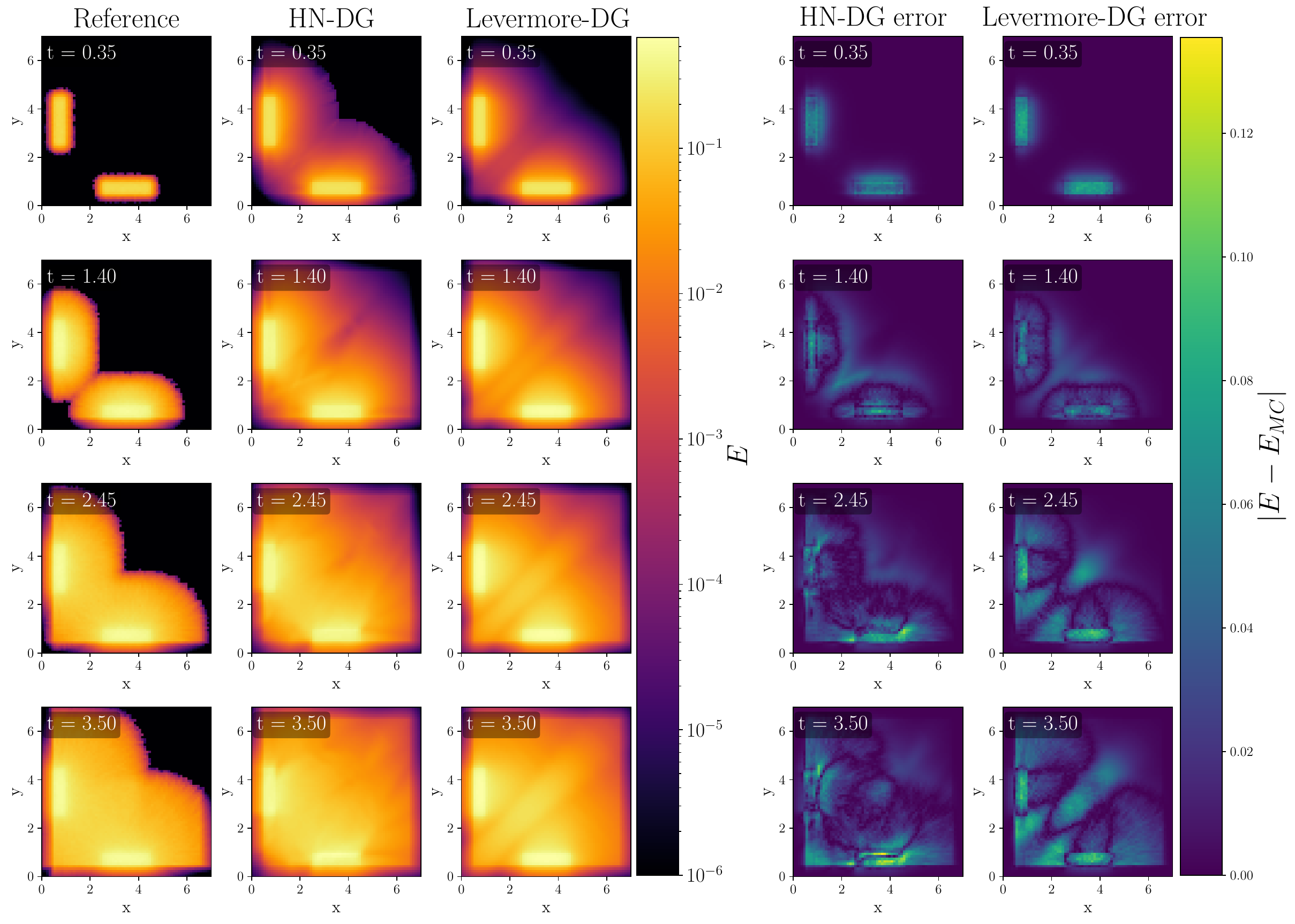}
    \end{subfigure}
\caption{Beam crossing comparison at four time instances \(t=\{0.35,\,1.40,\,2.45,\,3.50\}\) on a \(70\times70\) grid with source scale \(Q=1.0\). Each row corresponds to one time snapshot. Columns show the Monte Carlo solution, the HN-DG solution, the Levermore-DG solution, and the corresponding absolute errors \(|E-E_{\mathrm{MC}}|\) for HN-DG and Levermore-DG.}
    \label{fig:solution_ex2}
\end{figure}

\begin{table}[h!]
\centering
\footnotesize
\caption{Error comparison for the beam crossing time-dependent problem at source scale \(1.0\) (grid \(70\times70\), endpoint comparison, reference: MC 2D-isotropic). RMSE and MAE are reported together with relative \(L^2\) error. Smaller error is highlighted in bold for each time.}
\label{tab:beam_time_error_endpoint_src1}
\begin{tabular}{llccc}
\hline
Time & Closure & Relative \(L^2\) error & RMSE & MAE \\
\hline
\multirow{2}{*}{\(t=0.35\)}
& HN   & \(\mathbf{4.5114\times10^{-1}}\) & \(\mathbf{1.3412\times10^{-2}}\) & \(\mathbf{5.1016\times10^{-3}}\) \\
& Levermore  & \(5.2848\times10^{-1}\)          & \(1.5711\times10^{-2}\)          & \(5.3511\times10^{-3}\) \\
\hline
\multirow{2}{*}{\(t=1.40\)}
& HN   & \(\mathbf{1.3102\times10^{-1}}\) & \(\mathbf{1.3665\times10^{-2}}\) & \(\mathbf{8.1473\times10^{-3}}\) \\
& Levermore  & \(1.4530\times10^{-1}\) & \(1.4135\times10^{-2}\) & \(8.2548\times10^{-3}\) \\
\hline
\multirow{2}{*}{\(t=2.45\)}
& HN   & \(\mathbf{1.6136\times10^{-1}}\) & \(\mathbf{1.8102\times10^{-2}}\) & \(\mathbf{1.1137\times10^{-2}}\) \\
& Levermore  & \(2.0611\times10^{-1}\)          & \(2.3122\times10^{-2}\)          & \(1.4174\times10^{-2}\) \\
\hline
\multirow{2}{*}{\(t=3.50\)}
& HN   & \(\mathbf{1.8927\times10^{-1}}\) & \(\mathbf{2.3493\times10^{-2}}\) & \(\mathbf{1.5946\times10^{-2}}\) \\
& Levermore  & \(2.2992\times10^{-1}\)          & \(2.8539\times10^{-2}\)          & \(1.9480\times10^{-2}\) \\
\hline
\multirow{2}{*}{All four times}
& HN   & \(\mathbf{2.4070\times10^{-1}}\) & \(\mathbf{1.7668\times10^{-2}}\) & \(\mathbf{1.0258\times10^{-2}}\) \\
& Levermore  & \(2.7745\times10^{-1}\)          & \(2.0377\times10^{-2}\)          & \(1.1815\times10^{-2}\) \\
\hline
\end{tabular}
\end{table}

The HN closure uses the same architecture as in the steady case, with a total of 24,202 trainable parameters (Table~\ref{tab:hpnn_network_specs}). The model is trained on time-dependent Monte Carlo data collected over all time instances, and the training history (Fig.~\ref{fig:loss_ex2}) shows convergence of all tensor components together with the wave-speed penalty.
As shown in Fig.~\ref{fig:actual_predicted_ex2}, the HN closure accurately reproduces the radiation pressure tensor components, with predictions closely aligned to the reference values. This is confirmed quantitatively in Table~\ref{tab:beam_crossing_closure_metrics_p}, where the test errors are on the order of $5.13\times10^{-4}$ for $P_{xx}$, $5.13\times10^{-4}$ for $P_{yy}$, and $5.30\times10^{-4}$ for $P_{xy}$, with corresponding $R^2$ values of $0.9928$, $0.9934$, and $0.9812$. 
In contrast, the Levermore closure exhibits substantially larger errors for the diagonal components ($\approx 5.7\times10^{-2}$) with lower correlation ($R^2\approx0.68$--$0.71$), while also showing low accuracy for the off-diagonal component $P_{xy}$ ($R^2=0.8844$).

The spectral behavior of the learned closure is examined in Fig.~\ref{fig:histrogram_ex2}, where the directional Jacobian eigenvalues remain real across all sampled directions, with zero fraction of nonzero imaginary parts.
The resulting DG solutions are shown in Fig.~\ref{fig:solution_ex2}, and the corresponding quantitative errors are reported in Table~\ref{tab:beam_time_error_endpoint_src1}. 
Overall, the HN closure consistently provides a moderate improvement in accuracy over the Levermore closure across all examined time instances. For example, at $t=0.35$, the relative $L^2$ error is reduced from $5.28\times10^{-1}$ to $4.51\times10^{-1}$, while at $t=3.50$, it decreases from $2.30\times10^{-1}$ to $1.89\times10^{-1}$. 
However, compared to the lattice problem (Example~\ref{ex1}), the improvement observed in this beam-crossing test is less pronounced. The strongly anisotropic transport behavior induced by the crossing beams makes this problem considerably more challenging for low-order moment closures.

\subsection{Experiment 3: Crooked pipe problem}
\label{ex3}

We next consider a more challenging crooked-pipe benchmark designed to assess the closure performance in a strongly anisotropic transport environment. The computational configuration is shown in Fig.~\ref{fig:configuration_ex3}. The domain consists of an absorbing background region (\(\sigma_a=10,\ \sigma_s=0\)) surrounding a scattering pipe structure (\(\sigma_a=0,\ \sigma_s=1\)). Radiation is injected through a localized unit source patch located near the pipe entrance, and the resulting transport is constrained to propagate through multiple turns of the narrow channel. This geometry induces pronounced directional transport and nonlocal interactions, making accurate reconstruction of the radiation pressure tensor particularly important.

\begin{figure}[htp!]
    \centering
    \begin{subfigure}[b]{0.8\linewidth}
        \centering
        \includegraphics[width=\linewidth]{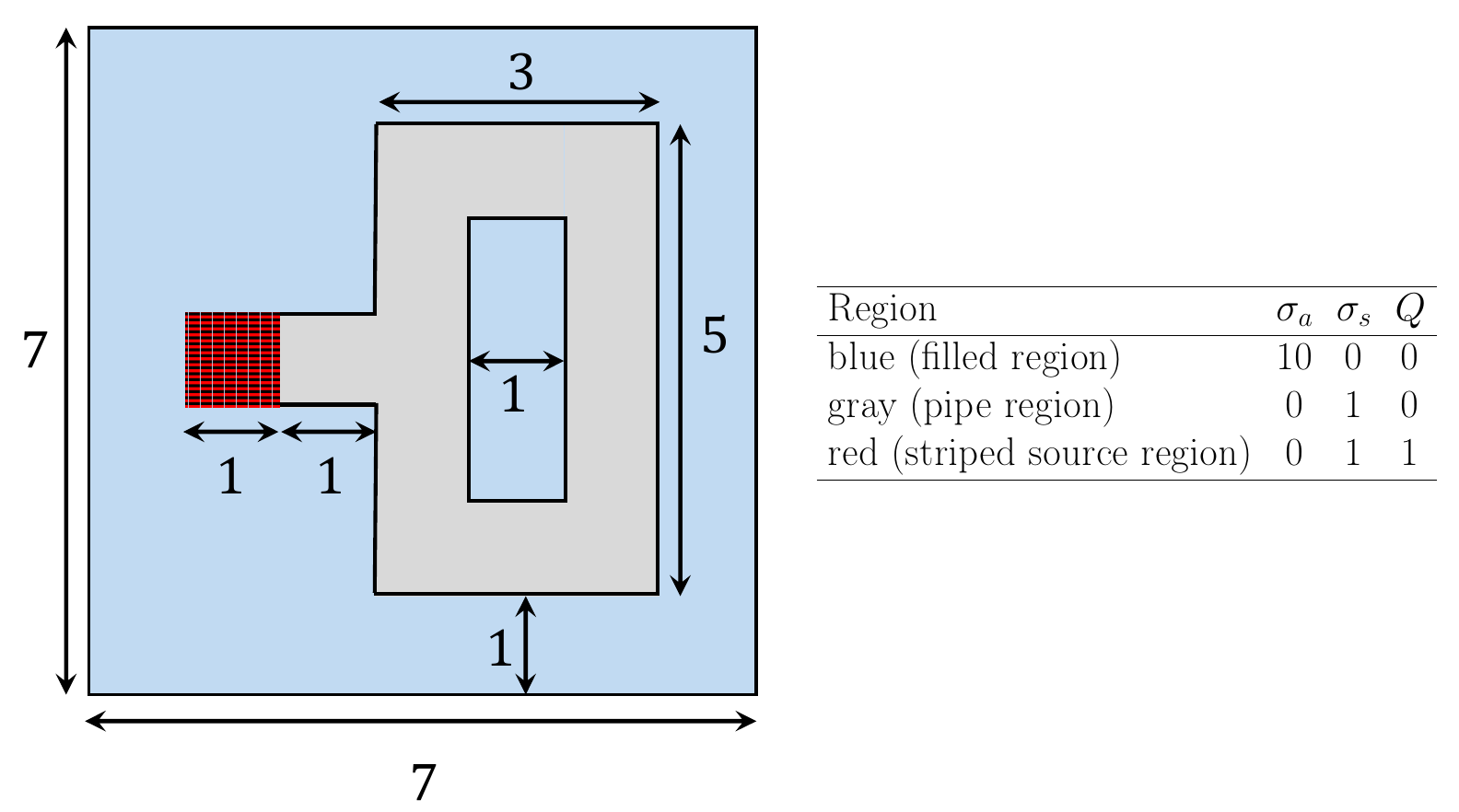}
    \end{subfigure}
    \caption{Crooked pipe configuration on a \(7\times7\) square domain \(\Omega=[0,7]\times[0,7]\), where \(x\in[0,7]\) and \(y\in[0,7]\). The blue filled region is absorbing, with \(\sigma_a=10\) and \(\sigma_s=0\). The gray pipe region is scattering, with \(\sigma_a=0\) and \(\sigma_s=1\). A unit square source patch, shown in red stripes, is located at \((1<x<2,\;3<y<4)\), where \(Q=1\); elsewhere \(Q=0\).}
    \label{fig:configuration_ex3}
\end{figure}

The HN closure employs the same neural network architecture used in the previous experiments, with the detailed network specifications summarized in Table~\ref{tab:hpnn_network_specs_3}. The model contains a total of 24,202 trainable parameters and is trained using Monte Carlo reference data. 
As shown in Fig.~\ref{fig:loss_ex3}, the losses associated with \(P_{xx}\), \(P_{yy}\), \(P_{xy}\), and the wave speed penalty decrease steadily throughout training, indicating stable optimization of both closure accuracy and the wave-speed constraint.

\begin{table}[htbp!]
\footnotesize
\renewcommand{\arraystretch}{1.08}
\setlength{\tabcolsep}{4pt}
\centering
\caption{Neural network specifications of the HN closure for the crooked pipe problem.}
\label{tab:hpnn_network_specs_3}
\begin{tabular}{lccc}
\hline
 & \makecell[c]{Entropy \\ network $\eta_\theta$}
 & \makecell[c]{Symmetric \\ network $\mathcal{N}_f$}
 & \makecell[c]{Integration constant \\ network $\mathcal{N}_A$} \\
\hline
Input dimension
& $3$
& $9$
& $9$ \\

Output dimension
& $1$
& $6$
& $3$ \\

Hidden layers
& $[128]$
& $[128,128]$
& $[64,64]$ \\

Activation
& \makecell[c]{Softplus}
& Tanh
& Tanh \\

Trainable parameters
& $641$
& $18{,}566$
& $4{,}995$ \\
\hline
Total trainable parameters
& \multicolumn{3}{c}{$24{,}202$} \\
\hline
\end{tabular}
\end{table}

\begin{figure}[htp!]
    \centering
    \begin{subfigure}[b]{0.8\linewidth}
        \centering
        \includegraphics[width=\linewidth]{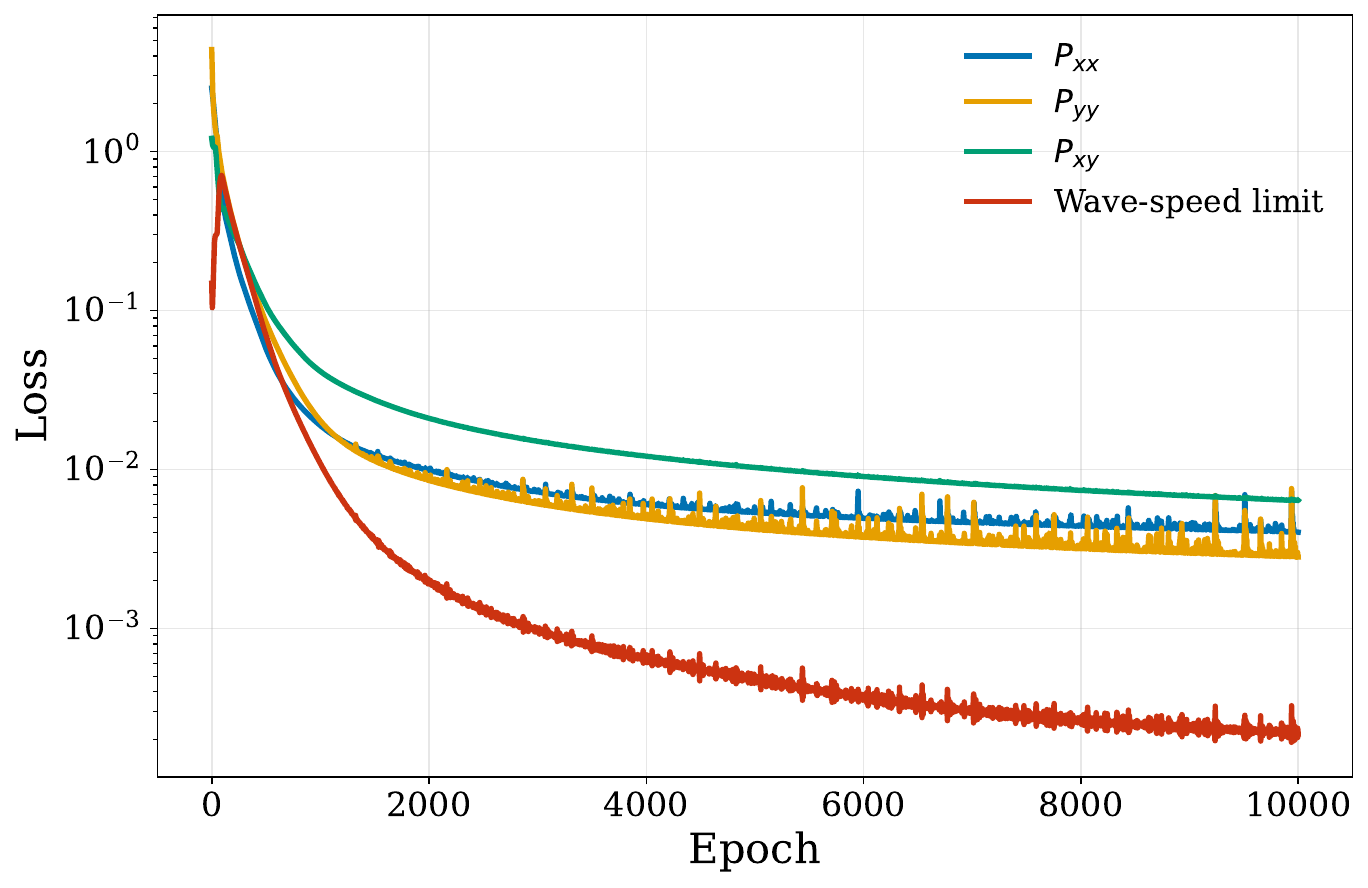}
    \end{subfigure} 
    \caption{Training loss histories of the HN closure, showing the component losses for $P_{xx}$, $P_{yy}$, $P_{xy}$, and the wave-speed limit penalty.}
    \label{fig:loss_ex3}
\end{figure}

\begin{figure}[htp!]
    \centering
    \begin{subfigure}[b]{0.8\linewidth}
        \centering
        \includegraphics[width=\linewidth]{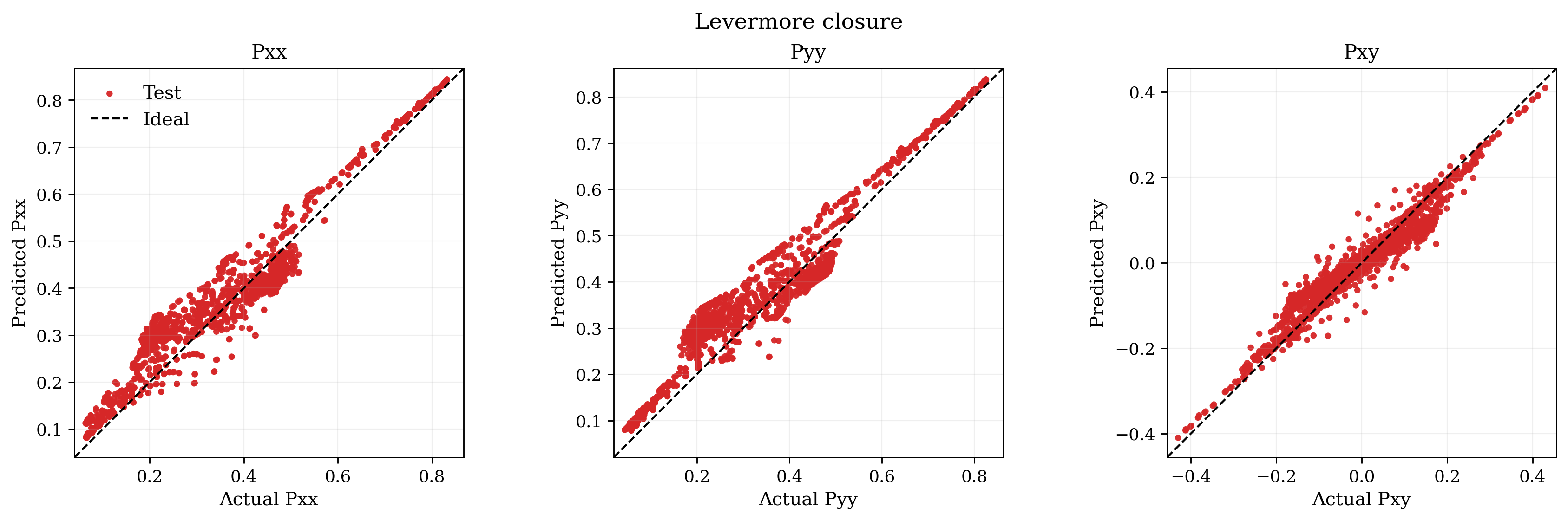}
        \subcaption{Levermore closure}
    \end{subfigure} \\
    \begin{subfigure}[b]{0.8\linewidth}
        \centering
        \includegraphics[width=\linewidth]{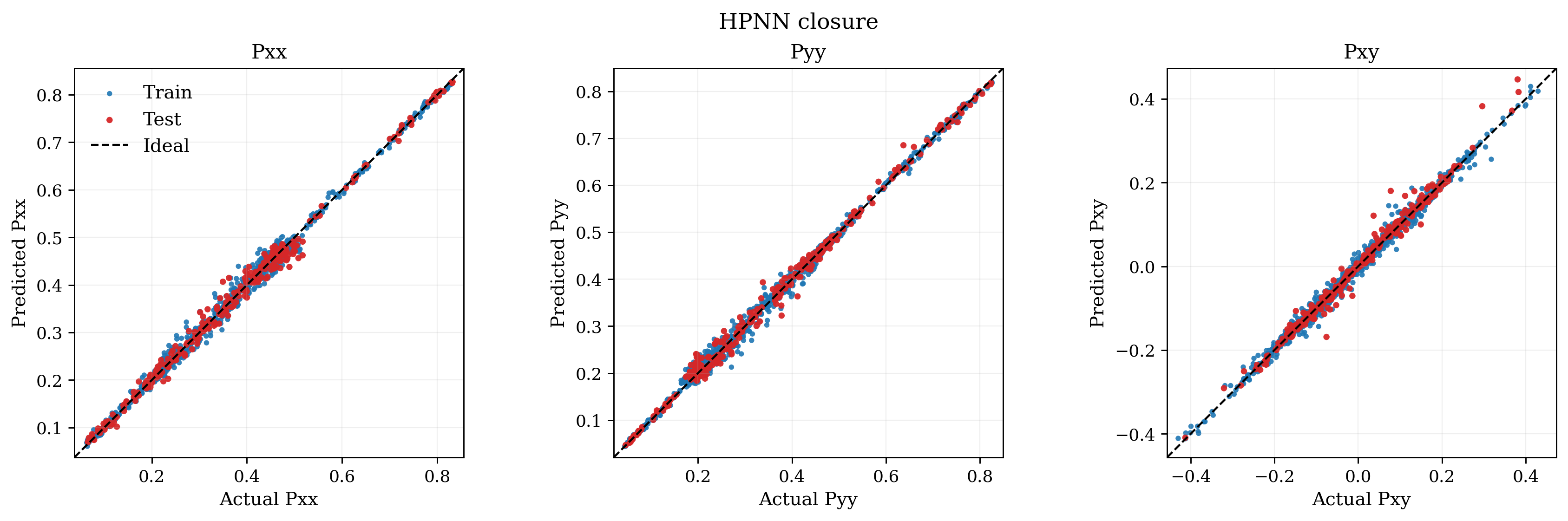}
        \subcaption{\textbf{HN closure}}
    \end{subfigure}
    \caption{Actual vs predicted scatter comparison of closure outputs $(P_{xx}, P_{yy}, P_{xy})$ for the two-dimensional steady radiation moment system (crooked pipe problem). In the HN panel, training samples are shown in blue and testing samples in red; the dashed line indicates the ideal relation $y=x$.}
    \label{fig:actual_predicted_ex3}
\end{figure}

The closure accuracy is summarized in Table~\ref{tab:crooked_pipe_closure_metrics_p}. On the test set, the HN closure achieves \(R^2\) values of 0.9936, 0.9945, and 0.9840 for \(P_{xx}\), \(P_{yy}\), and \(P_{xy}\), respectively. In comparison, the Levermore closure yields lower \(R^2\) values and larger prediction errors across all tensor components. These results indicate that the proposed closure provides a more accurate approximation of the radiation pressure tensor for the crooked pipe configuration.

The spectral properties of the learned closure are examined in Fig.~\ref{fig:histrogram_ex3}. For all sampled states and all four directions considered, the directional Jacobian exhibits real-valued eigenvalues with no nonzero imaginary components observed.
This behavior is consistent with the symmetrizable construction of the proposed HN closure and confirms that all sampled flux Jacobians remain free from complex eigenvalue violations.

\begin{table}[h!]
\footnotesize
\centering
\caption{Closure accuracy comparison for Example 3 (crooked pipe problem) on a $7\times7$ domain. The HN closure is trained using Monte Carlo reference data generated on a $70\times70$ crooked pipe simulation with a total of $8\times5\times10^6$ particles. The source is fixed at $Q=1$ on $(1<x<2,\;3<y<4)$, with 80\% of the data used for training and 20\% for testing. The neural network uses hidden width 128 with Tanh activation, quadrature points $N_q=4$, wave speed limiting in four directions, and is trained for 10{,}000 epochs. Errors are reported for $P_{xx}$, $P_{yy}$, and $P_{xy}$.}
\label{tab:crooked_pipe_closure_metrics_p}
\begin{tabular}{lllcccc}
\hline
Closure & Split & Component & MSE & $R^2$ & $\max\lvert P-\hat P\rvert$ & $\mathrm{mean}\,\lvert P-\hat P\rvert$ \\
\hline
\multirow{6}{*}{HN closure}
& \multirow{3}{*}{Train}
& $P_{xx}$ & $1.1535\times10^{-4}$ & $0.9960$ & $5.7151\times10^{-2}$ & $7.3368\times10^{-3}$ \\
& & $P_{yy}$ & $9.1294\times10^{-5}$ & $0.9972$ & $5.8339\times10^{-2}$ & $6.4190\times10^{-3}$ \\
& & $P_{xy}$ & $1.0932\times10^{-4}$ & $0.9937$ & $7.1976\times10^{-2}$ & $7.0609\times10^{-3}$ \\ \cline{2-7}
& \multirow{3}{*}{Test}
& $P_{xx}$ & $1.8354\times10^{-4}$ & $0.9936$ & $5.7327\times10^{-2}$ & $9.6903\times10^{-3}$ \\
& & $P_{yy}$ & $1.7560\times10^{-4}$ & $0.9945$ & $5.5433\times10^{-2}$ & $9.2756\times10^{-3}$ \\
& & $P_{xy}$ & $2.5415\times10^{-4}$ & $0.9840$ & $1.0250\times10^{-1}$ & $9.1376\times10^{-3}$ \\
\hline
\multirow{3}{*}{Levermore closure}
& \multirow{3}{*}{All (Test)}
& $P_{xx}$ & $3.1311\times10^{-3}$ & $0.8905$ & $1.2833\times10^{-1}$ & $4.6629\times10^{-2}$ \\
& & $P_{yy}$ & $4.0439\times10^{-3}$ & $0.8752$ & $1.3308\times10^{-1}$ & $5.4434\times10^{-2}$ \\
& & $P_{xy}$ & $1.0780\times10^{-3}$ & $0.9370$ & $1.3004\times10^{-1}$ & $2.3554\times10^{-2}$ \\
\hline
\end{tabular}
\end{table}

\begin{figure}[h!]
    \centering
    \begin{subfigure}[b]{0.8\linewidth}
        \centering
        \includegraphics[width=\linewidth]{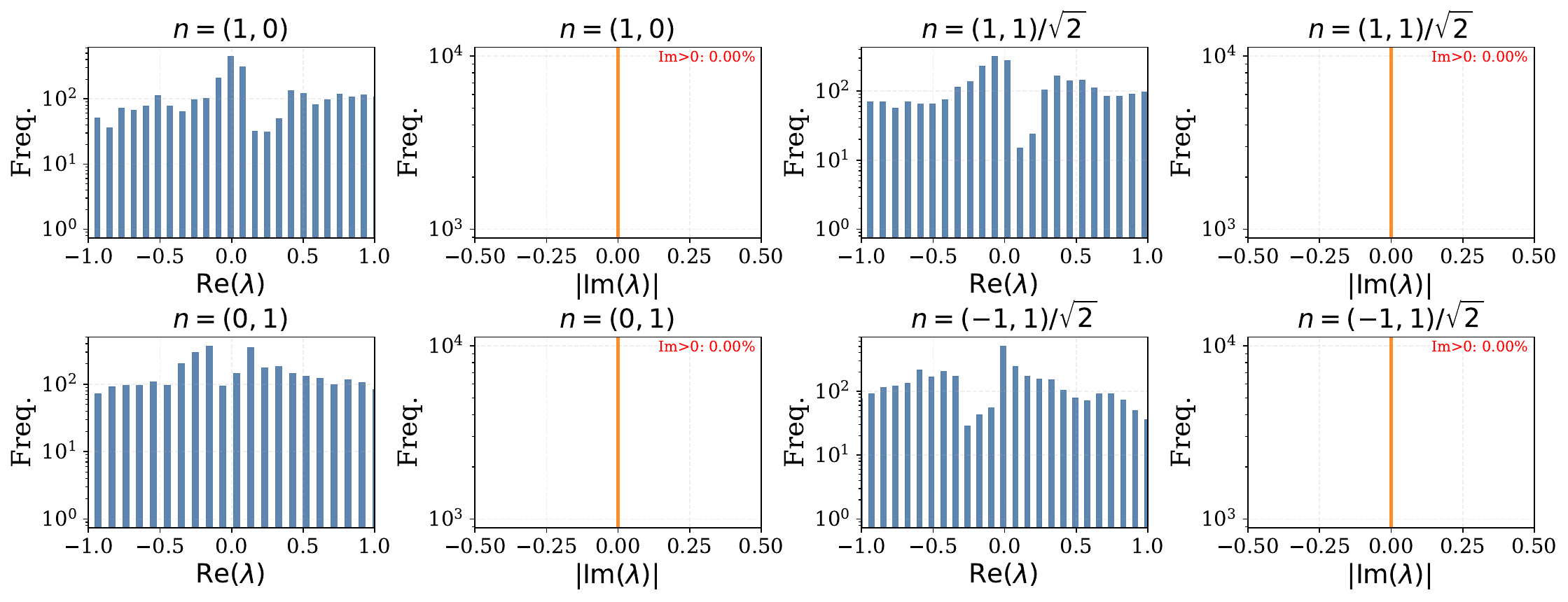}
    \end{subfigure}
    \caption{Eigenvalue histograms of the directional Jacobian $A(\mathbf{n})=\mathbf{n}_x J_x+\mathbf{n}_y J_y$ for the trained HN closure in the crooked pipe radiative transfer problem. Histograms are computed from a total of \(4900\) sampled states. For each direction $\mathbf{n}\in\{(1,0),(0,1),(1,1)/\sqrt{2},(-1,1)/\sqrt{2}\}$, the left panel shows the distribution of $\mathrm{Re}(\lambda)$ and the right panel shows $|\mathrm{Im}(\lambda)|$. The annotated fraction $\mathrm{Im}>0$ is $0.00\%$ in all cases, indicating real-valued spectra for these sampled states.}
    \label{fig:histrogram_ex3}
\end{figure}

The solution comparisons for three material configurations are shown in Fig.~\ref{fig:solution_separate_ex3}, and the corresponding error metrics are reported in Table~\ref{tab:three_case_error_ex3}. In all cases, both closures reproduce the overall transport pattern through the scattering pipe, while the HN-DG solution remains closer to the Monte Carlo reference. 
The error fields in Fig.~\ref{fig:solution_separate_ex3} show that the largest discrepancies are concentrated near the high-intensity source region, where the Levermore-DG exhibits numerical dissipation and tends to underestimate the energy distribution.
Compared to the Levermore-DG, the HN-DG exhibits smaller local errors and preserves sharper energy profiles, indicating reduced numerical dissipation.

\begin{figure}[h!]
    \centering
    \begin{subfigure}[b]{1\linewidth}
        \centering
        \includegraphics[width=\linewidth]{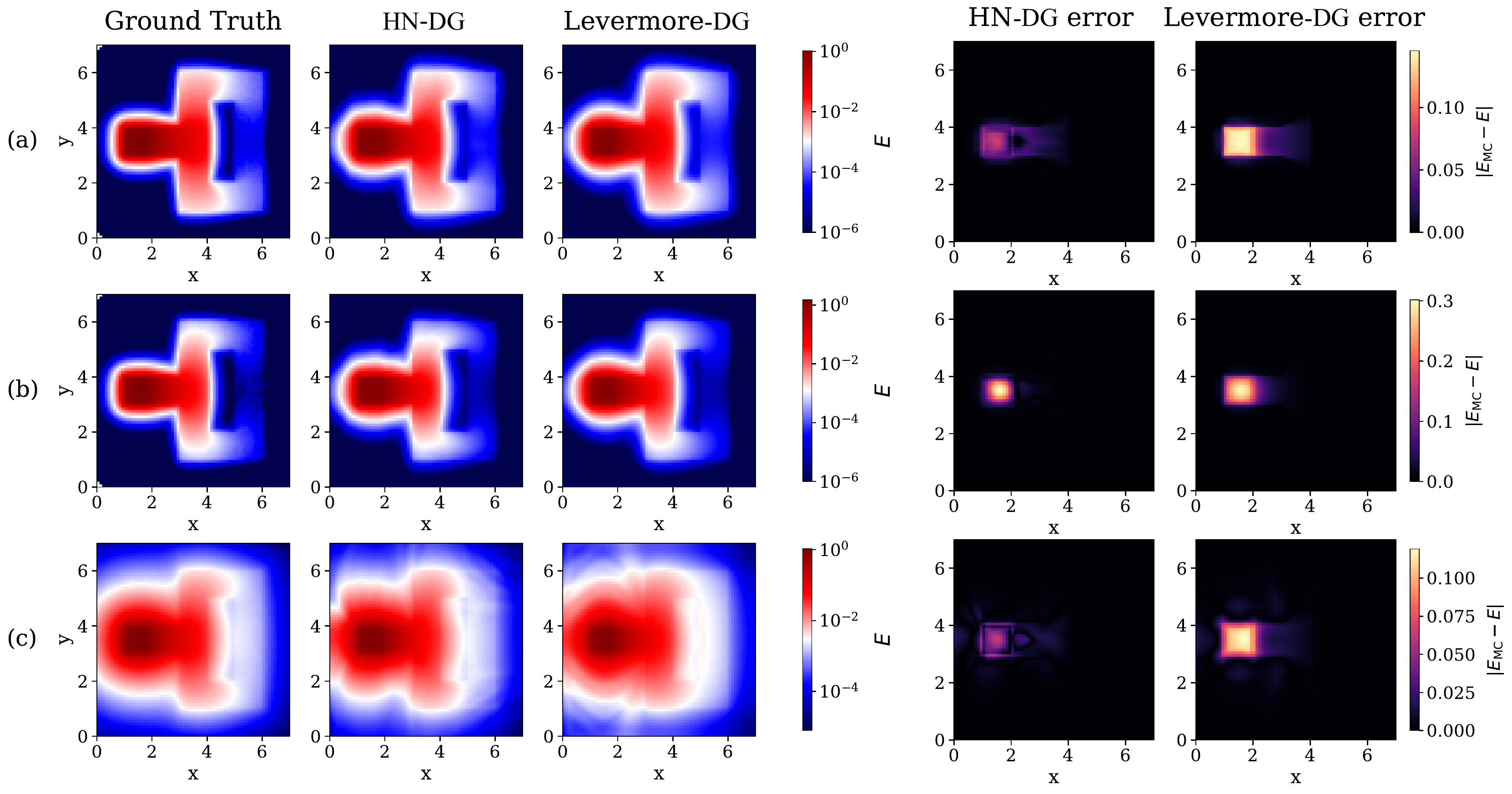}
    \end{subfigure}
    \caption{Solution comparison for the crooked pipe problem on a \(70\times70\) grid over \(\Omega=[0,7]\times[0,7]\), with source strength fixed at \(Q=1\) on \((1<x<2,\;3<y<4)\). Each row corresponds to one material configuration: (a) the baseline case with pipe coefficients \((\sigma_a,\sigma_s)=(0,1)\) and filled-region coefficients \((10,0)\); (b) the high-scattering pipe case with \((0,3)\) in the pipe and \((10,0)\) in the filled region; and (c) the barrier leakage case with \((0,1)\) in the pipe and \((2,0)\) in the filled region. Columns show the Monte Carlo reference, the HN-DG solution, the Levermore-DG solution, and the corresponding absolute-error fields \(|E_{\mathrm{MC}}-E|\) for HN-DG and Levermore-DG, respectively. All Monte Carlo references are ensemble averages based on \(8\times 5\times10^6\) particles.}
    \label{fig:solution_separate_ex3}
\end{figure}

Quantitatively, the HN-DG achieves lower relative \(L^2\) error, RMSE, and MAE for all three material settings. For the base configuration in Fig.~\ref{fig:solution_separate_ex3}a, the relative \(L^2\) error is reduced from \(1.66\times10^{-1}\) to \(7.74\times10^{-2}\). Similar improvements are observed for the high-scattering pipe case in Fig.~\ref{fig:solution_separate_ex3}b and the barrier leakage configuration in Fig.~\ref{fig:solution_separate_ex3}c. These results indicate that the improved closure accuracy translates into improved solution accuracy across a range of material parameters.

\begin{table}[h!]
\centering
\footnotesize
\caption{Error comparison for three material settings in the crooked pipe problem (Example 3). Here, Pipe denotes the white scattering pipe region with material coefficients \((\sigma_a,\sigma_s)\), and Fill denotes the blue filled region with material coefficients \((\sigma_a,\sigma_s)\). The reference is the Monte Carlo solution computed using \(8\times5\times10^6\) particles on a \(70\times70\) grid. RMSE (root mean square error), MAE (mean absolute error), and the relative \(L^2\) error are reported. Smaller error is highlighted in bold for each metric for each setting.}
\label{tab:three_case_error_ex3}
\begin{tabular}{llccc}
\hline
Setting & Closure & Relative \(L^2\) error & RMSE & MAE \\
\hline
\multirow{2}{*}{Fig.~\ref{fig:solution_separate_ex3}a: Pipe\((0,1)\), Fill\((10,0)\)}
& HN      & \(\mathbf{7.7416\times10^{-2}}\) & \(\mathbf{9.9086\times10^{-3}}\) & \(\mathbf{2.6991\times10^{-3}}\) \\
& Levermore & \(1.6559\times10^{-1}\)          & \(2.1194\times10^{-2}\)          & \(5.0183\times10^{-3}\) \\
\hline
\multirow{2}{*}{Fig.~\ref{fig:solution_separate_ex3}b: Pipe\((0,3)\), Fill\((10,0)\)}
& HN      & \(\mathbf{1.5211\times10^{-1}}\) & \(\mathbf{2.5647\times10^{-2}}\) & \(\mathbf{4.3513\times10^{-3}}\) \\
& Levermore & \(1.9839\times10^{-1}\)          & \(3.3450\times10^{-2}\)          & \(6.6687\times10^{-3}\) \\
\hline
\multirow{2}{*}{Fig.~\ref{fig:solution_separate_ex3}c: Pipe\((0,1)\), Fill\((2,0)\)}
& HN      & \(\mathbf{5.2037\times10^{-2}}\) & \(\mathbf{6.9305\times10^{-3}}\) & \(\mathbf{2.3118\times10^{-3}}\) \\
& Levermore & \(1.2718\times10^{-1}\)          & \(1.6938\times10^{-2}\)          & \(5.0180\times10^{-3}\) \\
\hline
\end{tabular}
\end{table}

\section{Conclusion}
\label{sec:conclusion}
In this work, we proposed a hyperbolic neural closure for the M1 radiative transfer system by constructing the closure through a structured flux Jacobian representation. Instead of directly regressing the radiation pressure tensor, the method learns a Jacobian field using a strictly convex entropy network and a symmetric neural network, and recovers the closure through path integration. 
This construction preserves the M1 flux structure while ensuring that, for each frozen set of gradient-based features, the directional Jacobian is similar to a symmetric matrix.
Another important advantage of the proposed neural closure is its ability to incorporate gradient-based nonlocal information that is not available in conventional entropy-based analytical closures. As illustrated in Fig.~\ref{fig:actual_predicted_ex1}, incorporating gradient information significantly improves closure accuracy compared with both the analytical Levermore closure.

The numerical experiments demonstrate that the proposed approach improves both closure modeling and DG solution quality while maintaining stable wave propagation. In the lattice benchmark (Experiment~\ref{ex1}), the learned closure achieves substantially lower tensor prediction errors than the classical Levermore closure and yields consistently improved numerical solutions across multiple material configurations. 
In the beam-crossing problem (Experiment~\ref{ex2}), the learned closure provides improved representation of anisotropic transport behavior and leads to lower average solution errors under strongly directional transport conditions. 
In the crooked pipe problem (Experiment~\ref{ex3}), the learned closure consistently produces smaller solution errors across all tested material settings.

Overall, the proposed method provides a data-driven closure framework with real-valued directional Jacobian spectra through a symmetrizable construction. The resulting closure improves both closure accuracy and solution quality while remaining compatible with DG-based radiative transfer solvers. Future work includes improving the representation of strongly anisotropic transport and extending the framework toward a unified closure model applicable to a broader range of geometries and boundary conditions.

\section*{Acknowledgment}
We would like to thank the support of National Science Foundation (DMS-2533878, DMS-2053746, DMS-2134209, ECCS-2328241, CBET-2347401 and OAC-2311848), and U.S.~Department of Energy (DOE) Office of Science Advanced Scientific Computing Research program DE-SC0023161, the SciDAC LEADS Institute, and DOE–Fusion Energy Science, under grant number: DE-SC0024583.

\bibliographystyle{unsrt}
\bibliography{mybib}

\appendix

\section{Sensitivity analysis}

To examine the relative importance of the input variables, we perform a correlation-based sensitivity analysis by evaluating the normalized correlation between each input feature and the pressure tensor components. Figure~\ref{fig:sensitivity_3dbar_journal_augmented} reports the relative influence of the state variables $(E,F_x,F_y)$ and the gradient-based features $(\nabla\!\cdot\!F,(\nabla\times F)_z,\|\nabla E\|^2)$ on each component of the predicted pressure tensor.

The results show that the diagonal components $P_{xx}$ and $P_{yy}$ are primarily influenced by the local state variables, with comparatively weaker dependence on gradient-based features. In contrast, the off-diagonal component $P_{xy}$ exhibits strong sensitivity to the gradient-based inputs, particularly $\nabla\!\cdot\!F$, $(\nabla\times F)_z$, and $\|\nabla E\|^2$. This indicates that the accurate representation of anisotropic stress requires local spatial variation information that cannot be captured solely by $(E,F_x,F_y)$.
These observations provide empirical support for the inclusion of gradient-based features in the closure model. In particular, they are consistent with the design choice in Sec.~\ref{Sec:method}, where the closure is augmented by $G(u)$ to improve the representation of nonlocal and anisotropic transport effects.

\begin{figure}[htp!]
    \centering
    \begin{subfigure}[b]{0.5\linewidth}
        \centering
        \includegraphics[width=\linewidth]{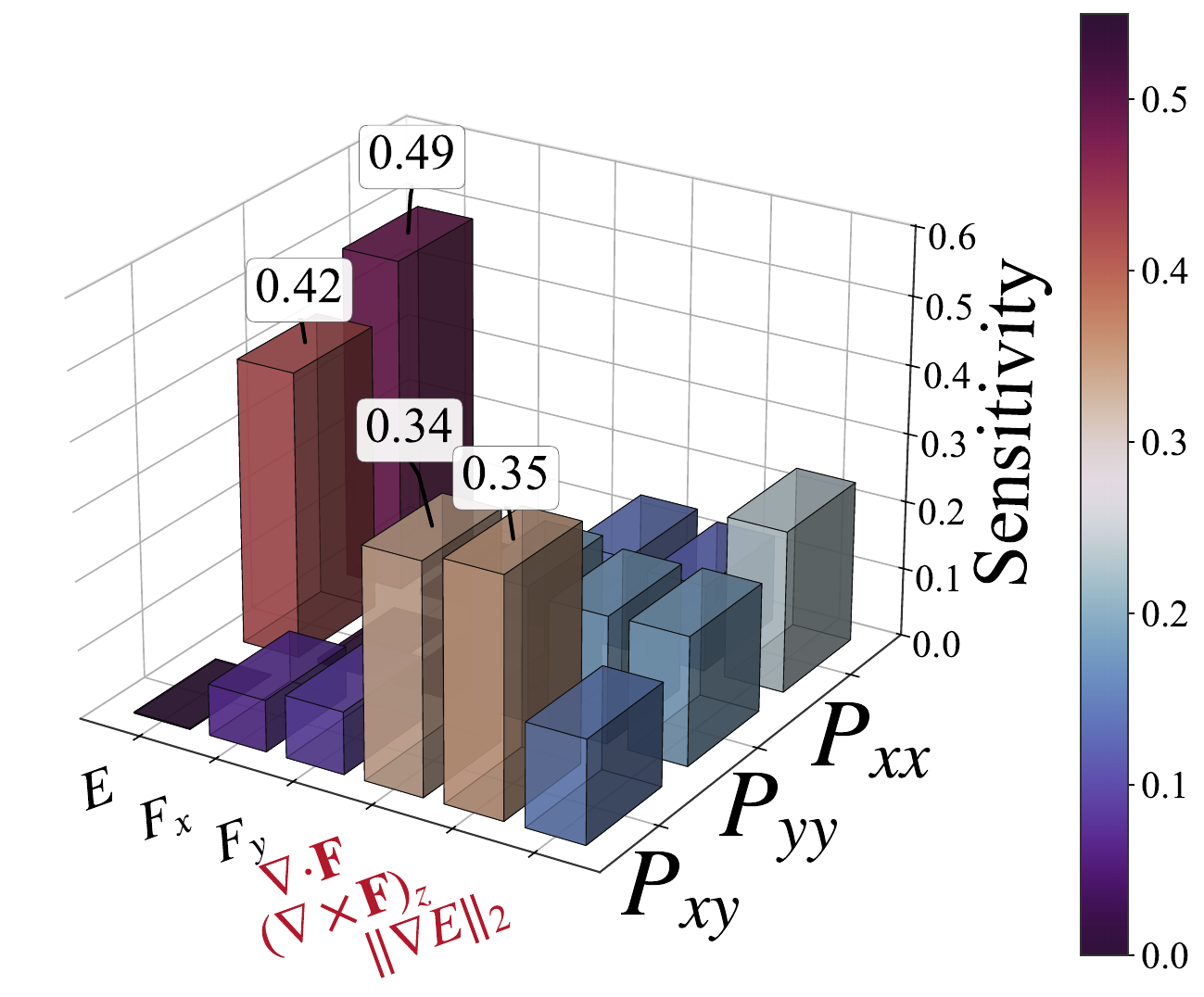}
    \end{subfigure}
\caption{Sensitivity analysis:
The bars indicate the relative influence of the input variables
$(E,F_x,F_y,\nabla\!\cdot\!F,(\nabla\times F)_z,\|\nabla E\|^2)$
on the predicted pressure tensor components $(P_{xx},P_{yy},P_{xy})$.
Notably, the off-diagonal component $P_{xy}$ shows strong dependence on gradient-based features,
such as $\nabla\!\cdot\!F$, $(\nabla\times F)_z$ and $\|\nabla E\|^2$, indicating that anisotropic stress requires local spatial information beyond $(E,F_x,F_y)$.}
\label{fig:sensitivity_3dbar_journal_augmented}
\end{figure}

\section{Determinant and trace comparison of closure models}

To examine the structural behavior of the learned closure, we compare the determinant and trace of the Eddington tensor produced by the HN closure with those of the classical Levermore, MEFD, and Janka closures, using the Monte Carlo solutions as reference. Figure~\ref{fig:det_trace} shows that the analytic closures follow smooth low-dimensional curves determined by their explicit functional forms, whereas both the Monte Carlo reference and the HN closure exhibit a broader distribution over sampled admissible states. The HN closure closely follows the Monte Carlo reference in both determinant and trace, capturing the structural variability that cannot be represented by the analytic closures. At the same time, it preserves spectral characteristics comparable to the classical closures, with the determinant exhibiting a similar decreasing trend in the high flux-factor regime and the trace remaining bounded over the sampled admissible states.

\begin{figure}[htp!]
    \centering
    \begin{subfigure}[b]{1\linewidth}
        \centering
        \includegraphics[width=\linewidth]{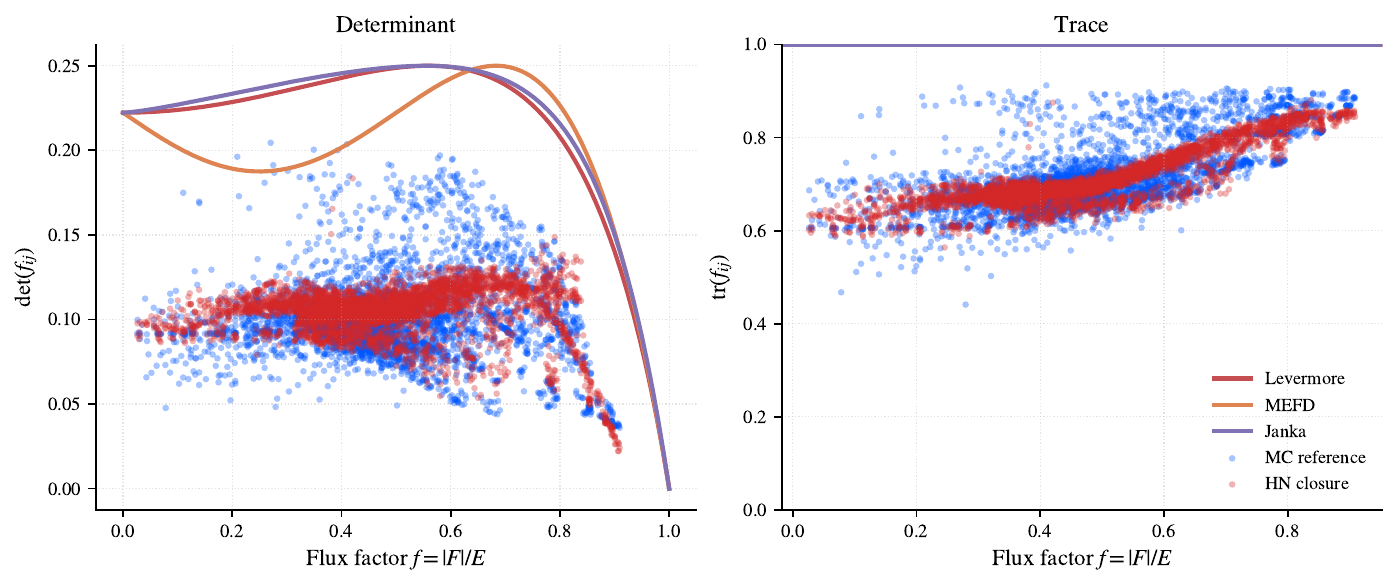}
    \end{subfigure}
\caption{
Determinant and trace of the Eddington tensor for the proposed HN closure, compared with the classical Levermore, MEFD, and Janka closures over sampled admissible states. 
The blue scattered points denote the Monte Carlo reference data, while the red scattered points denote predictions of the trained HN closure. The solid curves correspond to the analytical closures listed in the legend. 
A notable feature is that the HN closure follows the Monte Carlo reference more closely than the analytic closures in both determinant and trace.
}
\label{fig:det_trace}
\end{figure}

\section{Matrix formulation for the hyperbolic neural closure}

To make the structure of the proposed closure explicit, we express the pressure gradients in terms of the entropy Hessian and the factorized Jacobian representation. Recall that the directional Jacobians are constructed as $J_x = S_x H$ and $J_y = S_y H$ by Eq.~\eqref{eq:m1_Sx_Sy_explicit}, where $H(\mathbf{u})$ is the Hessian of the learned entropy. This formulation allows the gradients of the pressure tensor components to be directly extracted from specific rows of the Jacobians. 

Writing
\begin{equation}
H(\mathbf{u})
=
\begin{bmatrix}
\frac{\partial^2 \eta_\theta(\mathbf{u})}{\partial E^2} &
\frac{\partial^2 \eta_\theta(\mathbf{u})}{\partial E\,\partial F_x} &
\frac{\partial^2 \eta_\theta(\mathbf{u})}{\partial E\,\partial F_y}\\
\frac{\partial^2 \eta_\theta(\mathbf{u})}{\partial F_x\,\partial E} &
\frac{\partial^2 \eta_\theta(\mathbf{u})}{\partial F_x^2} &
\frac{\partial^2 \eta_\theta(\mathbf{u})}{\partial F_x\,\partial F_y}\\
\frac{\partial^2 \eta_\theta(\mathbf{u})}{\partial F_y\,\partial E} &
\frac{\partial^2 \eta_\theta(\mathbf{u})}{\partial F_y\,\partial F_x} &
\frac{\partial^2 \eta_\theta(\mathbf{u})}{\partial F_y^2}
\end{bmatrix},
\end{equation}
the pressure gradients are obtained from the lower rows of $J_x$ and $J_y$, namely $\nabla_{\mathbf{u}} P_{xx} = (J_x)_{2,:}$, $\nabla_{\mathbf{u}} P_{xy}^{(x)} = (J_x)_{3,:}$, $\nabla_{\mathbf{u}} P_{xy}^{(y)} = (J_y)_{2,:}$, and $\nabla_{\mathbf{u}} P_{yy} = (J_y)_{3,:}$.
More explicitly,
\begin{equation}
\nabla_{\mathbf{u}} P_{xx}
=
\begin{bmatrix}
\frac{\partial F_x}{\partial v_2}\frac{\partial^2 \eta_\theta(\mathbf{u})}{\partial E^2}
+
s_{22}^{(x)}\frac{\partial^2 \eta_\theta(\mathbf{u})}{\partial E\,\partial F_x}
+
s_{23}^{(x)}\frac{\partial^2 \eta_\theta(\mathbf{u})}{\partial E\,\partial F_y}\\
\frac{\partial F_x}{\partial v_2}\frac{\partial^2 \eta_\theta(\mathbf{u})}{\partial E\,\partial F_x}
+
s_{22}^{(x)}\frac{\partial^2 \eta_\theta(\mathbf{u})}{\partial F_x^2}
+
s_{23}^{(x)}\frac{\partial^2 \eta_\theta(\mathbf{u})}{\partial F_x\,\partial F_y}\\
\frac{\partial F_x}{\partial v_2}\frac{\partial^2 \eta_\theta(\mathbf{u})}{\partial E\,\partial F_y}
+
s_{22}^{(x)}\frac{\partial^2 \eta_\theta(\mathbf{u})}{\partial F_x\,\partial F_y}
+
s_{23}^{(x)}\frac{\partial^2 \eta_\theta(\mathbf{u})}{\partial F_y^2}
\end{bmatrix}^{\!\top},
\label{eq:m1_grad_pxx_explicit}
\end{equation}
\begin{equation}
\nabla_{\mathbf{u}} P_{xy}^{(x)}
=
\begin{bmatrix}
\frac{\partial F_y}{\partial v_2}\frac{\partial^2 \eta_\theta(\mathbf{u})}{\partial E^2}
+
s_{23}^{(x)}\frac{\partial^2 \eta_\theta(\mathbf{u})}{\partial E\,\partial F_x}
+
s_{33}^{(x)}\frac{\partial^2 \eta_\theta(\mathbf{u})}{\partial E\,\partial F_y}\\
\frac{\partial F_y}{\partial v_2}\frac{\partial^2 \eta_\theta(\mathbf{u})}{\partial E\,\partial F_x}
+
s_{23}^{(x)}\frac{\partial^2 \eta_\theta(\mathbf{u})}{\partial F_x^2}
+
s_{33}^{(x)}\frac{\partial^2 \eta_\theta(\mathbf{u})}{\partial F_x\,\partial F_y}\\
\frac{\partial F_y}{\partial v_2}\frac{\partial^2 \eta_\theta(\mathbf{u})}{\partial E\,\partial F_y}
+
s_{23}^{(x)}\frac{\partial^2 \eta_\theta(\mathbf{u})}{\partial F_x\,\partial F_y}
+
s_{33}^{(x)}\frac{\partial^2 \eta_\theta(\mathbf{u})}{\partial F_y^2}
\end{bmatrix}^{\!\top},
\label{eq:m1_grad_pxyx_explicit}
\end{equation}
\begin{equation}
\nabla_{\mathbf{u}} P_{xy}^{(y)}
=
\begin{bmatrix}
\frac{\partial F_x}{\partial v_3}\frac{\partial^2 \eta_\theta(\mathbf{u})}{\partial E^2}
+
s_{22}^{(y)}\frac{\partial^2 \eta_\theta(\mathbf{u})}{\partial E\,\partial F_x}
+
s_{23}^{(y)}\frac{\partial^2 \eta_\theta(\mathbf{u})}{\partial E\,\partial F_y}\\
\frac{\partial F_x}{\partial v_3}\frac{\partial^2 \eta_\theta(\mathbf{u})}{\partial E\,\partial F_x}
+
s_{22}^{(y)}\frac{\partial^2 \eta_\theta(\mathbf{u})}{\partial F_x^2}
+
s_{23}^{(y)}\frac{\partial^2 \eta_\theta(\mathbf{u})}{\partial F_x\,\partial F_y}\\
\frac{\partial F_x}{\partial v_3}\frac{\partial^2 \eta_\theta(\mathbf{u})}{\partial E\,\partial F_y}
+
s_{22}^{(y)}\frac{\partial^2 \eta_\theta(\mathbf{u})}{\partial F_x\,\partial F_y}
+
s_{23}^{(y)}\frac{\partial^2 \eta_\theta(\mathbf{u})}{\partial F_y^2}
\end{bmatrix}^{\!\top},
\label{eq:m1_grad_pxyy_explicit}
\end{equation}
\begin{equation}
\nabla_{\mathbf{u}} P_{yy}
=
\begin{bmatrix}
\frac{\partial F_y}{\partial v_3}\frac{\partial^2 \eta_\theta(\mathbf{u})}{\partial E^2}
+
s_{23}^{(y)}\frac{\partial^2 \eta_\theta(\mathbf{u})}{\partial E\,\partial F_x}
+
s_{33}^{(y)}\frac{\partial^2 \eta_\theta(\mathbf{u})}{\partial E\,\partial F_y}\\
\frac{\partial F_y}{\partial v_3}\frac{\partial^2 \eta_\theta(\mathbf{u})}{\partial E\,\partial F_x}
+
s_{23}^{(y)}\frac{\partial^2 \eta_\theta(\mathbf{u})}{\partial F_x^2}
+
s_{33}^{(y)}\frac{\partial^2 \eta_\theta(\mathbf{u})}{\partial F_x\,\partial F_y}\\
\frac{\partial F_y}{\partial v_3}\frac{\partial^2 \eta_\theta(\mathbf{u})}{\partial E\,\partial F_y}
+
s_{23}^{(y)}\frac{\partial^2 \eta_\theta(\mathbf{u})}{\partial F_x\,\partial F_y}
+
s_{33}^{(y)}\frac{\partial^2 \eta_\theta(\mathbf{u})}{\partial F_y^2}
\end{bmatrix}^{\!\top}.
\label{eq:m1_grad_pyy_explicit}
\end{equation}

\section{Line integration}

The pressure tensor is reconstructed from its gradients by line integration in the state space. For implementation, we consider a simple reference path connecting the origin to the target state, parameterized as $t\mathbf{u}$ for $t\in[0,1]$. We emphasize that this reconstruction is not intended to enforce a path-independent representation of the closure. In general, the gradient field obtained from the learned Jacobian may not be exactly integrable, and different integration paths could lead to slightly different values.

In the proposed framework, the primary structural property is enforced at the level of the flux Jacobian, where real eigenvalues are guaranteed through the symmetrization mechanism. The line integration step is therefore used as a consistent and practical procedure to recover the pressure tensor from the learned gradients, and serves mainly to match the reference closure values. In particular, the accuracy of the reconstructed pressure is controlled by the data-driven training objective, rather than by exact path independence of the gradient field.
Under this choice of reference path, Eq.~\eqref{eq:m1_pressure_line_integrals} can be written componentwise as
\begin{equation}
\begin{aligned}
P_{xx}(\mathbf{u},\mathcal{G}(\mathbf{u}))
=\;&
P_{xx}^0
+
\int_0^1
\Bigl(
\frac{\partial F_x}{\partial v_2}\frac{\partial^2 \eta_\theta(t\mathbf{u})}{\partial E^2}
+
s_{22}^{(x)}\frac{\partial^2 \eta_\theta(t\mathbf{u})}{\partial E\,\partial F_x}
+
s_{23}^{(x)}\frac{\partial^2 \eta_\theta(t\mathbf{u})}{\partial E\,\partial F_y}
\Bigr)_{(t\mathbf{u},\mathcal{G}(\mathbf{u}))}
E\,dt \\
&+
\int_0^1
\Bigl(
\frac{\partial F_x}{\partial v_2}\frac{\partial^2 \eta_\theta(t\mathbf{u})}{\partial E\,\partial F_x}
+
s_{22}^{(x)}\frac{\partial^2 \eta_\theta(t\mathbf{u})}{\partial F_x^2}
+
s_{23}^{(x)}\frac{\partial^2 \eta_\theta(t\mathbf{u})}{\partial F_x\,\partial F_y}
\Bigr)_{(t\mathbf{u},\mathcal{G}(\mathbf{u}))}
F_x\,dt \\
&+
\int_0^1
\Bigl(
\frac{\partial F_x}{\partial v_2}\frac{\partial^2 \eta_\theta(t\mathbf{u})}{\partial E\,\partial F_y}
+
s_{22}^{(x)}\frac{\partial^2 \eta_\theta(t\mathbf{u})}{\partial F_x\,\partial F_y}
+
s_{23}^{(x)}\frac{\partial^2 \eta_\theta(t\mathbf{u})}{\partial F_y^2}
\Bigr)_{(t\mathbf{u},\mathcal{G}(\mathbf{u}))}
F_y\,dt,
\end{aligned}
\label{eq:m1_pxx_component_integral}
\end{equation}
\begin{equation}
\begin{aligned}
P_{xy}^{(x)}(\mathbf{u},\mathcal{G}(\mathbf{u}))
=\;&
P_{xy}^0
+
\int_0^1
\Bigl(
\frac{\partial F_y}{\partial v_2}\frac{\partial^2 \eta_\theta(t\mathbf{u})}{\partial E^2}
+
s_{23}^{(x)}\frac{\partial^2 \eta_\theta(t\mathbf{u})}{\partial E\,\partial F_x}
+
s_{33}^{(x)}\frac{\partial^2 \eta_\theta(t\mathbf{u})}{\partial E\,\partial F_y}
\Bigr)_{(t\mathbf{u},\mathcal{G}(\mathbf{u}))}
E\,dt \\
&+
\int_0^1
\Bigl(
\frac{\partial F_y}{\partial v_2}\frac{\partial^2 \eta_\theta(t\mathbf{u})}{\partial E\,\partial F_x}
+
s_{23}^{(x)}\frac{\partial^2 \eta_\theta(t\mathbf{u})}{\partial F_x^2}
+
s_{33}^{(x)}\frac{\partial^2 \eta_\theta(t\mathbf{u})}{\partial F_x\,\partial F_y}
\Bigr)_{(t\mathbf{u},\mathcal{G}(\mathbf{u}))}
F_x\,dt \\
&+
\int_0^1
\Bigl(
\frac{\partial F_y}{\partial v_2}\frac{\partial^2 \eta_\theta(t\mathbf{u})}{\partial E\,\partial F_y}
+
s_{23}^{(x)}\frac{\partial^2 \eta_\theta(t\mathbf{u})}{\partial F_x\,\partial F_y}
+
s_{33}^{(x)}\frac{\partial^2 \eta_\theta(t\mathbf{u})}{\partial F_y^2}
\Bigr)_{(t\mathbf{u},\mathcal{G}(\mathbf{u}))}
F_y\,dt,
\end{aligned}
\label{eq:m1_pxyx_component_integral}
\end{equation}
\begin{equation}
\begin{aligned}
P_{xy}^{(y)}(\mathbf{u},\mathcal{G}(\mathbf{u}))
=\;&
P_{xy}^0
+
\int_0^1
\Bigl(
\frac{\partial F_x}{\partial v_3}\frac{\partial^2 \eta_\theta(t\mathbf{u})}{\partial E^2}
+
s_{22}^{(y)}\frac{\partial^2 \eta_\theta(t\mathbf{u})}{\partial E\,\partial F_x}
+
s_{23}^{(y)}\frac{\partial^2 \eta_\theta(t\mathbf{u})}{\partial E\,\partial F_y}
\Bigr)_{(t\mathbf{u},\mathcal{G}(\mathbf{u}))}
E\,dt \\
&+
\int_0^1
\Bigl(
\frac{\partial F_x}{\partial v_3}\frac{\partial^2 \eta_\theta(t\mathbf{u})}{\partial E\,\partial F_x}
+
s_{22}^{(y)}\frac{\partial^2 \eta_\theta(t\mathbf{u})}{\partial F_x^2}
+
s_{23}^{(y)}\frac{\partial^2 \eta_\theta(t\mathbf{u})}{\partial F_x\,\partial F_y}
\Bigr)_{(t\mathbf{u},\mathcal{G}(\mathbf{u}))}
F_x\,dt \\
&+
\int_0^1
\Bigl(
\frac{\partial F_x}{\partial v_3}\frac{\partial^2 \eta_\theta(t\mathbf{u})}{\partial E\,\partial F_y}
+
s_{22}^{(y)}\frac{\partial^2 \eta_\theta(t\mathbf{u})}{\partial F_x\,\partial F_y}
+
s_{23}^{(y)}\frac{\partial^2 \eta_\theta(t\mathbf{u})}{\partial F_y^2}
\Bigr)_{(t\mathbf{u},\mathcal{G}(\mathbf{u}))}
F_y\,dt,
\end{aligned}
\label{eq:m1_pxyy_component_integral}
\end{equation}
\begin{equation}
\begin{aligned}
P_{yy}(\mathbf{u},\mathcal{G}(\mathbf{u}))
=\;&
P_{yy}^0
+
\int_0^1
\Bigl(
\frac{\partial F_y}{\partial v_3}\frac{\partial^2 \eta_\theta(t\mathbf{u})}{\partial E^2}
+
s_{23}^{(y)}\frac{\partial^2 \eta_\theta(t\mathbf{u})}{\partial E\,\partial F_x}
+
s_{33}^{(y)}\frac{\partial^2 \eta_\theta(t\mathbf{u})}{\partial E\,\partial F_y}
\Bigr)_{(t\mathbf{u},\mathcal{G}(\mathbf{u}))}
E\,dt \\
&+
\int_0^1
\Bigl(
\frac{\partial F_y}{\partial v_3}\frac{\partial^2 \eta_\theta(t\mathbf{u})}{\partial E\,\partial F_x}
+
s_{23}^{(y)}\frac{\partial^2 \eta_\theta(t\mathbf{u})}{\partial F_x^2}
+
s_{33}^{(y)}\frac{\partial^2 \eta_\theta(t\mathbf{u})}{\partial F_x\,\partial F_y}
\Bigr)_{(t\mathbf{u},\mathcal{G}(\mathbf{u}))}
F_x\,dt \\
&+
\int_0^1
\Bigl(
\frac{\partial F_y}{\partial v_3}\frac{\partial^2 \eta_\theta(t\mathbf{u})}{\partial E\,\partial F_y}
+
s_{23}^{(y)}\frac{\partial^2 \eta_\theta(t\mathbf{u})}{\partial F_x\,\partial F_y}
+
s_{33}^{(y)}\frac{\partial^2 \eta_\theta(t\mathbf{u})}{\partial F_y^2}
\Bigr)_{(t\mathbf{u},\mathcal{G}(\mathbf{u}))}
F_y\,dt.
\end{aligned}
\label{eq:m1_pyy_component_integral}
\end{equation}

% ====
% \section{Quadratic-path reconstruction}
\section{Path dependence study}

To assess the sensitivity of the closure reconstruction to the choice of path, one may also consider a curved quadratic path in state space. For example, letting $\mathbf{u}=(E,F_x,F_y)^\top$, we define
\begin{equation}
\gamma_\beta(t)
=
\begin{pmatrix}
\bigl(t+\beta t(1-t)\bigr)E\\
tF_x\\
tF_y
\end{pmatrix},
\qquad t\in[0,1],
\label{eq:quadratic_path}
\end{equation}
where $\beta$ controls the amount of curvature. The endpoints are unchanged, namely
\[
\gamma_\beta(0)=\mathbf{0},
\qquad
\gamma_\beta(1)=\mathbf{u}.
\]

With this path, the reconstruction becomes
\begin{equation}
P_{ij}(\mathbf{u},\mathcal{G}(\mathbf{u}))
=
P_{ij}^0
+
\int_0^1
\nabla_{\mathbf{u}}P_{ij}(\gamma_\beta(t),\mathcal{G}(\mathbf{u}))
\cdot
\dot{\gamma}_\beta(t)\,dt,
\qquad (i,j)\in\{(x,x),(y,y),(x,y)\},
\label{eq:quadratic_path_integral}
\end{equation}
with
\begin{equation}
\dot{\gamma}_\beta(t)
=
\begin{pmatrix}
\bigl(1+\beta(1-2t)\bigr)E\\
F_x\\
F_y
\end{pmatrix}.
\label{eq:quadratic_path_derivative}
\end{equation}
This quadratic-path reconstruction is not used in the main method, but provides a simple diagnostic for examining the path sensitivity of the learned gradient field.

\begin{figure}[h!]
    \centering
    \includegraphics[width=\linewidth]{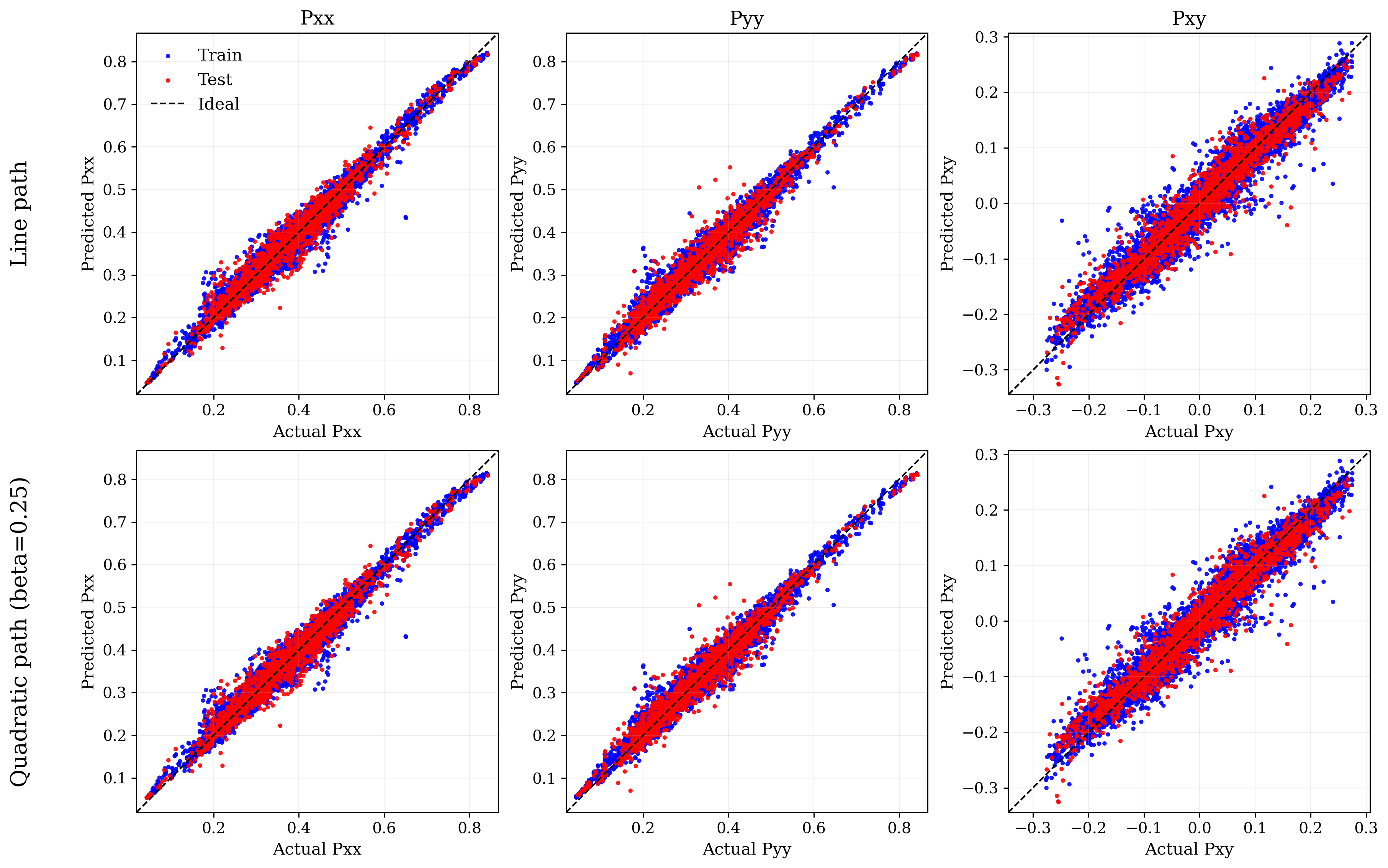}
    \caption{Scatter plots of reference and predicted closure outputs for the HN closure with gradient information in Experiment~\ref{ex1}, comparing line and quadratic integration paths. The top row shows results obtained using line integration, while the bottom row shows results obtained using quadratic integration with $\beta=0.25$. Training samples are shown in blue, testing samples in red, and the dashed line indicates the ideal relation $y=x$.}
    \label{fig:ex1_line_vs_quadratic_scatter}
\end{figure}

\begin{table}[h!]
\footnotesize
\centering
\caption{Comparison of closure accuracy for the lattice radiative transfer problem (Fig.~\ref{fig:configuration_ex1}) on a $7\times7$ domain, using line and quadratic integration paths for the HN closure with gradient information. The closure is trained using Monte Carlo reference data generated on a $100\times100$ lattice with $N_p=10^7$ particles, with 80\% of the data used for training and 20\% for testing. The neural network uses hidden width 128 with Tanh activation and quadrature points $N_q=4$. The results shown here are evaluated after 10{,}000 training epochs. Errors are reported for $P_{xx}$, $P_{yy}$, and $P_{xy}$. The quadratic path uses $\beta=0.25$.}
\label{tab:closure_metrics_path_compare_ex1_grad}
\begin{tabular}{lllcccc}
\hline
Closure & Split & Component & MSE & $R^2$ & $\max\lvert P-\hat P\rvert$ & $\mathrm{mean}\,\lvert P-\hat P\rvert$ \\
\hline
\multirow{6}{*}{\makecell{HN closure\\(line path)}}
& \multirow{3}{*}{Train}
& $P_{xx}$ & $3.5055\times10^{-4}$ & $0.9778$ & $2.1815\times10^{-1}$ & $1.2935\times10^{-2}$ \\
& & $P_{yy}$ & $3.3245\times10^{-4}$ & $0.9770$ & $1.6276\times10^{-1}$ & $1.2490\times10^{-2}$ \\
& & $P_{xy}$ & $5.7838\times10^{-4}$ & $0.9529$ & $2.1690\times10^{-1}$ & $1.6951\times10^{-2}$ \\ \cline{2-7}
& \multirow{3}{*}{Test}
& $P_{xx}$ & $5.0047\times10^{-4}$ & $0.9671$ & $1.3338\times10^{-1}$ & $1.5504\times10^{-2}$ \\
& & $P_{yy}$ & $5.3949\times10^{-4}$ & $0.9616$ & $1.7365\times10^{-1}$ & $1.5476\times10^{-2}$ \\
& & $P_{xy}$ & $7.0924\times10^{-4}$ & $0.9436$ & $1.9701\times10^{-1}$ & $1.9000\times10^{-2}$ \\
\hline
\multirow{6}{*}{\makecell{HN closure\\(quadratic path)}}
& \multirow{3}{*}{Train}
& $P_{xx}$ & $3.6380\times10^{-4}$ & $0.9770$ & $2.2053\times10^{-1}$ & $1.3335\times10^{-2}$ \\
& & $P_{yy}$ & $3.4410\times10^{-4}$ & $0.9762$ & $1.6279\times10^{-1}$ & $1.2909\times10^{-2}$ \\
& & $P_{xy}$ & $5.8382\times10^{-4}$ & $0.9525$ & $2.1680\times10^{-1}$ & $1.7159\times10^{-2}$ \\ \cline{2-7}
& \multirow{3}{*}{Test}
& $P_{xx}$ & $5.1686\times10^{-4}$ & $0.9660$ & $1.3366\times10^{-1}$ & $1.6106\times10^{-2}$ \\
& & $P_{yy}$ & $5.5504\times10^{-4}$ & $0.9604$ & $1.7339\times10^{-1}$ & $1.5929\times10^{-2}$ \\
& & $P_{xy}$ & $7.1774\times10^{-4}$ & $0.9429$ & $1.9886\times10^{-1}$ & $1.9233\times10^{-2}$ \\
\hline
\end{tabular}
\end{table}

To examine the sensitivity of the closure reconstruction to the choice of integration path, we compare the standard line integration with a representative quadratic integration path. As shown in Fig.~\ref{fig:ex1_line_vs_quadratic_scatter}, both reconstruction paths produce nearly identical scatter distributions for all pressure components in both training and testing samples. This observation is further confirmed quantitatively in Table~\ref{tab:closure_metrics_path_compare_ex1_grad}, where the error metrics obtained from the two reconstruction paths remain very close across all components. These results suggest that the closure reconstruction is not strongly sensitive to moderate variations in the integration path.

\end{document}